\documentclass[doctor,english,final, pdfdoc]{kaist-ucs} 
\usepackage{graphicx}
\usepackage{booktabs,multirow,adjustbox}
\usepackage{dsfont}
\usepackage{subcaption}
\usepackage{wrapfig}
\usepackage{enumitem}
\usepackage{xcolor}

\usepackage[numbers,sort&compress]{natbib}

\usepackage{hyperref}
\hypersetup{
  colorlinks   = true,
  linkcolor    = black,
  citecolor    = darkgreen,
  urlcolor     = darkblue
}

\usepackage{amsmath,amsfonts,bm}

\def\eqref#1{equation~\ref{#1}}

\def\1{\bm{1}}

\DeclareMathAlphabet{\mathsfit}{\encodingdefault}{\sfdefault}{m}{sl}
\SetMathAlphabet{\mathsfit}{bold}{\encodingdefault}{\sfdefault}{bx}{n}

\newcommand{\mytitle}{\alg: Stabilizing Dot-regression with Global Feature Distillation for Federated Learning}

\newcommand{\alg}{\textbf{\code{FedDr+}}\xspace}

\usepackage{booktabs} 
\usepackage{bbm}
\usepackage{mathtools}
\usepackage{nccmath}
\usepackage{setspace}
\usepackage[T1]{fontenc}
\usepackage[scaled]{beramono}
\usepackage{subcaption}

\usepackage{algorithm} 
\usepackage[linesnumbered,ruled,vlined,algo2e]{algorithm2e}
\SetKwInput{KwInput}{Input}                
\SetKwInput{KwOutput}{Output}              

\SetCommentSty{mycommfont}

\SetAlCapSty{algcapsty}

\usepackage[T1]{fontenc}
\usepackage{wrapfig,lipsum,booktabs}

\usepackage{soul}
\usepackage{dsfont}
\usepackage{enumitem}

\usepackage{amsmath}
\usepackage{amsfonts}
\usepackage{bbm}
\usepackage{dsfont}
\usepackage[Symbol]{upgreek}
\usepackage{lscape}
\usepackage{caption}
\usepackage{balance}
\usepackage{xspace}
\usepackage{float}

\usepackage{wasysym}
\usepackage{multirow}
\usepackage{array, boldline, rotating}

\usepackage{amssymb}
\usepackage{pifont}

\newcommand{\mc}[1]{\mathcal{#1}}

\definecolor{LightCyan}{rgb}{0.88,1,1}
\definecolor{Blue}{rgb}{0, 0.5, 1}
\definecolor{Green}{rgb}{0.0, 0.8, 0.0 }
\definecolor{Red}{rgb}{0.95, 0.55, 0.6}
\definecolor{Skyblue}{rgb}{0.6, 0.6, 0.95 }

\renewcommand*\eqref[1]{(\ref{#1})}

\newcommand{\ie}{\emph{i.e.,~}}

\newcommand{\myparagraph}[1]{\vspace{0.07cm}\noindent\textbf{#1}~}

\def\code#1{\texttt{#1}}

\NewDocumentCommand{\supptitle}{s}{
\onecolumn
\begin{center}
    \rule{\textwidth}{0.03cm}\\[0.1cm]
    - Appendix -\\[0.2cm]
    {\Large 
        \textbf{\mytitle }
    }\\
    \rule{\textwidth}{0.03cm}\\[0.2cm]
\end{center}
}

\usepackage{url}
\usepackage{multirow}
\usepackage{adjustbox}
\usepackage{wrapfig}
\usepackage{kotex}
\usepackage{bm}

\usepackage{amsfonts}
\usepackage{amsthm}
\usepackage{amssymb}
\usepackage{enumitem}
\newtheorem{thm}{Theorem}[section]

\newtheorem{lem}{Lemma}
\newtheorem{prop}{Proposition}

\usepackage{tcolorbox}

\usepackage{lastpage}

\title[korean]{연합학습에서의 효율적인 개인화 모델 구축을 위한 중앙 학습 전략}
\title[english]{Global Federated Learning strategies for building efficient personalized models}

\author[korean]{김}{성 윤}
\author[korean2]{김}{성윤}    
\author[chinese]{金}{成 潤}
\author[english]{Kim}{Seong Yoon}

\advisor[major]{윤 세 영}{Se-Young Yun}{signed}
\advisor[major2]{윤세영}{Se-Young Yun}{signed}    
\advisorinfo{Professor of Kim Jaechul Graduate School of AI} 

\department{IE}{engineering}{a}

\studentid{20197008}

\referee[1]{윤 세 영}
\referee[2]{양 은 호}
\referee[3]{박 찬 영}
\referee[4]{이 문 용}
\referee[5]{김 희 영}

\approvaldate{2025}{12}{15} 

\refereedate{2025}{12}{15}

\gradyear{2026}

\begin{document}


   \thesisinfo
    \begin{summary}      
    연합학습은 데이터 프라이버시를 보장하면서 분산된 사용자 데이터로 모델을 학습할 수 있는 실용적 틀이지만, 사용자마다 데이터 분포가 다른 이질성 때문에 전역 모델 성능과 개인화 성능이 동시에 저하되는 문제가 빈번히 발생한다. 본 학위논문은 효율적인 개인화 모델을 구축하기 위해, 전역 학습 단계에서 어떤 중앙 학습 전략이 유효한지, 그리고 로컬 적응 과정에서 전역 지식을 어떻게 보존하면서도 사용자 특화 성능을 확보할 수 있는지에 대한 방법론을 제시한다.
    첫째, 데이터 이질성이 커질수록 분류기 가중치보다 특징 벡터 붕괴가 더 본질적인 병목이 됨을 보이고, 로컬 모델과 전역 모델 사이의 표현 크기 불일치를 직접 완화하는 방법을 제안한다.
    둘째, 로컬 정렬을 강화하는 학습 방식이 전역 지식(예: 로컬에서 관측되지 않은 범주)에 대한 망각을 유발할 수 있음을 분석하고, 전역 모델의 특징 벡터를 기준으로 한 특징 벡터 증류를 결합하여 로컬 정렬과 전역 지식 보존을 동시에 달성하는 방법을 제시한다.
    셋째, 선호 이질성이 존재하는 연합 개인화 보상 모델 학습에서 ``전역 모델을 여러 개로 늘리면 더 좋은 초기화를 얻는다”는 통념을 실증적으로 검증하고, 충분한 로컬 미세조정이 허용될 때는 단일 전역 초기화가 오히려 더 강한 개인화 성능을 제공할 수 있음을 보인다. 본 연구는 데이터 및 선호 이질성 환경에서 전역 초기화의 역할을 재정의하고, 전역 지식 보존과 개인화를 동시에 만족시키는 실용적 학습 전략을 제공한다.

    \end{summary}
   
    \begin{Korkeyword}
    연합학습, 개인화, 데이터 이질성, 특징 벡터 정규화, 지식 보존, 로컬 정렬, 보상 모델, 선호 이질성, 전역 초기화
    \end{Korkeyword}
    \vspace{-15pt}

    \begin{abstract}
        Federated learning (FL) is a practical framework that can train models on distributed user data while guaranteeing data privacy; however, due to heterogeneity in which each user has a different data distribution, problems frequently arise where both global and personalization performance deteriorate simultaneously. This dissertation presents methodologies for building efficient personalized models by identifying which strategies are effective in the global training stage and by showing how to preserve global knowledge while securing user-specific performance during local adaptation. First, we show that as data heterogeneity increases, the collapse of feature vectors is a more fundamental bottleneck than classifier weights, and propose a method that directly mitigates the discrepancy in representation magnitude between local and global models. Second, we analyze that a training approach that strengthens local alignment can induce forgetting of global knowledge (e.g., categories not observed locally), and propose a method that achieves both local alignment and global knowledge preservation by combining feature distillation based on the global model’s feature vectors. Third, in federated personalized reward model learning with preference heterogeneity, we empirically verify the conventional belief that “increasing the number of global models yields better initialization,” and we show that when sufficient local fine-tuning is allowed, a single global initialization can instead provide stronger personalization performance. This study redefines the role of global initialization under data and preference heterogeneity and provides practical training strategies that simultaneously satisfy global knowledge preservation and personalization.
    \end{abstract} 
    
    \vspace{-5pt}
    \begin{Engkeyword}
    Federated Learning, Personalization, Data Heterogeneity, Feature Normalization, Knowledge Preservation, Reward Modeling, Preference Heterogeneity, Global Initialization
    \end{Engkeyword}

    \addtocounter{pagemarker}{1}                 
    \newpage

    \tableofcontents

    \listoftables

    \listoffigures



\chapter{Introduction}

\begin{figure}[h!]
    \centering
    \includegraphics[width=\linewidth]{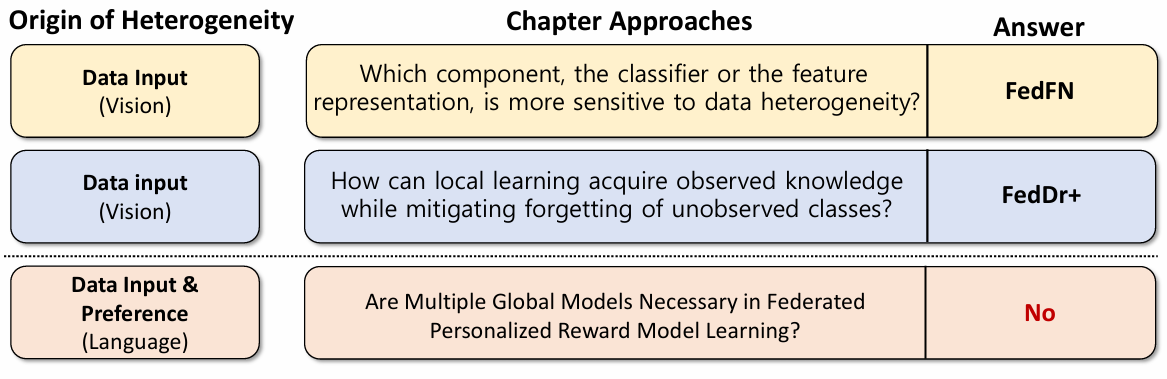}
    \caption{An overview of dissertation}
    \label{fig:overview}
\end{figure}

Modern machine learning systems are increasingly deployed in user-facing and privacy-sensitive environments, where data are naturally distributed across devices and users. Federated learning (FL) offers a practical solution by keeping data local while training models through server aggregation. Yet, a persistent challenge is heterogeneity: users differ not only in data quantity and label distribution, but also in the latent structure of their tasks and preferences. Such heterogeneity often turns global training into a compromise that underfits many users, while naive personalization can overfit local data and damage the model’s general capability.

This dissertation addresses the central question of how to build efficient personalized models under heterogeneous data and preferences. The key theme is the role of global learning strategies as a foundation for personalization: a strong global initialization should capture broadly transferable structure, and personalization should then adapt to each user without unnecessarily destroying that structure. Concretely, we study heterogeneity through three lenses: (i) distribution shifts that distort representation learning in conventional FL, (ii) local objectives that improve client-specific alignment but may induce global forgetting, and (iii) preference heterogeneity in reward modeling, where personalization depends critically on the quality of global initialization. Figure~\ref{fig:overview} provides an overview of how these three perspectives map to the core questions and methods studied in this dissertation. 

\section{Outline}

Chapter 2 investigates why global FL degrades as heterogeneity increases, showing that the dominant failure mode is the feature-norm discrepancy between local and global models. Building on this observation, we propose FedFN, which incorporates feature normalization into FL to improve robustness under non-IID settings.

Chapter 3 examines dot-regression as a direct way to strengthen feature–classifier alignment on clients. The empirical analysis shows that dot-regression can improve local alignment but may degrade the global model by amplifying forgetting of unobserved classes. To resolve this trade-off, we propose FedDr+, which combines dot-regression with global feature distillation to preserve general knowledge while maintaining strong local alignment.

Chapter 4 studies federated personalized reward model learning under preference heterogeneity. Although a common intuition is that clustering-based multi-global training yields better initializations for personalization, the experiments demonstrate that a single global model trained with FL can serve as a strong initialization and consistently yields the best personalized performance after local fine-tuning. The chapter further discusses why additional global models are unnecessary for strong personalization in this setting.

Overall, this dissertation provides a unified view of global initialization and personalization in heterogeneous FL: effective global training should stabilize feature learning, preserve transferable knowledge, and produce an initialization that clients can adapt efficiently to diverse data and preferences.

\chapter{FedFN: Feature Normalization for Alleviating Data Heterogeneity Problem in Federated Learning}
\begin{tcolorbox}[
    colback=gray!10, 
    colframe=black, 
    width=0.9\textwidth, 
    boxrule=1pt, 
    arc=4pt, 
    left=5pt, 
    right=5pt, 
    center
]

Federated Learning (FL) is a collaborative method for training models while preserving data privacy in decentralized settings. However, FL encounters challenges related to data heterogeneity, which can result in performance degradation. In our study, we observe that as data heterogeneity increases, feature representation in the FedAVG model deteriorates more significantly compared to classifier weight. 
Additionally, we observe that as data heterogeneity increases, the gap between higher feature norms for observed classes, obtained from local models, and feature norms of unobserved classes widens, in contrast to the behavior of classifier weight norms. This widening gap extends to encompass the feature norm disparities between local and the global models. To address these issues, we introduce Federated Averaging with Feature Normalization Update (FedFN), a straightforward learning method. We demonstrate the superior performance of FedFN through extensive experiments, even when applied to pretrained ResNet18. Subsequently, we confirm the applicability of FedFN to foundation models.

\end{tcolorbox}

\section{Introduction}
\label{ch2sec:intro}

Federated Learning (FL) facilitates collaborative model training while preserving data privacy~\citep{MLSYS2020_1f5fe839, mcmahan2017communication}. This approach consists of four iterative stages: (1) client selection, (2) broadcasting, (3) local training, and (4) aggregation. Selected clients receive the global model, train it locally with their own data, and transmit the trained models back to the central server for aggregation. These steps are repeated to progressively enhance the performance of global model.
In FL, a fundamental challenge arises from the existence of diverse data distributions among different clients, referred to as \emph{data heterogeneity}. This leads to performance degradation of the global model~\citep{li2021fedrs, mcmahan2017communication, li2019convergence}.

Recently, numerous studies~\citep{dong2022spherefed, oh2021fedbabu, shi2022towards} have been conducted to identify the specific aspects, such as feature representations or classifier weights, that are significantly influenced by data heterogeneity. \citet{luo2021no} demonstrated that classifier weights are the most sensitive to data heterogeneity, illustrating how classifiers can easily become biased depending on the data distribution. To mitigate classifier bias, algorithms have been proposed using restricted softmax loss or fixed orthogonal classifiers during local training~\citep{li2021fedrs, oh2021fedbabu, dong2022spherefed}. 
Meanwhile, \citet{shi2022towards} demonstrates the impact of data heterogeneity on feature representations, potentially leading to dimensional collapse where only a subset of high-dimensional feature vectors is employed to represent features within the global model. In our study, we find that the primary concern lies not in the classifier weight but in the features. 

Feature normalization, as utilized in various fields~\citep{wang2018additive, savvides2021convex, hasnat2017deepvisage, mettes2019hyperspherical, khosla2020supervised, li2021model, dong2022spherefed}, enhances the discriminative power of feature representation, making it easier to distinguish data belonging to different classes. We reveal that data heterogeneity in FL leads to a substantial discrepancy in feature norms between the global model and local models. Based on this observation, we incorporate feature normalization into the FL framework. Our contributions are outlined as follows:
\begin{itemize}
    \item In FedAVG, we find that feature representations are more adversely affected by data heterogeneity than classifier weights. Furthermore, as data heterogeneity increases, the disparity between the higher feature norms for observed classes, derived from local models, and the feature norms of unobserved classes widens, in contrast to classifier weight norms. This widening gap extends to encompass feature norm disparities between local models and the global model. \textbf{(Section~\ref{ch2sec:problem})}
    \item To tackle this challenge, we introduce \textbf{Fed}erated Averaging with \textbf{F}eature \textbf{N}ormalization Update (FedFN), which effectively eliminates discrepancies in feature norms during local training. FedFN robustly maintains the quality of feature representations even in highly heterogeneous data settings. \textbf{(Section~\ref{ch2sec:method})}
    \item We incorporate the feature normalization technique into existing algorithms, and show notable performance improvements. Furthermore, this effectiveness persists even when using pretrained model.  \textbf{(Section~\ref{ch2sec:exp})}
\end{itemize}

\section{Related Work}
\label{ch2sec:related}

\subsection{Federated Learing}
\textbf{Global Federated Learning}
Global Federated Learning (GFL) aims to enhance the performance of a single global model across decentralized clients by addressing data heterogeneity arising from diverse user behaviors. Researchers have explored various methodologies within GFL to create robust models for diverse devices and data sources. These approaches include client drift mitigation~\citep{MLSYS2020_1f5fe839,karimireddy2020scaffold,jhunjhunwala2023fedexp}, aggregation schemes to improve model fusion mechanisms at the server~\citep{wang2020federated, wang2020tackling}, and data sharing techniques introducing public datasets or synthesized data to achieve a more balanced data distribution~\citep{zhao2018federated, lin2020ensemble, luo2021no}.

\noindent{\textbf{Personalized Federated Learning}}
Personalized Federated Learning (PFL) focuses on training personalized models for individual clients, adapting to their specific data distributions and tasks. PFL methodologies include decoupling methods that separate the feature extractor and classifier during communication, enabling unique updates for the data distribution of each client~\citep{arivazhagan2019federated, collins2021exploiting, oh2021fedbabu}, modifying local loss functions to improve task performance~\citep{fallah2020personalized, li2021ditto}, and utilizing prototype communication techniques~\citep{tan2022fedproto, xu2023personalized}.

\subsection{Partial Model Updates in FL}
\textbf{Debiasing classifier in FL}
Efforts to address data heterogeneity in both GFL and PFL domains have explored differential updates within the model parameters, with a particular focus on the classifier part. For instance, \citet{luo2021no} propose classifier post-calibration with virtual features to tackle a notable bias among classifiers of different local models. \citet{li2021fedrs} introduce the restricted softmax loss for local updates to prevent classifiers from becoming inaccurate when updating for missing classes. Additionally, certain studies~\citep{oh2021fedbabu, dong2022spherefed, li2023no, huang2023neural} suggest the use of fixed classifiers constructed from orthogonal basis vectors during training.

\noindent\textbf{Feature Enhancement in FL} Recent research~\citep{yu2021fed2, li2021model, li2022federated, shi2022towards} in the context of GFL has focused on aligning feature representations among local models. For instance, \citet{li2021model} introduce the contrastive loss during local iteration to improve feature alignment. Additionally, \citet{shi2022towards} highlights the issue of data heterogeneity leading to severe dimensional collapse in the global model, resulting in representations tending towards lower dimensions. To address this, they propose using a regularization term during local training to mitigate the issue and improve feature alignment. Moreover, some studies ~\citep{tan2022fedproto, xu2023personalized} have explored communication strategies involving feature prototypes to further enhance feature alignment in the context of PFL.

\subsection{Feature Normalized Model in DL}
Feature normalized model has been widely adapted in various fields of deep learning (DL), including face recognition~\citep{wang2018additive, savvides2021convex}, regression~\citep{mettes2019hyperspherical}, and federated learning~\citep{dong2022spherefed}, with the aim of enhancing the discriminative power of features. In the several studies~\citep{wang2018additive, dong2022spherefed}, both the feature vectors and classifier weights are normalized to enforce cosine-similarity element logits, restricting the values of each element in the logit vector.

\section{Preliminaries}
\label{ch2sec:prelim}
In the upcoming section, we elucidate the concept of FL, followed by an in-depth discussion of the experimental setup, encompassing dataset descriptions, model specifications, hyperparameter configurations during FL, and detailed description four factor analysis conducted in main paragraph. For clarity and convenience, we present a concise overview of key notations in Table \ref{ch2tab:main_notation_summary}, facilitating comprehension of the paper.  

\begin{table}[htp]
\centering
\caption{Main Notations Throughout the Paper.}
\resizebox{0.8\textwidth}{!}{
\begin{tabular}{ll}

\hline
\textbf{Indices} & \\ 
$c\in[C]$  & index for a class   \\
$r\in[R]$ & index for FL round\\
$n\in[N]$ & index for a client   \\ 

\hline
\textbf{Dataset}&  \\
$(x,y)$ & (input image , true class label of $x$) \\
$D_{train}, D_{test}$ & total train and test dataset \\

$D_{train}^n, D_{test}^n$ & train and test dataset of client $n$\\

$D(c)$ & collection of dataset $D$ with the label $c$\\

\hline
\textbf{Parameters} & \\ 
$\theta:=(\theta_{ext}, \theta_{cls})$  & model parameter   \\
$\theta_{ext}$ & feature extractor part of $\theta$   \\ 
$\theta_{cls}\in \mathds{R}^{C\times d}$ & classifier part of $\theta$ \\
$\theta_{cls,i}, i\in[C]$ & i-th row vector of $\theta_{cls}$ \\
\hline

\textbf{Model Forward}&  \\
$p(x;\theta)\in\mathds{R}^{C}$& softmax probability of input $x$ \\ 
$p_i(x;\theta), i\in[C]$& i-th element of $p(x;\theta)$ \\ 
$\mathcal{L}_{CE}(x;\theta):=-\log p_{y}(x;\theta)$& cross entropy loss of input $x$\\
$f(x;\theta_{ext})\in\mathds{R}^{d}$ & feature vector of input $x$  \\ 
$\hat{f}(x;\theta_{ext}):=f(x;\theta_{ext})/||f(x;\theta_{ext})||_{2}$ & normalized feature vector of input $x$  \\
$z(x;\theta):=\theta_{cls}\,f(x;\theta_{ext})\in\mathds{R}^{C}$& logit vector of input $x$  \\ 
$z_{i}(x;\theta), i\in[C]$& i-th element of $z(x;\theta)$  \\ 
\hline

\hline
\textbf{Prototype  of label c}&  \\
$f(D(c);\theta_{ext}):=\frac{1}{|D(c)|}\sum_{(x,y)\in D(c)}f(x;\theta_{ext})$ & prototype of $D(c)$ \\
$\hat{f}(D(c);\theta_{ext}):=\frac{1}{|D(c)|}\sum_{(x,y)\in D(c)}\hat{f}(x;\theta_{ext})$ & prototype from the normalized features \\
\hline

\end{tabular}
}
\label{ch2tab:main_notation_summary}\end{table}

\subsection{FL Procedure}
In FL, we aim to train a robust image classification model on the central server while preserving the privacy and security of individual client data. The procedure involves communication over $R$ rounds. In each round $r \in [R]$, a random subset of clients $S_r\subset [N]$ is selected from the client pool. These selected clients receive the current global model parameters $\theta^{r-1}$ from the central server. Subsequently, each client $n \in S_r$ performs local updates on its local dataset $D_{train}^n$ for $E$ epochs using a batch size of $B$. The updated model parameters for client $n$ are denoted as ${\theta}^{r,n}$. Afterward, the central server updates the global model parameter $\theta^r$ as a result of convex combination based on the ${\theta}^{r,n}$~\cite{li2019convergence}. This collaborative approach iteratively refines the image classification model on the central server over the $R$ rounds, resulting in a robust model that performs well on the entire test dataset $D_{test}$ while preserving individual client data privacy.

\subsection{Experimental Setup}

\textbf{Datasets and Models}  To simulate a realistic federated learning scenario involving 100 clients, we conduct extensive studies on two widely-used datasets: CIFAR-10 and CIFAR-100~\citep{krizhevsky2009cifar}. For CIFAR-10, we employ the VGG11~\citep{simonyan2014very} model, while for CIFAR-100, the MobileNet~\citep{howard2017mobilenets} model is chosen. The training data is distributed among 100 clients using two distinct Non-IID partition strategies:

\begin{itemize}
\item \textbf{Sharding}~\citep{mcmahan2017communication, oh2021fedbabu}: We meticulously organize the data by label and divide it into non-overlapping shards of equal size. Each shard encompasses $\frac{|D_{train}|}{100\times s}$ samples of the same class, and $s$ denotes the number of shards per client. This results in each client having access to a maximum of $s$ different classes. As we decrease the number of shards per user $s$, the level of data heterogeneity among clients increases. For CIFAR-10, we explore various $s$ values, such as $s\in\{2, 3, 5, 10\}$, while for CIFAR-100, we experiment with $s \in \{10, 50, 100\}$.

\item \textbf{Latent Dirichlet Allocation (LDA)}~\citep{luo2021no, wang2020federated}: We utilize the LDA technique to sample a probability vector $p_c = (p_{c,1}, p_{c,2}, \cdots, p_{c,100}) \sim Dir(\alpha)$ and allocate a proportion $p_{c,k}$ of instances of class $c \in [C]$ to each client $k \in [100]$, where $Dir(\alpha)$ represents the Dirichlet distribution with the concentration parameter $\alpha$. The parameter $\alpha$ controls the strength of data heterogeneity, where smaller values lead to stronger heterogeneity among clients. For both CIFAR-10 and CIFAR-100, we conduct experiments with various $\alpha$ values, such as $\alpha \in \{0.1, 0.3, 0.5, 1.0\}$.

\end{itemize}

\noindent\textbf{Hyperparameter Search:} To optimize the hyperparameters for federated learning, we conduct grid searches for the initial learning rate on both CIFAR-10 and CIFAR-100. For CIFAR-10, we explore learning rates in the range of $\{0.01, 0.03, 0.05, 0.1\}$, while for CIFAR-100, we consider learning rates of $\{0.1, 0.3, 0.5, 1.0\}$. Additionally, we perform grid searches to determine the optimal number of local epochs, evaluating values in the set $\{1, 5, 10, 15, 20\}$ for both datasets. The optimal number of local epochs is found to be 15 for CIFAR-10 and 5 for CIFAR-100. In cases where specific values are not mentioned, we use default initial learning rates of 0.01 for CIFAR-10 and 0.1 for CIFAR-100. 

\noindent\textbf{Implementation Details} All experiments are conducted for 320 rounds to thoroughly assess the performance and convergence behavior of the models. To ensure convergence during training, we decay the learning rate by 0.1 at half and three-quarters of the federated learning rounds. Additionally, we utilize random horizontal flipping as a data augmentation technique throughout the training process. Table~\ref{ch2tab:gfl_acc} in Section~\ref{ch2sec:exp} and Table~\ref{ch2tab:gfl_acc_all} in Subsection~\ref{ch2app:add_gfl_result} are constructed using the code structure from \texttt{https://github.com/Lee-Gihun/FedNTD}, while the rest of the implementations are based on \texttt{https://github.com/jhoon-oh/FedBABU}.

\section{Heterogeneity in FedAVG: The Devil is in Feature Norm Discrepancy}\label{ch2sec:problem}

\subsection{4-Factor Analysis of FedAVG}\label{ch2subsec:4-factor-avg}
Within the FedAVG, we explore the impact of data heterogeneity on both feature representations at the penultimate layer and classifier weights, denoted as $f(\cdot; \theta_{ext})\in \mathds{R}^{d}$ and $\theta_{cls}\in \mathds{R}^{C \times d}$, respectively. Our investigation centers on four factors often used to assess model performance~\citep{kang2019decoupling, papyan2020prevalence}:
\begin{enumerate}[label=\textbf{(\roman*)}]
\item \textbf{Weight similarity:} Measuring the similarity or dissimilarity among classifiers across classes, this factor computes the cosine similarity between their normalized weight vectors, resulting in a symmetric matrix. The detail form is :
\begin{equation*}
N(\theta_{cls})^{\top}N(\theta_{cls})\in [-1,1]^{C\times C}, \text{where}\,\, N(\theta_{cls})=\left[\frac{\theta_{cls,1}^\top}{||\theta_{cls,1}||_{2}}|\cdots ,\frac{\theta_{cls,C}^\top}{||\theta_{cls,C}||_{2}}\right]\in\mathds{R}^{d\times C}.
\end{equation*}
Lower weight similarity is preferred, indicating distinct classifiers for each class and better discrimination.
\item \textbf{Inter-class similarity:} This factor delves into the relationships between feature prototypes representing different classes, represented by a symmetric matrix $f(D_{test};\theta)^{\top}f(D_{test};\theta)\in [-1,1]^{C\times C}$. Here, $f(D_{test};\theta)$  is represented as
\begin{equation*}
    \left[f(D_{test}(1);\theta)|\cdots|f(D_{test}(C);\theta)\right]\in \mathds{R}^{d\times C}.
\end{equation*}

Lower inter-class similarity is desired, representing distinguishable feature vectors for different classes.    
\item \textbf{Intra-class similarity:} This factor provides insights into the diversity or similarity of representations within individual classes. We achieve this by calculating the cosine similarity between feature vectors of test prototypes belonging to the same class.  For each class $c\in [C]$, it is evaluated by:
\begin{equation*}
    \frac{1}{|D_{test}(c)|}\sum_{(x,y)\in D_{test}(c)} \frac{f(x;\theta_{ext})^{\top}}{||f(x;\theta_{ext})||_2}\frac{f(D_{test}(c);\theta_{ext})}{||f(D_{test}(c);\theta_{ext})||_{2}}\in\mathds{R}.
\end{equation*}
Higher intra-class similarity is preferred, indicating that feature vectors belonging to the same class are closer to each other, thereby improving class representation and classification performance.
\item \textbf{Prototype-weight alignment:} Assessing this factor reveals the degree of alignment between the classifier and
prototypes for each class, according to their internal
product. High alignment signifies a strong match, while
low alignment indicates a potential mismatch or poor fit. For each class $c\in [C]$, it is evaluated by inner product of $\theta_{cls,c}$ and $f(D_{test}(c);\theta)$. This factor is generally not a primary concern but may be correlated with weight similarity and inter-class similarity.

\end{enumerate}

These four factors encompass both feature and classifier-related aspects, enabling us to discern which aspects are more negatively impacted as data heterogeneity increases.

\newpage
\begin{figure}[t!]
    \centering
    \makebox[\textwidth]{\includegraphics[width=\textwidth]{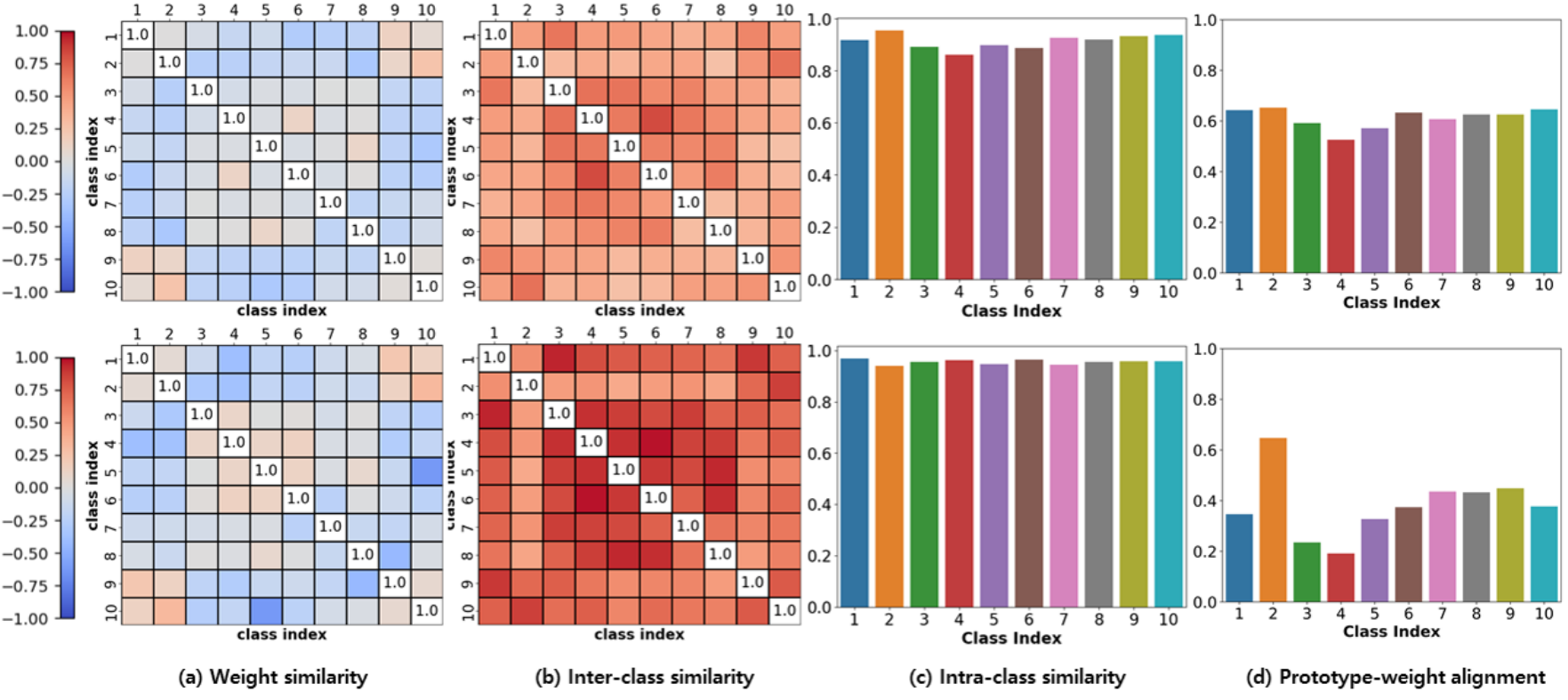}}
    \caption{4-Factor Analysis of FedAVG: Results for $s=10$ (Top Row) and $s=2$ (Bottom Row).}
    \label{ch2fig:4_factor_fedavg}
\end{figure}

Figure~\ref{ch2fig:4_factor_fedavg} visualizes the four factors concerning data heterogeneity. The upper and lower rows present the results under smaller (i.e., $s$=10) and larger data heterogeneity (i.e., $s$=2), respectively. In both settings, weight similarity exhibits lower values, as indicated by the blue color. However, with increasing data heterogeneity, there is an increase in inter-class similarity within FedAVG, indicating a negative impact.  Conversely, as data heterogeneity rises, intra-class similarity improves. With higher data heterogeneity,  prototype-weight alignment deteriorates, likely influenced by the more pronounced decrease in inter-class similarity. In summary, as data heterogeneity increases, the factors most adversely affected are inter-class similarity and prototype-weight alignment, both of which are common feature-related factors.

\subsection{Feature Norm Discrepancy Persists in Local and Global Models}

\begin{figure}[h!]
    \centering
    \begin{subfigure}[b]{0.45\textwidth}
        \includegraphics[width=\textwidth]{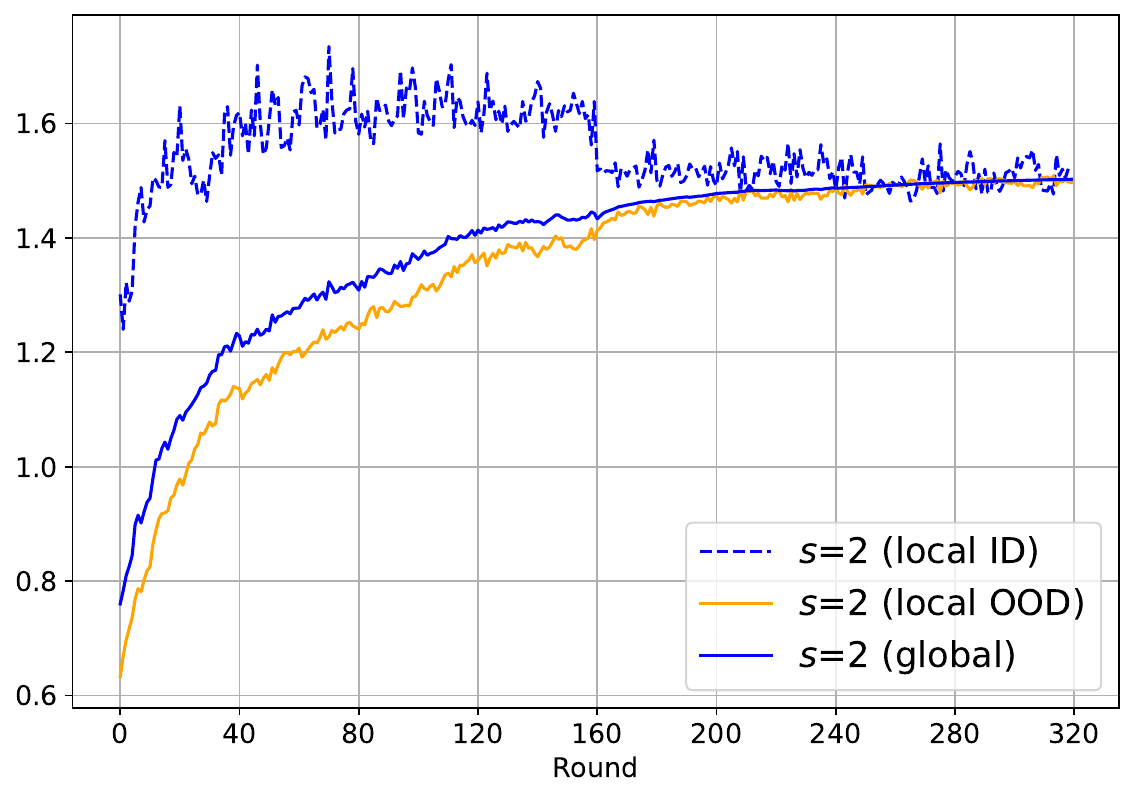}
        \label{ch2fig:weight-norm}
        \caption{Classifier weight norm means}
    \end{subfigure}    
    \begin{subfigure}[b]{0.45\textwidth}
        \includegraphics[width=\textwidth]{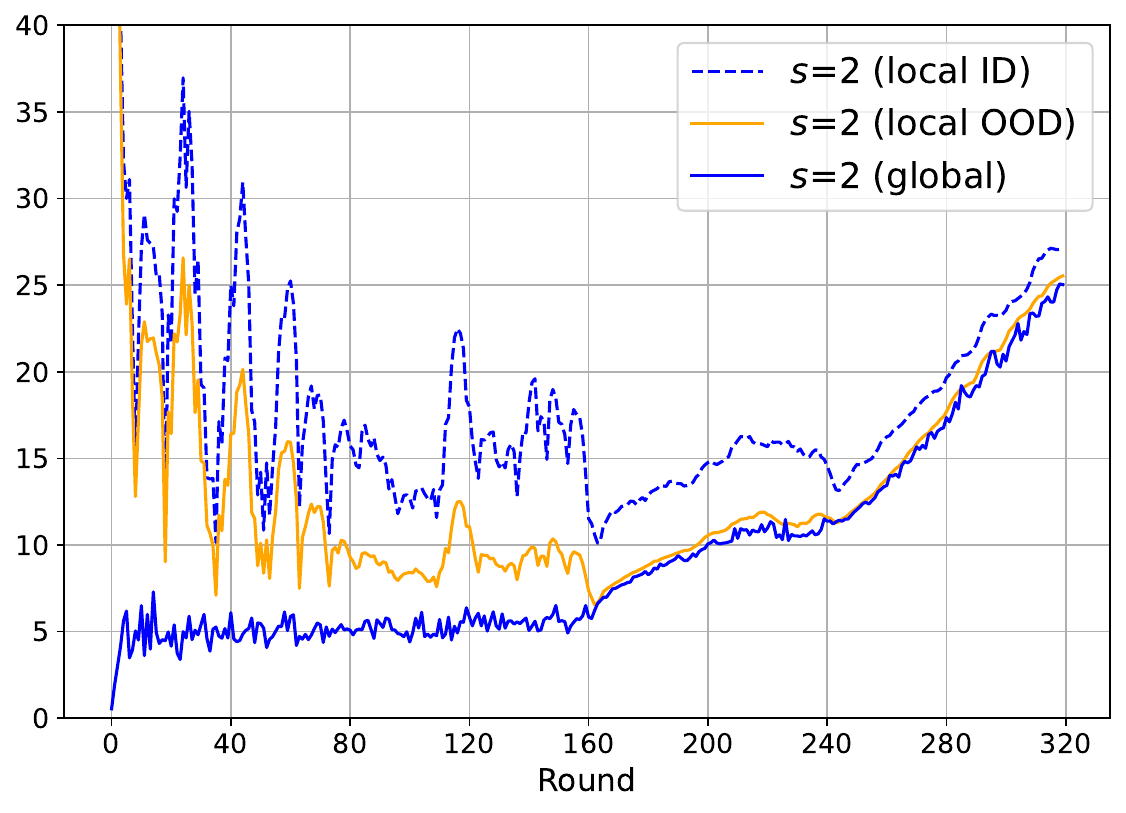}
        \label{ch2fig:feature-norm}
        \caption{Feature norm means}
    \end{subfigure}
    \caption{Means of Classifier Weight and Feature Norms for Local and Global Models with $s$=2.}
    \label{ch2fig:weight-feature-norm}
\end{figure}

We take a closer look at feature norm bias alongside weight norm bias~\citep{luo2021no, zhao2020maintaining, wu2019large, kang2019decoupling, oh2021fedbabu, yu2020devil}, which tends to favor major classes with larger weight norms. Our investigation is conducted in a high data heterogeneity setting ($s$=2) using the CIFAR-10 dataset on VGG11. We are motivated by the extensive distribution of calssifier weight norms within classes observed across clients in FL, as discussed in \citep{luo2021no}. Furthermore, our prior four factor analysis has emphasized the significant influence of feature-related aspects in response to varying data heterogeneity. This analysis strongly encourages us to look into feature norm bias. This is important because even though a lot of research has been done on weight norm bias, not much attention has been given to feature norm bias.

We compute weight norm means for two groups of classes: those seen (ID) and unseen (OOD) classes during their respective local training. Additionally, we calculate the total class weight norm mean from the global model. Furthermore, we explore feature norm means derived from local models using both the ID and OOD test datasets, in addition to the feature norm mean obtained from the global model using the entire test dataset. Figure~\ref{ch2fig:weight-feature-norm} illustrates the visual representation of the results.

During the initial stage of training, weight norms in local models exhibit a significant bias in favor of ID classes over OOD classes. Simultaneously, feature norm mean within local model also display a corresponding bias. However, as the learning rate gradually decreases, both feature norm bias and weight norm bias diminish. Weight norm bias eventually vanishes in local models, aligning with the weight norm mean of global model. In contrast, feature norm bias persists, consistently resulting in higher feature norm mean in local model from ID classes compared to the global model.
\vspace{0.2 in}

\subsection{Feature Norm Discrepancy with Increasing Heterogeneity}

\begin{figure}
    \centering
    \includegraphics[width=0.7\linewidth]{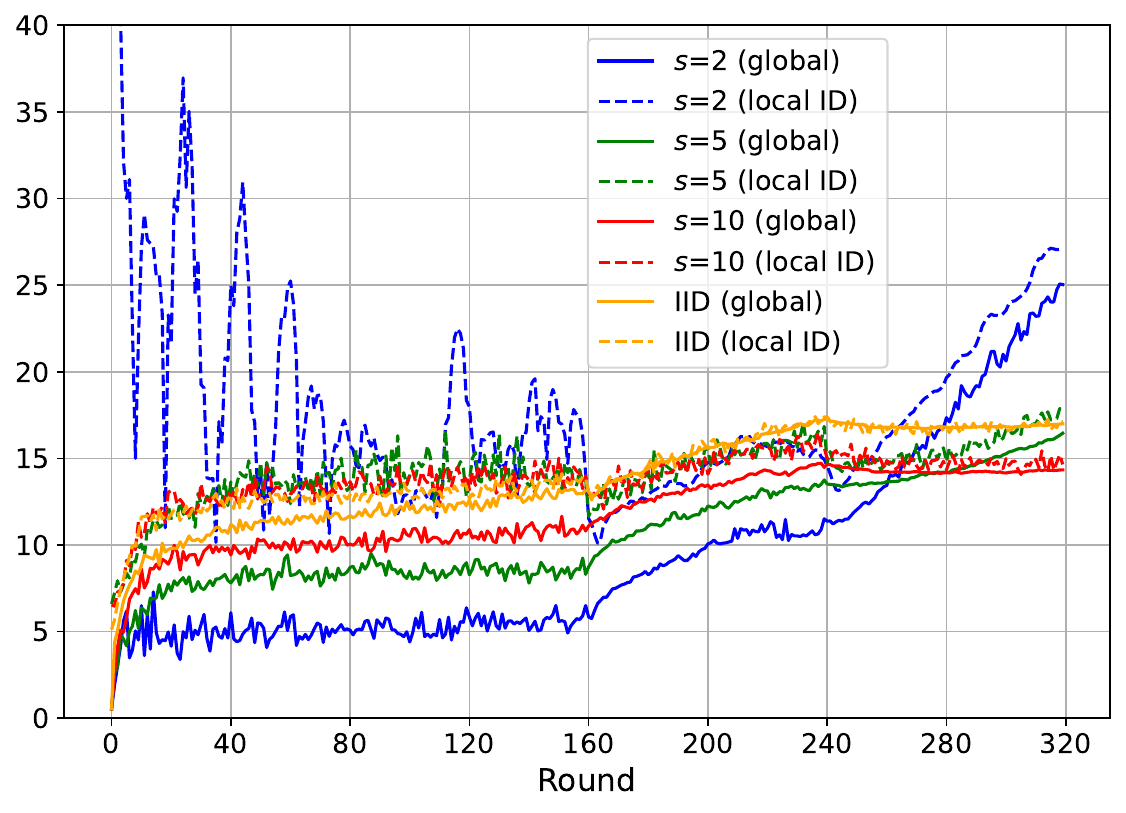}
    \caption{Feature Norm Means of Global and Local Models across Varying Data Heterogeneity.}
\label{ch2fig:local_global_feature_norm}
\end{figure}

We examine the discrepancy of  feature norm means between local models on ID test dataset and global model on the entire test dataset under varying data heterogeneity. Specifically, we consider $s \in \{2, 5, 10\}$ and IID (Exactly class balanced data distribution across clients) on the CIFAR-10 dataset. As depicted in Figure~\ref{ch2fig:local_global_feature_norm}, the discrepancy in the norms between the global model and local models increases, as data heterogeneity increases (i.e., IID $\rightarrow$ $s$=10 $\rightarrow$ $s$=5 $\rightarrow$ $s$=2). 

\subsection{Reducing Feature Norm Discrepancy for Improved Performance}

\begin{figure}[h!]
    \centering
    \begin{subfigure}[b]{0.48\textwidth}
        \includegraphics[width=\textwidth]{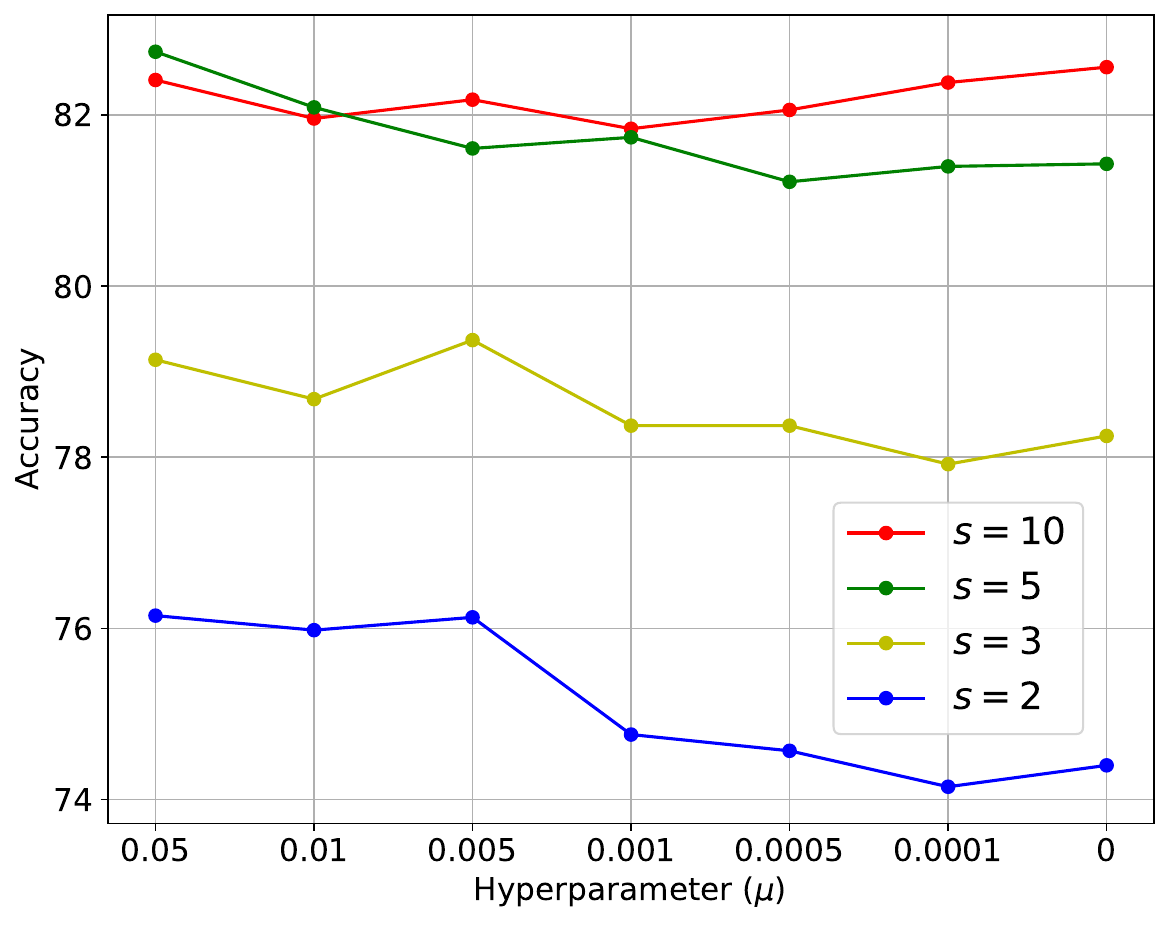}
        \label{ch2fig:hyperparam-effect}
        \vspace{-0.25 in}
        \caption{Accuracy on hyperparameter $\mu$}
    \end{subfigure}    
    \begin{subfigure}[b]{0.47\textwidth}
        \includegraphics[width=\textwidth]{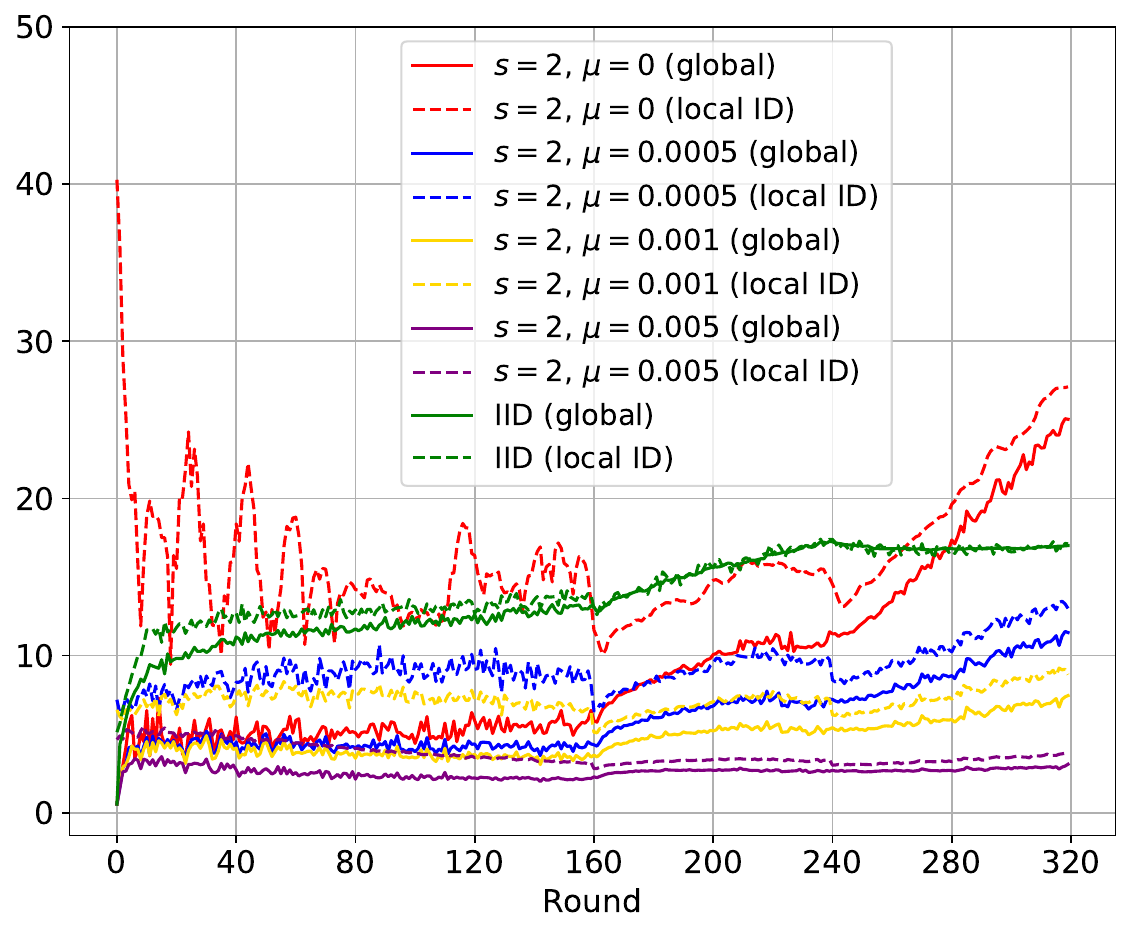}
        \label{ch2fig:fedfr-local-global}
        \vspace{-0.25 in}
        \caption{Feature norm means of local and global models}
    \end{subfigure}
    \caption{Effects of Feature Norm Regularization.}
    \label{ch2fig:fedfr-result}
\end{figure}

To mitigate feature norm discrepancy between local and global models, especially in scenarios of high heterogeneity that lead to elevated local model feature norms, we introduce an L2 feature norm regularization term with a hyperparameter $\mu$ to the local model training loss. This term is added alongside the cross-entropy loss $\mathcal{L}_{CE}$, and the combined loss $\mathcal{L}_{\mu}$ is formulated as follows:
\begin{equation}
    \mathcal{L}_{\mu}(x;\theta)=\mathcal{L}_{CE}(x;\theta)+\mu\,||f(x;\theta_{ext})||_{2}.
\end{equation}\label{ch2eqn:fedfr}
We apply $\mathcal{L}_{\mu}$ with various hyperparameters $\mu\in\{0, 0.0001, 0.0005, 0.001, 0.005, 0.01, 0.05\}$ across different $s\in\{2,3,5,10\}$ settings, and the results are illustrated in Figure~\ref{ch2fig:fedfr-result}. 

The optimal hyperparameter $\mu$ varies across different heterogeneity settings. For example, in Figure~\ref{ch2fig:fedfr-result} (a) , we observe that for $s=5$ and $s=2$, the optimal $\mu$ values are 0.05 and 0.005, respectively. In Figure~\ref{ch2fig:fedfr-result} (b), we illustrate the feature norm means of local and global models during training under high heterogeneity setting ($s$=2) for different $\mu$ values. 
Notably, at the optimal $\mu$ value of 0.005,  we observe a significant reduction in these discrepancies, aligning more closely with the ideal IID setting. Moreover, each feature norm mean of local and global models are noticeably lower than those in the ideal IID setting. This difference can be attributed to the decrease in the feature norm mean of local model, which subsequently impacts the feature norm mean of the global model. In summary, our findings underscore the performance improvements can be achieved by reducing discrepancies in feature norm means between the global and local models during the training process.

\section{FedFN: Federated Averaging with Feature Normalization Update}\label{ch2sec:method}
In this section, we present Federated Averaging with Feature Normalization (FedFN), which integrates feature normalization (FN) with FedAVG. We also apply the four-factor analysis conducted in Section~\ref{ch2sec:problem} to FedFN. Furthermore, the gradual reduction in weight norm bias within FedFN during training is detailed in Subsection~\ref{ch2app:add_exp}.
\subsection{FedFN Algorithm}\label{ch2subsec:method}
We revisit \emph{Federated Averaging with Feature Normalization Update} (FedFN), an extension of FedAVG enriched with Feature Normalization (FN) updates as discussed in \cite{dong2022spherefed}\footnote{SphereFed~\citep{dong2022spherefed} proposes an approach to FL where the classifier is initialized in an orthonormalized manner and kept frozen. Meanwhile, feature normalization is applied to train the local model, utilizing MSE loss. The comparison between FedFN and SphereFed can be found in Subsection~\ref{ch2app:modify_norm}.}. 
 FN eliminates feature norm bias within the local model by normalizing the feature vector, ensuring that the norm is consistently set to 1 for any input $x$. In FedFN, compared to FedAVG, the FN update modifies the logit vector $z$ of an input $x$, represented as $\hat{z}(x;\theta)=\theta_{cls}\frac{f(x;\theta_{ext})}{||f(x;\theta_{ext})||_{2}}$. Consequently, the gradient of $\theta_\text{cls}$ concerning the cross-entropy loss $\mathcal{L}_{CE}(\cdot)$ for FedFN is expressed as follows:
\vspace{-0.05 in}
\begin{equation*}
\nabla_{\theta_{cls}}\mathcal{L}_{CE}(x;\theta)=\nabla_{\hat{z}(x;\theta)}\mathcal{L}_{CE}(x;\theta)\frac{f(x;\theta_{ext})^{\top}}{||f(x;\theta_{ext})||_{2}}\in \mathds{R}^{C \times d}.
\end{equation*}
\begin{wraptable}[6]{r}{0.45\textwidth}
\centering
\vspace{-0.2in}
\caption{Accuracy of FedAvg and FedFN.}
    \begin{tabular}{c|cc}
    \toprule
     & FedAvg  & FedFN  \\ \midrule 
    $s$=10  & 81.97 & \textbf{83.80}  \\ 
    $s$=2 &  74.24 & \textbf{77.77} \\
    \bottomrule
    \end{tabular}
\label{ch2tab:acc_fedavg_fedfn}
\end{wraptable}
Unlike FedAVG, FedFN scales the gradient of $\theta_{cls}$ by dividing it by the feature vector norm. This scaling significantly impacts the gradient of $\theta_{cls}$ and, consequently, the applied learning rate. As a result of this influence, we conduct  a thorough fine-tuning for the learning rate for the FN update, leading FedFN to adopt a larger initial learning rate of 0.03, compared to the baseline rate of 0.01 in FedAVG.
These learning rates undergo careful selection through an extensive grid search, with detailed findings available in Subsection~\ref{ch2app:grid search}.  Table~\ref{ch2tab:acc_fedavg_fedfn} demonstrates significant accuracy improvement with FedFN compared to FedAVG.
\vspace{-0.1 in}
\subsection{4-Factor Analysis of FedFN}\label{ch2subsec:4-factor-fn}
\begin{figure}[htp]
    \centering
    \makebox[\textwidth]{
    \includegraphics[width=1.0\textwidth]{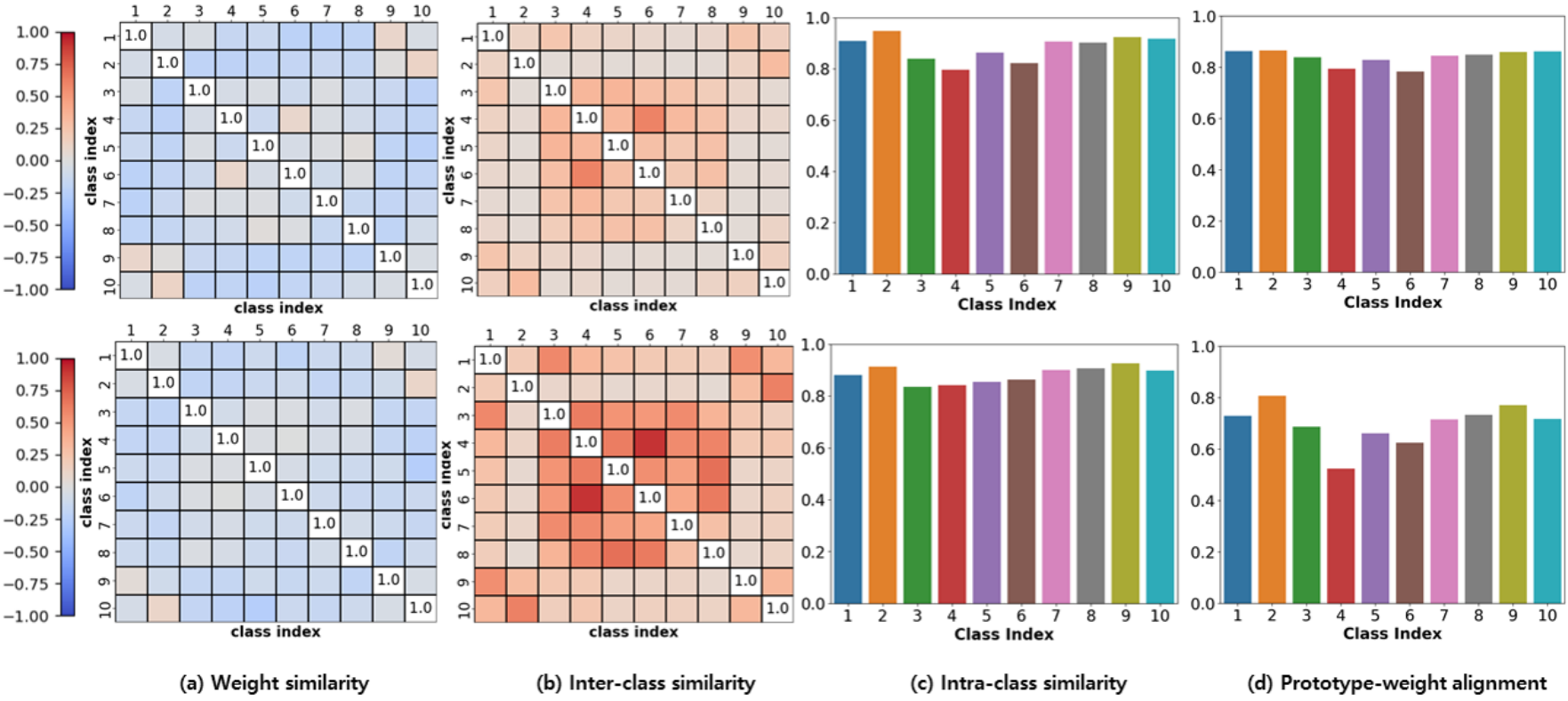}
    }
    \caption{4-Factor Analysis of FedFN; Results for $s$=10 (Top Row) and $s$=2 (Bottom Row).} 
    \label{ch2fig:4_factor_fedFN}
\end{figure}

We conduct four factor analysis on the improved global model compared to FedAVG, as reported in Table~\ref{ch2tab:acc_fedavg_fedfn}. The results are visualized in Figure \ref{ch2fig:4_factor_fedFN}. Similar to FedAVG, weight similarity remains robust even in high data heterogeneity. On the other hand, FedFN demonstrates significant improvement in inter-class similarity between prototypes compared to FedAVG, especially in the high data heterogeneity. In contrast to FedAVG, FedFN exhibits a slight decrease in intra-class similarity as data heterogeneity increases. This shift likely be attributed to the simultaneous improvement in inter-class similarity within FedFN, enabling finer class separation. While some degradation in prototype-weight alignment is observed in FedFN, FedFN consistently outperforms FedAVG, possibly due to its enhanced inter-class similarity. In summary, the superior inter-class similarity of FedFN contributes to its enhanced performance, resulting in improved discrimination and stability across various data heterogeneity.

\section{Experiment Results}\label{ch2sec:exp}

In this section, we present the experimental results of FedFN, demonstrating its compatibility with existing FL algorithms. Moreover, we investigate a comparative analysis with the pretrained model To assess the feasibility of implementing FedFN in the foundation models. Additionally, the application of FedFN for the personalized FL can be found in Subsection~\ref{ch2subsec:pfl_result}.

\subsection{Compatibility of FN with Existing FL Algorithms}\label{ch2subsec:gfl_compatible}
We assess the compatibility of the FN update module with existing FL algorithms, including FedAVG~\cite{mcmahan2017communication}, Scaffold~\cite{karimireddy2020scaffold}, and FedEXP~\cite{jhunjhunwala2023fedexp} (referred to as 
``Baseline'' algorithms). Additionally, we compare the results with those obtained by applying the BABU module~\citep{oh2021fedbabu}, which freezes the classifier part of the model, to the baseline algorithms. Table~\ref{ch2tab:gfl_acc} summarizes the accuracy comparison for different FL algorithms, including baseline, +BABU, and +FN, on VGG11 for CIFAR-10 and MobileNet for CIFAR-100. For CIFAR-10 and CIFAR-100, we use FN updates with initial learning rates of 0.03 and 0.5, respectively. 
While Scaffold generally demonstrates superior performance, it faces training failures in scenarios with high data heterogeneity, even when applying the BABU module (e.g., $s$=3). 
In contrast, the FN update consistently outperforms all algorithms, including baselines and those with the BABU module, across all heterogeneity settings, demonstrating its superiority.

\begin{table}[h!]
    \centering
    \caption{FL Accuracy Comparison for Baseline, +BABU, and +FN.}
    \small
    \resizebox{\textwidth}{!}{
    \begin{tabular}{cc|cccc|ccc}
      \toprule 
      \multirow{2.5}{*}{Algorithm} & \multirow{2.5}{*}{Module} & \multicolumn{4}{c|}{VGG11 on CIFAR-10} & \multicolumn{3}{c}{MobileNet on CIFAR-100} \\ \cmidrule{3-9}
      & & $s$=2 & $s$=3 & $s$=5 & $s$=10 & $s$=10 & $s$=50 & $s$=100 \\ \midrule
      & Baseline  & 73.62 & 77.70 & 81.62 & 82.13 & 37.25 & 42.90 & 43.36 \\ 
      FedAVG (2017) & + BABU & 73.11 & 76.57 & 81.22 & 81.89 & 43.20 & 39.70 & 39.59 \\ 
      & \textbf{+ FN} & \textbf{76.47} & \textbf{78.66} & \textbf{82.07} & \textbf{83.09} & \textbf{44.67} & \textbf{48.17} & \textbf{49.67}  \\  \midrule
      & Baseline & 77.07 & (\textit{Failed}) & \textbf{84.30} & 84.98 & 41.86 & 43.59 & 41.74 \\
      Scaffold (2020)   & + BABU & 77.17 & (\textit{Failed}) & 83.54 & 84.43 & 46.41 & 41.61 & 42.55 \\ 
       & \textbf{+ FN} & \textbf{77.96} & \textbf{79.24} & \textbf{84.40} & \textbf{85.54} & \textbf{49.42} & \textbf{50.42} & \textbf{52.10}  \\  \midrule
      & Baseline & 73.49 & 77.90 & 81.64 & 82.42 & 36.35 & 41.06 & 42.38\\
      FedEXP (2023) & +BABU & 72.58 & 77.59 & 81.07 & 81.96 & 43.38 & 40.73 & 39.04 \\ 
      & \textbf{+ FN} & \textbf{76.33} & \textbf{78.20} & \textbf{82.41} & \textbf{83.26} & \textbf{45.90} & \textbf{49.10} & \textbf{49.11}  \\  \bottomrule 
    \end{tabular}
    }
    \label{ch2tab:gfl_acc}
\end{table}
\vspace{-0.25 in}
\subsection{Comparative Analysis with Pretrained Model}\label{ch2subsec:gfl_foundation}
\vspace{-0.05 in}
To evaluate the feasibility of applying FedFN within the foundation models, we conducted experiments using a pretrained ResNet18 model on the CIFAR-10 dataset~\citep{yu2023federated, tan2022federated, chen2206importance, nguyen2210begin}. Specifically, we compare the performance of FedAVG, FedBABU, and FedFN when utilizing both pretrained and non-pretrained ResNet18 models. Table~\ref{ch2tab:gfl_foundation} presents the accuracy comparison for these algorithms. Across all experimental settings, FedFN consistently outperforms both FedBABU and FedAVG. Notably, FedAVG and FedBABU, particularly in scenarios with high data heterogeneity, exhibit significant performance declines when pretrained models are employed. In contrast, FedFN consistently improves with the application of pretrained models, utilizing FN updates with an initial learning rate of 0.1.

\begin{table}[htp]
    \centering
    \caption{Accuracy Comparison on CIFAR-10 on ResNet18.}
    \small
    \begin{tabular}{c|cccc|cccc}
    \toprule
    \multirow{2.5}{*}{Algorithm} & \multicolumn{4}{c|}{Pretrained=\emph{False}} & \multicolumn{4}{c}{Pretrained=\emph{True}} \\ \cmidrule{2-9} 
                                 & $s$=2 & $s$=3 & $s$=5 & $s$=10 & $s$=2 & $s$=3 & $s$=5 & $s$=10 \\ \midrule
    FedAVG   & 41.50 & 55.31 & 67.64 & 73.39 & 37.87 & 58.31 & 71.75 & 84.50 \\ 
    FedBABU  & 49.21 & 58.44 & 68.84 & 73.82 & 49.78 & 49.61 & 66.46 & 84.22 \\ 
    FedFN    & \textbf{55.17} & \textbf{60.47} & \textbf{77.12} & \textbf{81.26} & \textbf{56.84} & \textbf{76.84} & \textbf{80.02} & \textbf{84.99} \\ \bottomrule                            
    \end{tabular}    \label{ch2tab:gfl_foundation}
\end{table}

\section{Supplemental Evidence}
\label{ch2sec:suppl}

\subsection{Grid Search Result}
\label{ch2app:grid search}
To optimize the hyperparameters for FL, we conduct grid searches for the initial learning rate on both CIFAR-10 and CIFAR-100. For CIFAR-10, we explore learning rate $\eta$ in the range of $\{0.01, 0.03, 0.05, 0.1\}$, while for CIFAR-100, we consider $\eta$ of $\{0.1, 0.3, 0.5, 1.0\}$. Additionally, we performe grid searches to determine the optimal number of local epochs $E$, evaluating values in the set $\{1, 5, 10, 15, 20\}$ for both datasets. The optimal number of $E$ is found to be 15 for CIFAR-10 and 5 for CIFAR-100. Unless otherwise specified, the values determined through grid search for the hyperparameter $\eta$ are as follows: default initial learning rates are 0.01 for CIFAR-10 and 0.1 for CIFAR-100.
\subsubsection{FedBABU vs FedAVG vs FedFN}
We conduct grid searches for FedAVG~\cite{mcmahan2017communication}, FedBABU~\cite{oh2021fedbabu}, and FedFN. Table ~\ref{ch2apptab:grid_vgg} to ~\ref{ch2apptab:grid_resnet_pretrained} present the grid search results for VGG on CIFAR-10, MobileNet on CIFAR-100, ResNet18 on CIFAR-10, and Pretrained ResNet18 on CIFAR-10, respectively. In these tables, we abbreviate FedAVG, FedBABU, and FedFN as AVG, BABU, and FN, respectively, and indicate cases where the final global model fails to converge during training with a ``-" in the respective table cells. The determined optimal hyperparameter $\eta$ values for FedFN are as follows: 0.03, 0.5, 0.1, and 0.1 for Table ~\ref{ch2apptab:grid_vgg} to ~\ref{ch2apptab:grid_resnet_pretrained}, respectively.
\begin{table}[htp]
    \centering
    \caption{Grid Search Results for VGG11 on CIFAR-10.}  
    \resizebox{\textwidth}{!}{
        \begin{tabular}{c|ccc|ccc|ccc|ccc|ccc}
        \toprule
        \multirow{2.5}{*} {\emph{$\eta$=0.01}} & \multicolumn{3}{c|}{$E$=1}                      & \multicolumn{3}{c|}{$E$=5} & \multicolumn{3}{c|}{$E$=10} & \multicolumn{3}{c|}{$E$=15} & \multicolumn{3}{c}{$E$=20}  \\ \cmidrule{2-16}         
        & BABU & AVG & FN & BABU & AVG & FN & BABU & AVG & FN & BABU & AVG & FN & BABU & AVG & FN \\ \midrule  
        $s$=10 & 49.51 & 52.98  & 52.02 & 82.43 & 82.68 & 82.82 & 82.56 & 83.03        & 82.7  & \textbf{82.16} & \textbf{81.97} & 82.11 & 81.82 & 81.66 & 82.42\\ 
        
        $s$=5 & 42.47 & 46.28 & 50.26 & 79.39 & 80.42 & 79.76 & 81.13 & 81.36 & 82.09 & \textbf{81.04} & \textbf{81.08} & 81.74 & 80.82 & 81.35 & 81.3   \\ 
        
        $s$=3 & 29.45 & 38.32 & 48.24 & 71.64 & 73.83 & 72.94 & 78.39 & 77.39 & 77.49 & \textbf{77.73} & \textbf{77.29}        & 78.37 & 78.00 & 78.03 & 78.48   \\ 
        
        $s$=2 & 24.21 & 30.61 & 46.24 & 56.37 & 63.78 & 62.5 & 73.31 & 73.24 & 73.79  & \textbf{75.05} & \textbf{74.24} & 75.34 & 75.25 & 74.98 & 76.33   \\ \midrule
            
        \multirow{2.5}{*} {\emph{$\eta$=0.03}} & \multicolumn{3}{c|}{$E$=1}                      & \multicolumn{3}{c|}{$E$=5} & \multicolumn{3}{c|}{$E$=10} & \multicolumn{3}{c|}{$E$=15} & \multicolumn{3}{c}{$E$=20}  \\ \cmidrule{2-16}         
        & BABU & AVG & FN & BABU & AVG & FN & BABU & AVG & FN & BABU & AVG & FN & BABU & AVG & FN \\ \midrule  
        $s$=10 & 63.23   & 66.81 & 64.32 & 84.23 & 84.32 & 83.76 & 84.10 & 84.13 & 83.80 & 83.61 & 84.38 & \textbf{83.80} & 83.82 & 83.42 & 83.68   \\ 
        
        $s$=5 & 54.86 & 59.21 & 57.94 & 82.14 & 82.56 & 81.43 & 82.19 & 82.39 & 82.34  & 83.17 & 82.92 & \textbf{82.43} & 82.49 & 82.66 & 82.79   \\ 
        
        $s$=3 & 34.12& 49.90 & 52.11 & 75.71 & - & 74.01 & - & - & 78.51 & - & - & \textbf{78.93} & - & - & 78.77 \\ 
        
        $s$=2 & - & 34.80 & 46.54 & 50.80 & 62.80 & 66.60 & 74.27 & 75.08 & 76.26  & 74.57 & - & \textbf{77.77} & 77.36 & - & 77.67   \\ \midrule
        
        \multirow{2.5}{*} {\emph{$\eta$=0.05}} & \multicolumn{3}{c|}{$E$=1}                      & \multicolumn{3}{c|}{$E$=5} & \multicolumn{3}{c|}{$E$=10} & \multicolumn{3}{c|}{$E$=15} & \multicolumn{3}{c}{$E$=20}  \\ \cmidrule{2-16}         
        & BABU & AVG & FN & BABU & AVG & FN & BABU & AVG & FN & BABU & AVG & FN & BABU & AVG & FN \\ \midrule    
        $s$=10 & 69.58 & 70.99 & 68.22 & 84.49 & 84.51 & 83.60 & 84.22 & 84.30 & 83.35 & 84.19 & 84.18 & 83.81 & 83.94 & - & 82.92 \\ 
        
        $s$=5 & 56.65 & 62.67 & 59.31 & 81.74 & 81.62 & 81.87 & 82.77 & - & 82.57  & - & - & 82.63 & 82.90 & - & 82.52   \\ 
        
        $s$=3 & 28.78 & 50.34 & 52.09 & - & - & 74.69 & - & - & 77.45  & - & - & 78.69 & - & - & 78.39  \\ 
        
        $s$=2 & - & 31.77 & 40.34 & - & 61.66 & 62.52 & 71.52 & - & 71.6  & - & - & 76.06 & - & - & 76.77 \\ \midrule

        \multirow{2.5}{*} {\emph{$\eta$=0.1}} & \multicolumn{3}{c|}{$E$=1}                      & \multicolumn{3}{c|}{$E$=5} & \multicolumn{3}{c|}{$E$=10} & \multicolumn{3}{c|}{$E$=15} & \multicolumn{3}{c}{$E$=20}  \\ \cmidrule{2-16}         
        & BABU & AVG & FN & BABU & AVG & FN & BABU & AVG & FN & BABU & AVG & FN & BABU & AVG & FN \\ \hline     

        $s$=10 & 74.58 & 72.59 & 70.30 & 83.90 & - & 83.72 & 84.05 & - & 82.71  & - & - & 82.71 & - & - & 82.35   \\ 
        
        $s$=5 & 60.35 & 61.88 & 59.01 & - & - & 80.08 & - & - & 81.79  & - & - & 81.54 & - & - & 81.77   \\ 
        
        $s$=3 & 25.06 & 28.36 & 47.54 & - & - & 65.51 & - & - & 76.06  & - & - & 76.07 & -& -  & 76.19   \\ 
        
        $s$=2 &  -  & - & 41.27 &  - & - & 58.43 & - & - & 70.86  & - & - & 73.10 & - & - & 74.90   \\ \bottomrule
     
        \end{tabular}
        } 
\label{ch2apptab:grid_vgg}
\end{table}

\begin{table}[htp]
    \centering
    \caption{Grid Search Results for MobileNet on CIFAR-100.}
    \resizebox{\textwidth}{!}{
        \begin{tabular}{c|ccc|ccc|ccc|ccc|ccc}
        \toprule
        \multirow{2.5}{*} {\emph{$\eta$=0.1}} & \multicolumn{3}{c|}{$E$=1}                      & \multicolumn{3}{c|}{$E$=5} & \multicolumn{3}{c|}{$E$=10} & \multicolumn{3}{c|}{$E$=15} & \multicolumn{3}{c}{$E$=20}  \\ \cmidrule{2-16}         
        & BABU & AVG & FN & BABU & AVG & FN & BABU & AVG & FN & BABU & AVG & FN & BABU & AVG & FN \\ \midrule

        $s$=100 & 41.32 & 41.57 & 76.61 & \textbf{40.49} & \textbf{43.19} & 51.18 & 37.03 & 38.13 & 45.26 & 36.16 & 36.10 & 44.03 & 35.91 & 36.19 & 42.25   \\ 
        
        $s$=50 & 40.86 & 38.35 & 7.51 & \textbf{40.76} & \textbf{40.63} & 51.11 & 36.79 & 37.77 & 46.74  & 36.10 & 37.54 & 44.47 & 34.93 & 36.05 & 43.61   \\ 
        
        $s$=10 & 35.95 & 27.02 & 4.85 & \textbf{45.56} & \textbf{36.60} & 43.18 & 41.01 & 34.67 & 47.63  & 37.87 & 35.53 & 45.50 & 38.28 & 35.94 & 47.41   \\ 
        \midrule
        \multirow{2}{*} {\emph{$\eta$=0.3}} & \multicolumn{3}{c|}{$E$=1}                      & \multicolumn{3}{c|}{$E$=5} & \multicolumn{3}{c|}{$E$=10} & \multicolumn{3}{c|}{$E$=15} & \multicolumn{3}{c}{$E$=20}  \\ \cmidrule{2-16}         
        & BABU & AVG & FN & BABU & AVG & FN & BABU & AVG & FN & BABU & AVG & FN & BABU & AVG & FN \\ \midrule

        $s$=100 & 46.65 & 41.76 & 19.52 & 40.57 & 46.67 & 50.80 & 37.79 & 41.28 & 44.34  & 36.16 & 38.29 & 41.18 & 35.91 & 38.22 & 39.38   \\ 
        
        $s$=50 & 44.03 & 37.83 & 19.63 & 40.02 & 45.28 & 49.92 & 37.28 & 42.65 & 43.62  & 36.10 & 41.66 & 41.71 & 34.93 & 40.61 & 38.08   \\   
        
        $s$=10 & 37.37 & 26.66 & 16.66 & 47.07 & 33.78 & 48.89 & 44.57 & 36.42 & 46.77  & 37.87 & - & 47.95 & 38.28 & 35.12 & 48.39   \\         
        \midrule
        \multirow{2}{*} {\emph{$\eta$=0.5}} & \multicolumn{3}{c|}{$E$=1}                      & \multicolumn{3}{c|}{$E$=5} & \multicolumn{3}{c|}{$E$=10} & \multicolumn{3}{c|}{$E$=15} & \multicolumn{3}{c}{$E$=20}  \\ \cmidrule{2-16}         
        & BABU & AVG & FN & BABU & AVG & FN & BABU & AVG & FN & BABU & AVG & FN & BABU & AVG & FN \\ \midrule

        $s$=100 & 47.18 & 35.81 & 28.46 & 39.33 & - & \textbf{51.51} & 38.23 & - & 44.54  & 35.59 & - & 39.56 & 35.99 & - & 38.20   \\ 
        
        $s$=50 & 47.53 & 36.28 & 25.71 & 40.42 & - & \textbf{50.6} & 37.71 & 44.76 & 42.47  & 37.66 & - & 40.96 & 38.38 & 43.15 & 38.42   \\   
        
        $s$=10 & 38.91 & 19.51 & 21.02 & 46.48 & 11.83 & \textbf{46.87} & 45.74 & 23.02 & 44.25  & 43.61 & 24.74 & 45.05 & 42.62 & 30.68 & 43.28   \\ 
        \midrule
        \multirow{2.5}{*} {\emph{$\eta$=1.0}} & \multicolumn{3}{c|}{$E$=1}                      & \multicolumn{3}{c|}{$E$=5} & \multicolumn{3}{c|}{$E$=10} & \multicolumn{3}{c|}{$E$=15} & \multicolumn{3}{c}{$E$=20}  \\ \cmidrule{2-16}         
        & BABU & AVG & FN & BABU & AVG & FN & BABU & AVG & FN & BABU & AVG & FN & BABU & AVG & FN \\ \midrule

        $s$=100 & 46.92 & 31.15 & 36.07 & 39.61 & 34.66 & 50.95 & 37.12 & 42.06 & 42.17  & 37.83 & 46.53 & 39.64 & 37.64 & 33.12 & 39.97   \\ 
        
        $s$=50 & 45.24 & 26.64 & 32.92 & 40.11 & 28.02 & 48.46 & 38.76 & 39.47 & 40.82 & 41.13 & 42.79 & 41.07 & 39.10 & 42.41 & 35.88   \\ 
        
        $s$=10 & 38.59 & 8.39 & 26.63 & 45.14 & - & 39.64 & 44.03 & - & 38.64  & 40.51 & - & 36.05 & 39.67 & 12.26 & 37.41   \\ \bottomrule

        \end{tabular}
        }
\label{ch2apptab:grid_mobile}
\end{table}
\newpage

\begin{table}[htp]
    \centering
    \caption{Grid Search Results for ResNet18 on CIFAR-10.} 
    \resizebox{\textwidth}{!}{
        \begin{tabular}{c|ccc|ccc|ccc|ccc}
        \toprule
        \multirow{2.5}{*} {\emph{$\eta$=0.01}} & \multicolumn{3}{c|}{$E$=5} & \multicolumn{3}{c|}{$E$=10} & \multicolumn{3}{c|}{$E$=15} & \multicolumn{3}{c}{$E$=20}  \\ \cmidrule{2-13}         
        & BABU & AVG & FN & BABU & AVG & FN & BABU & AVG & FN & BABU & AVG & FN  \\ \midrule 
        $s$=10 &  75.32  & 76.07 & 78.00 &  74.53 & 74.17 & 77.72 & \textbf{73.55} & \textbf{73.85} & 77.80  & 73.16 & 72.32 & 77.18    \\
        
        $s$=5 &  68.02  & 69.34 & 72.14 &  67.54 & 68.13 & 71.17 & \textbf{68.36} & \textbf{67.71} & 73.06  & 68.56 & 66.48 & 74.74    \\ 
        
        $s$=3 &  58.60  & 58.93 & 64.94 &  57.87 & 58.45 & 66.45 & \textbf{57.07} & \textbf{57.38} & 64.23  & 60.93 & 55.13 & 66.58    \\ 
        
        $s$=2 &  46.79  & 41.15 & 45.50 &  45.50 & 41.92 & 47.22 & \textbf{49.21} & \textbf{45.92} & 50.81  & 47.50 & 45.23 & 47.99    \\ \midrule
            
        \multirow{2}{*} {\emph{$\eta$=0.03}} & \multicolumn{3}{c|}{$E$=5} & \multicolumn{3}{c|}{$E$=10} & \multicolumn{3}{c|}{$E$=15} & \multicolumn{3}{c}{$E$=20}  \\ \cmidrule{2-13}         
        & BABU & AVG & FN & BABU & AVG & FN & BABU & AVG & FN & BABU & AVG & FN  \\ \midrule  
        $s$=10 &  78.64  & 78.77 & 80.63 &  77.76 & 77.73 & 79.86 & 76.83 & 76.57 & 79.34  & 76.02 & 76.87 & 78.78    \\
        
        $s$=5 &  70.04  & 71.50 & 75.57 & 72.01 & 71.84 & 75.77 & 71.45 & 71.72 & 75.99  & 72.64 & 71.83 & 75.94    \\ 
        
        $s$=3 &  57.53  & 61.54 & 67.55 & 56.60 & 59.58 & 66.88 & 58.06 & 60.53 & 68.63  & 60.80 & 62.07 & 68.65    \\ 
        
        $s$=2 &  47.56  & 40.02 & 52.70 & 46.26 & 46.27 & 50.25 & 50.16 & 45.60 & 48.52  & 50.60 & 46.01 & 52.37    \\ \midrule
        
        \multirow{2.5}{*} {\emph{$\eta$=0.05}} & \multicolumn{3}{c|}{$E$=5} & \multicolumn{3}{c|}{$E$=10} & \multicolumn{3}{c|}{$E$=15} & \multicolumn{3}{c}{$E$=20}  \\ \cmidrule{2-13}         
        & BABU & AVG & FN & BABU & AVG & FN & BABU & AVG & FN & BABU & AVG & FN  \\ \midrule 
        $s$=10 &  80.17  & 79.41 & 81.21 & 78.22 & 77.93 & 80.67 & 77.87 & 78.50 & 80.98  & 77.16 & 77.77 & 79.59    \\
        
        $s$=5 &  70.65  & 72.44 & 77.25 & 72.08 & 71.08 & 76.65 & 73.21 & 71.85 & 77.41  & 74.16 & 72.52 & 77.19    \\ 
        
        $s$=3 &  56.26  & 62.22 & 66.91 & 56.22 & 61.38 & 67.01 & 59.22 & 60.85 & 64.09  & 61.79 & 57.00 & 64.51    \\ 
        
        $s$=2 &  53.18  & 43.99 & 50.88 & 45.76 & 45.26 & 51.06 & 46.53 & 40.44 & 51.28  & 46.49 & 41.58 & 56.41    \\ \midrule

        \multirow{2.5}{*} {\emph{$\eta$=0.1}} & \multicolumn{3}{c|}{$E$=5} & \multicolumn{3}{c|}{$E$=10} & \multicolumn{3}{c|}{$E$=15} & \multicolumn{3}{c}{$E$=20}  \\ \cmidrule{2-13}         
        & BABU & AVG & FN & BABU & AVG & FN & BABU & AVG & FN & BABU & AVG & FN  \\ \hline  
        $s$=10 &  80.83  & 80.13  & 82.37 & 80.04 & 79.87 & 82.01 & 79.50 & 78.57 & \textbf{81.19}  & 78.64 & 78.21 & 80.37    \\
        
        $s$=5 & 72.81  & 71.56 & 77.91 & 74.37 & 69.65 & 76.89 & 74.71 & 69.66 & \textbf{75.98}  & 73.97 & 71.65 & 77.47    \\ 
        
        $s$=3 &  57.14  & 59.65 & 63.75 & 57.33 & 52.71 & 61.12 & 60.34 & 58.13 & \textbf{65.09}  & 58.46 & 54.72 & 67.08    \\ 
        
        $s$=2 & 47.35 & 50.51 & 50.58 & 42.78 & 41.36 & 51.19 & 47.63 & 44.81 & \textbf{50.12}  & 40.23 & 39.03 & 49.05    \\ \bottomrule
     
        \end{tabular}
        }
        \label{ch2apptab:grid_resnet}
\end{table}

\begin{table}[htp]
    \centering
    \caption{Grid Search Results for Pretrained ResNet18 on CIFAR-10.}
    \resizebox{\textwidth}{!}{
        \begin{tabular}{c|ccc|ccc|ccc|ccc}
        \toprule
        \multirow{2.5}{*} {\emph{$\eta$=0.01}} & \multicolumn{3}{c|}{$E$=5} & \multicolumn{3}{c|}{$E$=10} & \multicolumn{3}{c|}{$E$=15} & \multicolumn{3}{c}{$E$=20}  \\ \cmidrule{2-13}         
        & BABU & AVG & FN & BABU & AVG & FN & BABU & AVG & FN & BABU & AVG & FN  \\ \hline  
        $s$=10 &  83.85  & 83.80 & 84.65 & 84.06 & 84.29 & 84.20 & \textbf{84.02} & \textbf{83.96} & 84.43  & 84.43 & 84.33 & 84.81    \\
        
        $s$=5 &  68.83  & 73.31 & 74.90 & 69.73 & 69.24 & 74.72 & \textbf{69.44} & \textbf{70.77} & 75.26  & 68.58 & 70.89 & 76.75    \\ 
        
        $s$=3 & 54.21 & 58.52 & 63.87 & 51.85 & 48.52 & 61.76 & \textbf{48.96} & \textbf{57.85} & 60.72  & 47.84 & 51.11 & 62.06    \\ 
        
        $s$=2 & 48.21  & 33.84 & 47.74 & 43.02 & 35.91 & 47.30 & \textbf{49.00} & \textbf{39.04} & 50.82  & 44.91 & 36.95 & 49.98    \\ \midrule
            
        \multirow{2.5}{*} {\emph{$\eta$=0.03}} & \multicolumn{3}{c|}{$E$=5} & \multicolumn{3}{c|}{$E$=10} & \multicolumn{3}{c|}{$E$=15} & \multicolumn{3}{c}{$E$=20}  \\ \cmidrule{2-13}         
        & BABU & AVG & FN & BABU & AVG & FN & BABU & AVG & FN & BABU & AVG & FN  \\ \hline  
        $s$=10 &  84.86  & 82.73 & 85.28 & 85.02 & 82.00 & 85.07 & 84.85 & 82.71 & 85.10  & 84.82 & 83.04 & 85.27    \\
        
        $s$=5 & 67.27  & 60.13 & 72.24 & 68.64 & 59.54 & 70.55 & 67.81 & 65.43 & 73.14  & 68.23 & 68.91 & 73.27    \\ 
        
        $s$=3 & 53.48 & 47.72 & 54.61 & 53.74 & 50.32 & 56.27 & 54.18 & 53.78 & 59.92  & 56.35 & 53.82 & 60.28    \\ 
        
        $s$=2 & 43.10  & 36.69 & 46.89 &  45.80 & 32.61 & 49.99 & 51.62 & 39.08 & 45.50 & 55.78 & 41.35 & 51.88    \\ \midrule
        
        \multirow{2.5}{*} {\emph{$\eta$=0.05}} & \multicolumn{3}{c|}{$E$=5} & \multicolumn{3}{c|}{$E$=10} & \multicolumn{3}{c|}{$E$=15} & \multicolumn{3}{c}{$E$=20}  \\ \cmidrule{2-13}         
        & BABU & AVG & FN & BABU & AVG & FN & BABU & AVG & FN & BABU & AVG & FN  \\ \hline  
        $s$=10 &  86.06  & 71.37 & 85.84 & 85.71 & 75.94 & 85.52 & 85.31 & 76.50 & 84.73  & 85.41 & 77.77 & 85.11    \\
        
        $s$=5 &  69.47  & 65.13 & 69.89 & 69.72 & 60.01 & 71.02 & 69.49 & 62.47 & 73.75  & 70.17 & 68.17 & 73.24    \\ 
        
        $s$=3 & 57.91 & 52.70 & 53.10 & 55.63 & 49.75 & 54.86 & 62.72 & 54.51 & 63.03  & 65.22 & 52.00 & 65.71    \\ 
        
        $s$=2 & 44.49 & 42.07 & 45.55 & 43.65 & 43.69 & 45.46 & 48.36 & 37.11 & 52.96  & 50.87 & 41.37 & 51.87    \\ \midrule

        \multirow{2.5}{*} {\emph{$\eta$=0.1}} & \multicolumn{3}{c|}{$E$=5} & \multicolumn{3}{c|}{$E$=10} & \multicolumn{3}{c|}{$E$=15} & \multicolumn{3}{c}{$E$=20}  \\ \cmidrule{2-13}         
        & BABU & AVG & FN & BABU & AVG & FN & BABU & AVG & FN & BABU & AVG & FN  \\ \hline  
        $s$=10 &  85.89  & 72.78 & 86.02 & 85.67 & 74.95 & 85.65 & 84.98 & 78.00 & \textbf{85.45}  & 85.25 & 78.73 & 85.20    \\
        
        $s$=5 &  68.32  & 65.71 & 74.01 & 70.17 & 64.67 & 73.20 & 75.10 & 67.00 & \textbf{73.43}  & 73.03 & 71.16 & 77.12    \\ 
        
        $s$=3 & 62.57 & 46.52 & 66.26 & 64.15 & 59.18 & 74.36 & 63.56 & 56.59 & \textbf{73.85}  & 66.52 & 59.25 & 73.86    \\ 
        
        $s$=2 & 45.93 & 38.19 & 41.82 & 45.26 & 38.53 & 50.13 & 52.22 & 38.06 & \textbf{50.94}  & 55.18 & 47.67 & 53.42    \\ \bottomrule
     
        \end{tabular}
        }
        \label{ch2apptab:grid_resnet_pretrained}
\end{table}
\clearpage
\subsubsection{SphereFed}
We conduct grid searches for SphereFed (CE) and SphereFed (MSE)~\citep{dong2022spherefed} for MobileNet on CIFAR-100. In Table ~\ref{ch2apptab:grid_sphere}, we conduct a $\eta$ search within the range $\{0.1, 0.3, 0.5, 1.0\}$ for both SphereFed (CE) and SphereFed (MSE). However, SphereFed (MSE) exhibit suboptimal performance within this range. Consequently, in Table ~\ref{ch2apptab:grid_sphere_additional}, we extend the $\eta$ search exclusively for SphereFed (MSE) to include values $\{1.5, 3.0, 4.5, 5.0\}$. In summary, the determined optimal hyperparameter $\eta$ values are 0.03 for SphereFed (CE) and 4.5 for SphereFed (MSE).

\begin{table}[htp]
    \centering
        \caption{Grid Search Results for SphereFed with MobileNet on CIFAR-100.}
        \resizebox{0.6\textwidth}{!}{
        \begin{tabular}{c|cc|cc|cc}
        \toprule
        \multirow{2.5}{*} {\emph{$\eta$=0.1}} & \multicolumn{2}{c|}{$E$=1}                      & \multicolumn{2}{c|}{$E$=5} & \multicolumn{2}{c}{$E$=10} \\ \cmidrule{2-7}         
        & CE & MSE & CE & MSE & CE & MSE  \\  
        \midrule
        $s$=100 & 12.96 & 1.37 & 44.92 & 2.16 & 39.50 & 4.25 \\ 
        $s$=50 & 12.21 & 1.30 & 43.38 & 1.73 & 38.65 & 3.84 \\ 
        $s$=10 & 6.67 & 1.46 & 33.32 & 2.34 & 30.02 & 3.81     \\ \midrule        
        \multirow{2.5}{*} {\emph{$\eta$=0.3}} & \multicolumn{2}{c|}{$E$=1}                      & \multicolumn{2}{c|}{$E$=5} & \multicolumn{2}{c}{$E$=10} \\ \cmidrule{2-7}       
        & CE & MSE & CE &  MSE & CE & MSE  \\ \hline 
        $s$=100 & 24.62 & 1.39 & \textbf{44.96} & 5.25 & 40.47 & 19.41     \\ 
        $s$=50 & 23.56 & 1.44 & \textbf{43.68} & 5.89 & 38.21 & 19.65     \\   
        $s$=10 & 14.91 & 1.40 & \textbf{40.39} & 4.46 & 33.53 & 15.69     \\ \midrule        
        \multirow{2.5}{*} {\emph{$\eta$=0.5}} & \multicolumn{2}{c|}{$E$=1}                      & \multicolumn{2}{c|}{$E$=5} & \multicolumn{2}{c}{$E$=10} \\ \cmidrule{2-7}       
        & CE & MSE & CE &  MSE & CE & MSE  \\ \hline         
        $s$=100 & 32.12 & 1.46 & 45.47 & 12.30 & 40.11 & 36.38     \\    
        $s$=50 & 28.45 & 1.43 & 44.97 & 11.76 & 39.49 & 34.07     \\   
        $s$=10 & 18.70 & 1.61 & 37.02 & 8.32 & 33.83 & 24.65     \\ \midrule
        \multirow{2}{*} {\emph{$\eta$=1.0}} & \multicolumn{2}{c|}{$E$=1}                      & \multicolumn{2}{c|}{$E$=5} & \multicolumn{2}{c}{$E$=10} \\ \cmidrule{2-7}       
        & CE & MSE & CE &  MSE & CE & MSE  \\ \midrule
        $s$=100 & 38.98 & 2.33 & 46.45 & 29.94 & 41.58 & 40.02 \\
        $s$=50 & 35.69 & 2.27 & 43.95 & 29.39 & 41.20 & 40.87     \\ 
        $s$=10 & 26.40 & 1.70 & 36.72 & 23.38 & 35.73 & 32.64     \\ \bottomrule
        \end{tabular}
        }
        \label{ch2apptab:grid_sphere}
\end{table}

\begin{table}[htp]
    \centering
    \caption{Additional Grid Search Results for SphereFed (MSE) with MobileNet on CIFAR-100.}
    \resizebox{0.3\textwidth}{!}{
        \begin{tabular}{cccc}
        \toprule
        \multirow{2.5}{*} {\emph{$\eta$=1.5}} & {$E$=1}                      & {$E$=5} & {$E$=10} \\ \cmidrule{2-4}         
        &  MSE &   MSE &  MSE  \\ \midrule

        $s$=100 & 2.84  & 44.40 & 40.32     \\ 
        
        $s$=50 & 2.81  & 41.91 & 42.18     \\  
        
        $s$=10 & 2.69  & 33.68 & 35.66     \\ \midrule
        \multirow{2.5}{*} {\emph{$\eta$=3.0}} & {$E$=1}                      & {$E$=5} & {$E$=10} \\ \cmidrule{2-4}         
        &  MSE &   MSE &  MSE  \\ \midrule

        $s$=100 & 6.37  & 47.44 & 42.15     \\ 
        
        $s$=50 & 5.05  & 45.56 & 42.61     \\  
        
        $s$=10 & 4.87  & 40.84 & 34.68     \\ \midrule

        \multirow{2.5}{*} {\emph{$\eta$=4.5}} & {$E$=1}                      & {$E$=5} & {$E$=10} \\ \cmidrule{2-4}         
        &  MSE &   MSE &  MSE  \\ \midrule        
         
        $s$=100 & 9.43  & \textbf{49.51} & 43.19     \\ 
        
        $s$=50 & 8.19  & \textbf{48.05} & 42.42     \\ 
        
        $s$=10 & 7.50  & \textbf{43.42} & 37.93     \\ \midrule

        \multirow{2.5}{*} {\emph{$\eta$=5.0}} & {$E$=1}                      & {$E$=5} & {$E$=10} \\ \cmidrule{2-4}         
        &  MSE &   MSE &  MSE  \\ \midrule         
        $s$=100 & 10.58  & 49.29 & 36.07     \\ 
        
        $s$=50 & 10.95  & 48.08 & 38.49     \\  
        
        $s$=10 & 8.26  & 43.26 & 21.73     \\ \bottomrule
        \end{tabular}
        }
        \label{ch2apptab:grid_sphere_additional}
\end{table}
\clearpage
\subsection{Additional Experiments}
\label{ch2app:add_exp}
\subsubsection{Weight Norm Does Not Matter on FedFN}
We investigate the impact of increased data heterogeneity on weight norm disparity in the case of FedFN. Focusing on the $s$=2 setting, representing the most heterogeneous scenario, as shown in Figure~\ref{ch2appfig:weight_norm_bias_fn}, we explore the evolution of weight norms between global and local models. Initially, during early training stages, a noticeable weight norm bias emerges within local models, favoring seen (ID) classes over unseen (OOD) classes. However, this bias progressively diminishes as the learning rate decreases, eventually aligning local models with the weight norm mean of the global model.

Additionally, we analyze variations in weight norm means between local models on the ID classes and the global model on the total classes under varying data heterogeneity. Specifically, we consider $s \in \{2, 5, 10\}$ and IID (Exactly class balanced data distribution across clients) on the CIFAR-10 dataset. As illustrated in Figure~\ref{ch2appfig:weight_norm_disparity_fn}, during initial training phases, the norm disparity between the global model and local models increases with data heterogeneity (i.e., IID $\rightarrow$ $s$=10 $\rightarrow$ $s$=5 $\rightarrow$ $s$=2). However, this disparity gradually diminishes across all settings as the learning rate decreases.

\begin{figure}[h!]
    \centering
    \makebox[\textwidth]{
    \includegraphics[width=0.6\textwidth]{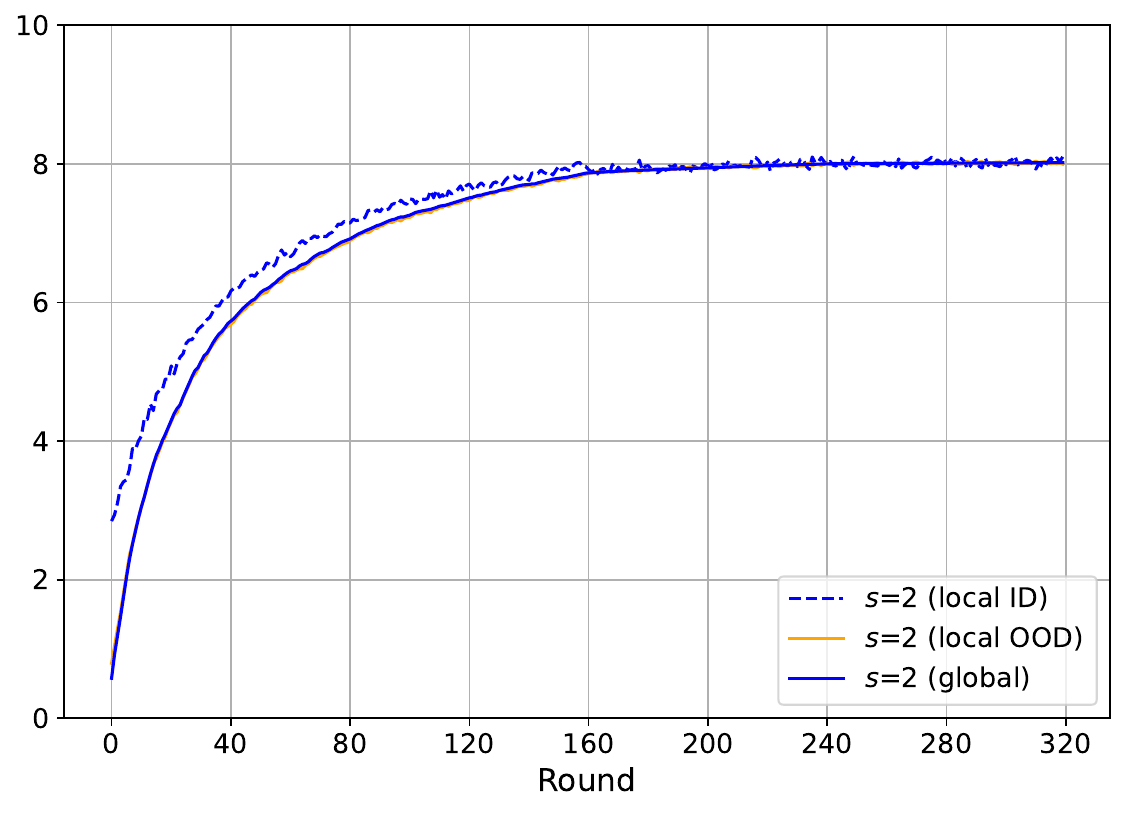}
    }
    \caption{Discrepancy in Weight Norms between Local and Global Models in the $s$=2 Setting.}
    \label{ch2appfig:weight_norm_bias_fn}
\end{figure}

\begin{figure}[h!]
    \centering
    \makebox[\textwidth]{
    \includegraphics[width=0.8\textwidth]{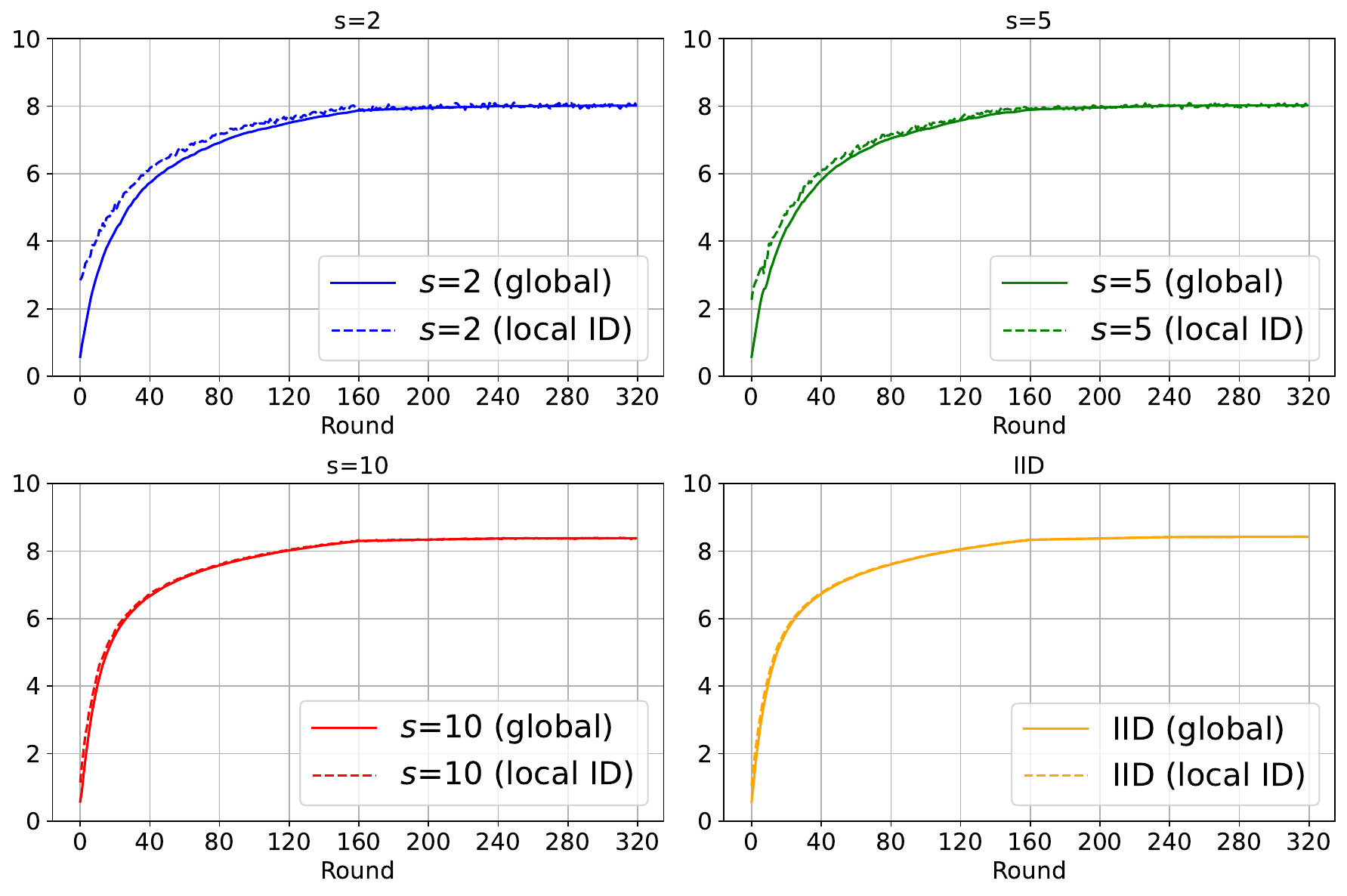}
    }
    \caption{Discrepancy in Weight Norms between Local and Global Models in Various Heterogeneity Settings.}
    \label{ch2appfig:weight_norm_disparity_fn}
\end{figure}

\newpage
\subsection{Additional GFL Results}\label{ch2app:add_gfl_result}

In the main paragraph of Section~\ref{ch2sec:exp}, we report results exclusively in a balanced environment, specifically in the sharding setting. In Table~\ref{ch2tab:gfl_acc_all}, results are presented, encompassing an unbalanced environment, specifically in the LDA setting. The reported outcomes stem from implementations with a consistent and identical seed, deviating from the context of the paragraph. In contrast to the main paragraph, we observe that in the MobileNet on CIFAR-100 environment with the $s=10$ setting, all baseline algorithms exhibit notably lower performance for the seed used in Table~\ref{ch2tab:gfl_acc_all}. However, when applying the BABU or FN module, the performance difference is less pronounced, showing a more stable outcome compared to the main paragraph. Furthermore, consistently across all settings, applying the FN module to the baseline consistently demonstrates superior performance.

\begin{table}[h!]
    \centering
    \caption{FL Accuracy Comparison for Baseline, +BABU, and +FN.}
    \small
    \resizebox{1.0\textwidth}{!}{
    \begin{tabular}{cc|cccc|cccc|ccc|cccc}
      \toprule 
      \multirow{2.5}{*}{Algorithm} & \multirow{2.5}{*}{Module} & \multicolumn{8}{c|}{VGG11 on CIFAR-10} & \multicolumn{7}{c}{MobileNet on CIFAR-100} \\ \cmidrule{3-17}
      & & $s$=2 & $s$=3 & $s$=5 & $s$=10 & $\alpha$=0.1 & $\alpha$=0.3 & $\alpha$=0.5 & $\alpha$=1.0 & $s$=10 & $s$=50 & $s$=100 & $\alpha$=0.1 & $\alpha$=0.3 & $\alpha$=0.5 & $\alpha$=1.0 \\ \midrule
      & Baseline  & 73.83 & 78.99 & 81.40 & 82.79 & 73.16 & 80.77 & 81.63 & 82.45 & 26.64 & 40.18 & 41.50 & 42.66 & 42.68 & 43.82 & 41.57  \\ 
      FedAVG  & + BABU & 74.31 & 78.88 & 81.17 & 82.37  & 72.42 & 80.59 & 80.94 & 82.93 &  43.96 & 39.94 & 40.37 & 42.99 & 39.61 & 39.53 & 39.58  \\ 
      & \textbf{+ FN} & \textbf{77.23} & \textbf{81.46} & \textbf{82.80} & \textbf{84.16}  & \textbf{75.58} & \textbf{81.85} & \textbf{83.06} & \textbf{83.65} & \textbf{45.16} & \textbf{48.83} & \textbf{49.54} & \textbf{47.87} & \textbf{47.66} & \textbf{47.86} & \textbf{47.80}  \\  \midrule
      
      & Baseline & 77.07 & (\textit{Failed}) & \textbf{84.50} & 85.28 & (\textit{Failed}) & 83.24 & 83.54 & 85.21 & 34.64 & 42.56 & 43.67 & \textbf{45.18} & 45.62 & 44.42 & 43.16  \\
      Scaffold & + BABU & 76.89 & 82.24 & 84.26 & 85.02 & (\textit{Failed}) & 82.95 & 84.12 & 85.06 & 46.80 & 42.73 & 44.50 & 43.28 & 43.76 & 44.37 & 44.20 \\ 
       & \textbf{+ FN} & \textbf{79.05} & \textbf{82.83} & \textbf{84.80} & \textbf{85.83} & \textbf{74.53} & \textbf{83.65} & \textbf{84.79} & \textbf{85.42} & \textbf{50.17} & \textbf{48.74} & \textbf{51.04} & 45.10 & \textbf{49.71} & \textbf{50.42} & \textbf{52.13}  \\  \midrule
       
      & Baseline & 73.92 & 78.86 & 81.47 & 82.61 & 73.39 & 80.94 & 81.16 & 82.64 & 26.57 & 39.93 & 42.11 & 42.71 & 42.77 & 44.54 & 41.55 \\
      FedEXP & +BABU & 74.48 & 78.91 & 81.16 & 82.17 & 71.61 & 80.50 & 81.24 & 82.72 & 45.44 & 41.49 & 40.24 & 42.73 & 40.09 & 39.54 & 40.66  \\ 
      & \textbf{+ FN} & \textbf{77.17} & \textbf{80.92} & \textbf{82.79} & \textbf{83.57} & \textbf{75.03} & \textbf{82.29} & \textbf{83.12} & \textbf{83.36} & \textbf{45.37} & \textbf{48.45} & \textbf{49.80} & \textbf{47.73} & \textbf{47.16} & \textbf{48.32} & \textbf{48.01} \\  \bottomrule 
    \end{tabular}
    }
    \label{ch2tab:gfl_acc_all}
\end{table}

\subsection{Personalized Federated Learning (PFL) Results}\label{ch2subsec:pfl_result}

We present FedFN-FT, a fine-tuned algorithm for PFL, inspired by prior work~\citep{oh2021fedbabu, dong2022spherefed}, utilizing local data.  We compare FedFN-FT with existing PFL methods, including simple local models, 1-step approaches like FedPer~\cite{arivazhagan2019federated}, Per-FedAVG~\cite{fallah2020personalized}, and FedRep~\cite{collins2021exploiting}, as well as 2-step methods such as FedAVG-FT, FedBABU-FT~\cite{oh2021fedbabu}, SphereFed-FT (CE, MSE)\cite{dong2022spherefed}. Table~\ref{ch2apptab:pfl_acc} provides detailed personalized accuracy results. The entries form of X±Y, representing the mean and standard deviation of personalized accuracies across all clients for PFL algorithms. Entries without standard deviation indicate performance on $D_{test}$, derived from the global model after the initial step of 2-step methods.

Regarding the SphereFed method, it constructs each element of the logit vector based on the cosine similarity between feature vectors and classifiers. Similar to FedFN, SphereFed necessitates rescaling the learning rate for classifier weights, leading to the need for separate learning rate tuning.  To address this requirement,  we conduct an extensive grid search to determine the appropriate initial learning rate, denoted as $\eta$. Further details about the grid search are provided in Section~\ref{ch2app:grid search}. 

For SphereFed (CE), SphereFed (MSE), and FedFN, we initialize the learning rates with $\eta$ values of 0.3, 4.5, and 0.5, respectively. The 2-step fine-tuning methods undergo a total of 5 local epochs, during which learning rate was carried out through grid search within the range of $\{\eta, 0.1\times\eta, 0.01\times\eta\}$. The comprehensive results of the grid search for SphereFed and FN, including learning rate adjustments, are confirmed to be available in Subsection~\ref{ch2app:finetuneing_lr_search}. The resulting tunned learning rates are documented in Table~\ref{ch2apptab:pfl_acc}. In summary, 2-step methods consistently outperform 1-step methods in terms of PFL performance. Among the 2-step methods, FedFN-FT consistently exhibits superior performance.

\begin{table}[htp]
    \centering
    \caption{PFL Accuracy Comparison for MobileNet on CIFAR-100.}
    \small
    \begin{tabular}{l|ccc}
    \toprule
    Algorithm & $s$=10 & $s$=50 & $s$=100 \\ \midrule
    Local only                              & 58.64{\tiny $\pm$7.21}  & 25.38{\tiny $\pm$4.12} & 18.52{\tiny $\pm$3.15} \\ \midrule
    FedPer (2019)                           & 70.92{\tiny $\pm$6.93}  & 33.73{\tiny $\pm$4.68} & 22.77{\tiny $\pm$4.30} \\
    Per-FedAVG (2020)                       & 32.57{\tiny $\pm$11.02} & 43.09{\tiny $\pm$7.36} & 45.00{\tiny $\pm$7.05} \\
    FedRep (2021)                           & 62.69{\tiny $\pm$7.26}  & 34.66{\tiny $\pm$5.34} & 26.53{\tiny $\pm$4.52} \\ \midrule
    FedAVG (2017)                           & 36.30 & 41.97   & 42.51 \\ 
    FedAVG-FT (0.1$\times\eta$)             & 77.39{\tiny $\pm$6.40}  & 50.57{\tiny $\pm$5.31} & 46.95{\tiny $\pm$4.90} \\
    FedBABU (2022)                          & 45.73 & 39.57   & 40.70  \\
    FedBABU-FT (0.01$\times\eta$)           & 79.76{\tiny $\pm$6.05}  & 50.93{\tiny $\pm$4.69} & 46.63{\tiny $\pm$5.25} \\
    SphereFed (CE) (2022)                   & 40.39 & 43.68   & 44.96 \\ 
    SphereFed-FT (CE) (0.1$\times\eta$)     & 77.24{\tiny $\pm$6.38}  & 55.08{\tiny $\pm$5.33} & 50.17{\tiny $\pm$5.04} \\ 
    SphereFed (MSE) (2022)                  & 43.42 & 48.05   & 49.51 \\ 
    SphereFed-FT (MSE) (0.01$\times\eta$)   & \textbf{82.50}{\tiny $\pm$6.04} & 56.45{\tiny $\pm$5.56} & 53.10{\tiny $\pm$4.90}   \\ \midrule
    FedFN                                   & 46.98 & 49.47   & 50.92 \\
    \textbf{FedFN-FT} (0.1$\times\eta$)     & \textbf{82.85}{\tiny $\pm$5.83} & \textbf{60.78}{\tiny $\pm$4.99} & \textbf{55.43}{\tiny $\pm$5.24} \\
    \bottomrule
    \end{tabular}
    \label{ch2apptab:pfl_acc}
\end{table}
\newpage

\subsection{Fine Tuning Learning Rate Search}\label{ch2app:finetuneing_lr_search}
\begin{table*}[h!]
  \centering  
  \caption{Search for Finetuning Learning Rates in Two-Step Methods.}
\resizebox{0.8\textwidth}{!}{%
\begin{tabular}{c|ccc}
\toprule
Algorithm &  $s$=10 & $s$=50 & $s$=100  \\ \midrule

FedAVG (2017) & 36.30    & 41.97    & 42.51        \\ 

FedAVG-FT ($\eta$) & 67.23 $\pm$ 6.32 & 35.98 $\pm$ 5.56 & 38.25 $\pm$ 5.29    \\ 

\textbf{FedAVG-FT} (0.1$\times\eta$) & \textbf{77.39 $\pm$ 6.40} & \textbf{50.57 $\pm$ 5.31} & \textbf{46.95 $\pm$ 4.90} \\

FedAVG-FT (0.01$\times\eta$) & 70.96 $\pm$ 6.81   & 44.78 $\pm$ 4.88   & 44.00 $\pm$ 4.93       \\ 
\midrule

FedBABU (2022) & 45.73 & 39.57  & 40.7  \\ 

FedBABU-FT ($\eta$) & 13.32 $\pm$ 4.51   & 4.72 $\pm$ 1.85   & 2.85 $\pm$ 1.27       \\ 

FedBABU-FT (0.1$\times\eta$) & 77.88 $\pm$ 6.59 & 51.75 $\pm$ 4.84 & 46.02 $\pm$ 5.24 \\ 

\textbf{FedBABU-FT} (0.01$\times\eta$) & \textbf{79.76 $\pm$ 6.05} & \textbf{50.93 $\pm$ 4.69} & \textbf{46.63 $\pm$ 5.25} \\ \midrule

SphereFed (CE) (2022) & 40.39  & 43.68  & 44.96  \\ 

SphereFed-FT (CE) ($\eta$) & 53.79 $\pm$ 9.47 & 32.11 $\pm$ 6.89 & 27.81 $\pm$ 4.98 \\ 

\textbf{SphereFed-FT (CE)} (0.1$\times\eta$) & \textbf{77.24 $\pm$ 6.38}   & \textbf{55.08 $\pm$ 5.33} & \textbf{50.17 $\pm$ 5.04} \\ 

SphereFed-FT (CE) (0.01$\times\eta$) & 77.18 $\pm$ 6.17 & 49.99 $\pm$ 4.79   & 47.45 $\pm$ 4.70       \\

\midrule

SphereFed (MSE) (2022) & 43.42  & 48.05  & 49.51  \\ 

SphereFed-FT (MSE) ($\eta$) & 57.34 $\pm$ 7.53 & 35.43 $\pm$ 5.32 & 27.95 $\pm$ 5.03 \\ 

SphereFed-FT (MSE) (0.1$\times\eta$) & 79.73 $\pm$ 6.63 & 56.80 $\pm$ 5.29 & 51.06 $\pm$ 5.59 \\ 

\textbf{SphereFed-FT (MSE)} (0.01$\times\eta$) & \textbf{82.50 $\pm$ 6.04} & \textbf{56.45 $\pm$ 5.56} & \textbf{53.10 $\pm$ 4.90} \\ 
\midrule
FedFN  & 46.98 & 49.47 & 50.92  \\

FedFN-FT ($\eta$) & 71.88 $\pm$ 6.73 & 43.80 $\pm$ 4.54 & 41.23 $\pm$ 5.46       \\ 

\textbf{FedFN-FT} (0.1$\times\eta$) & \textbf{82.85 $\pm$ 5.83}   & \textbf{60.78 $\pm$ 4.99} & \textbf{55.43 $\pm$ 5.24}    \\

FedFN-FT (0.01$\times\eta$) & 82.15 $\pm$ 5.74 & 54.24 $\pm$ 5.06 & 51.92 $\pm$ 5.06    \\ 
\bottomrule
   
        \end{tabular}%
    }
\label{ch2tab:2_step_fine_tune_lr_search}
\end{table*}

\newpage
\subsection{Logit Should Be Non-Restricted}\label{ch2app:modify_norm}
SphereFed~\cite{dong2022spherefed} modifies i-th index of the logit vector of an input $x$, represented as $\tilde{z}_{i}(x;\theta)=\tilde{\theta}_{cls, i}\frac{f(x;\theta_{ext})}{||f(x;\theta_{ext})||_{2}}$. It maintains the classifier $\Tilde{\theta}_{cls}$ in a frozen state, ensuring that the norms of $\Tilde{\theta}_{cls, i}$ are orthonormal to each other. As a result, $\tilde{z}_{i}(x;\theta)$ becomes the cosine similarity between $f(x;\theta_{ext})$ and $\tilde{\theta}_{cls, i}$, yielding values restricted to the range [-1,1]. Following this modification, the logit margin is constrained to a maximum value of 2.

As seen in Table~\ref{ch2apptab:pfl_acc}, despite utilizing feature normalization, SphereFed (CE) exhibits inferior performance compared to even FedBABU. To address this limitation, we propose a modification to the SphereFed (CE) logit vector. We transform it to $\tilde{z}^{\tau}_{i}(x;\theta)=\tau\,\tilde{\theta}_{cls, i}\frac{f(x;\theta_{ext})}{||f(x;\theta_{ext})||_{2}}$, which yields values in the range of [-$\tau$, $\tau$], providing less constrained outputs. We apply this approach with different values of $\tau$, specifically $\{10, 15, 20, 25, 30\}$, in the $s$=10 setting. We compared the results of this modified SphereFed (CE) with those of SphereFed (MSE), FedBABU, and FedFN, and the outcomes are presented in Table~\ref{ch2apptab:modified_spherefed}.
Increasing $\tau$ up to 15 results in improvements in SphereFed(CE), although not as significant as compared to FedFN. However, it shows enhancements over FedBABU and SphereFed(MSE) at 15. Consequently, this indicates that creating the logit vector through feature normalization with relaxed constraints on the elements of the logit is recommended.
 
\begin{table}[h!]
\centering
\caption{Comparison of Modified SphereFed (CE) with SphereFed (MSE), FedBABU, and FedFN on $s$=10 Setting of CIFAR-100.}
\resizebox{0.4\textwidth}{!}{
\begin{tabular}{c|c}
\toprule
Algorithm & Accuracy \\ \midrule
SphereFed (CE), $\tau$=1 & 40.39 \\   
SphereFed (CE), $\tau$=10  & 42.84 \\ 
SphereFed (CE), $\tau$=15  & \textbf{45.78} \\ 
SphereFed (CE), $\tau$=20  & 44.95 \\ 
SphereFed (CE), $\tau$=25  & 44.46 \\ 
SphereFed (CE), $\tau$=30  & 39.62 \\  \midrule
SphereFed (MSE)  & 43.42 \\\midrule
FedBABU  & 45.73 \\ \midrule
FedFN  & \textbf{46.98} \\
\bottomrule
\end{tabular}
}
\label{ch2apptab:modified_spherefed}
\end{table}

\newpage
\subsection{Reproduced Result from SphereFed}

We present the experimental results for SphereFed~\cite{dong2022spherefed} trained with LDA settings ($\alpha\in\{0.1, 0.5\}$) on CIFAR-100. To reproduce the experiments presented in the original paper, we deviated from our previous experimental settings. Specifically, for the MobileNetV2 model architecture, we constructed the layers exactly as described in Table 7 of \citep{dong2022spherefed}. Regarding the training setup, each case is trained for 500 rounds using cosine annealing, following the guideline of original paper. We follow the instructions of original paper for all other hyperparameters as well. It should be noted that we did not employ the FFC algorithm in any of the experiments, including those using FedAVG and FedFN.

Table~\ref{ch2apptab:spherefed_reproduce} presents the new results we obtain and compares them with the original outcomes of SphereFed, encompassing FedAVG, FedFN and centralized learning. If certain algorithms are not indicated at the original results, we represent them with a dash (``-"). For the algorithms implemented as described in the original paper, we provide specific details like the actual learning rate $\eta$. Furthermore, for each FL algorithm, we present reproduced results across a specified range of initial learning rates $\eta\in\{0.1, 0.3, 0.5, 1.0, 1.5, 3.0, 4.5, 5.0\}$. In the case of centralized learning, results are specifically provided for $\eta=0.1$.

Following this, we conclude that the results of the original paper could not be reproduced. The reported performance of FL algorithms in the actual original paper (71.85, 68.78) appears surprisingly higher than the reproduced results in centralized learning (68.27). Moreover, implementing the algorithms with the exact settings reported in the original results consistently leads to lower performance (71.85 vs 18.01, 68.78 vs 37.76). Even when implemented in accordance with the specifications of the original paper ($\eta$=0.5), SphereFed (MSE) demonstrated significantly poor performance at 18.01. Subsequently, despite a thorough investigation through grid search, the best-performing configuration obtained is 52.69, still falling below the reported performance. Furthermore, when comparing the performance at the optimal learning rate for each FL algorithm, we consistently observe that FedFN outperforms the baselines.

\begin{table}[h!]
    \centering
    \caption{Reproduced Results for SphereFed, FedAVG, and FedFN under the Same Settings, Utilizing MobileNet (as Described in \citep{dong2022spherefed}) on CIFAR-100.}

    \resizebox{\textwidth}{!}{
        \begin{tabular}{c|cccccccc|c}
        \toprule
        \emph{SphereFed (MSE)} & \emph{$\eta$=0.1} & \emph{$\eta$=0.3} & \emph{$\eta$=0.5} & \emph{$\eta$=1.0} & \emph{$\eta$=1.5} & \emph{$\eta$=3.0} & \emph{$\eta$=4.5} & \emph{$\eta$=5.0} & Original Result (\emph{$\eta$=0.5})   \\ \cmidrule{1-10}         
        $\alpha$=0.5 & 2.77 & 9.73 & 18.01 & 43.26 & 52.19 & \textbf{52.69} & 48.20 & 40.87 & 71.85 \\
        $\alpha$=0.1 & 2.88 & 9.65 & 20.19 & 41.48 & 45.41 & \textbf{46.34} & 43.62 & 40.56 & - \\ \midrule
        \emph{SphereFed (CE)} & \emph{$\eta$=0.1} & \emph{$\eta$=0.3} & \emph{$\eta$=0.5} & \emph{$\eta$=1.0} & \emph{$\eta$=1.5} & \emph{$\eta$=3.0} & \emph{$\eta$=4.5} & \emph{$\eta$=5.0} & Original Result   \\ \cmidrule{1-10}         
        $\alpha$=0.5 & 37.57 & 49.39 & 48.18 & 52.51 & \textbf{52.99} & 51.09 & 44.87 & 44.23 & - \\
        $\alpha$=0.1 & 22.37 & 40.58 & 42.85 & 40.92 & \textbf{47.72} & 42.11 & 35.30 & 33.83 & - \\ \midrule
        \emph{FedAVG} & \emph{$\eta$=0.1} & \emph{$\eta$=0.3} & \emph{$\eta$=0.5} & \emph{$\eta$=1.0} & \emph{$\eta$=1.5} & \emph{$\eta$=3.0} & \emph{$\eta$=4.5} & \emph{$\eta$=5.0} & Original Result (\emph{$\eta$=0.1})   \\ \cmidrule{1-10}         
        $\alpha$=0.5 & 37.76 & \textbf{38.82} & 23.76 & 1.04 & 1.02 & 1.03 & 1.25 & 1.01 & 68.78 \\
        $\alpha$=0.1 & 38.58 & \textbf{40.47} & 27.35 & 1.22 & 1.04 & 1.02 & 1.09 & 1.17 & - \\ \midrule
        \emph{FedFN} & \emph{$\eta$=0.1} & \emph{$\eta$=0.3} & \emph{$\eta$=0.5} & \emph{$\eta$=1.0} & \emph{$\eta$=1.5} & \emph{$\eta$=3.0} & \emph{$\eta$=4.5} & \emph{$\eta$=5.0} & Original Result   \\ \cmidrule{1-10}         
        $\alpha$=0.5 & 55.00 & \textbf{53.38} & 48.55 & 49.89 & 45.79 & 42.45 & 35.43 & 34.82 & - \\
        $\alpha$=0.1 & 46.38 & \textbf{49.16} & 46.61 & 41.69 & 42.03 & 38.80 & 30.92 & 2.89 & - \\ \midrule
        \emph{Centralized Learning} &   \multicolumn{8}{c|}{\textbf{68.27} (\emph{$\eta$=0.1})} & - \\ \bottomrule

        \end{tabular}
        }
        \label{ch2apptab:spherefed_reproduce}
\end{table}

\newpage
\subsection{FedFN vs FedFR}

We compare the performance of FedFN with Federated Averaging with Feature Norm Regularization (FedFR) introduced in Eq.~\eqref{ch2eqn:fedfr} in the main paragraph. Table~\ref{ch2apptab:cifar10_fedfr} and Table~\ref{ch2apptab:cifar100_fedfr} present the performance on CIFAR-10 and CIFAR-100 with $s=10$ setting, respectively.

Table~\ref{ch2apptab:cifar10_fedfr} reports results of FedFR on CIFAR-10, referring to the optimal hyperparameter $\mu=0.005$ from Figure~\ref{ch2fig:fedfr-result} in the main paragraph. FedFR exhibits slightly lower performance compared to FedFN but demonstrates superiority over FedAVG and FedBABU. In the $s=10$ setting of CIFAR-100, FedFR shows comparable or superior performance to FedAVG. However, FedFR performs worse than both FedBABU and FedFN across all hyperparameter candidates. In contrast to FedFR, FedFN consistently demonstrates superior performance across all settings.

\begin{table}[htp]
    \centering
    \vspace{-0.15 in}
    \caption{Accuracy Comparison on CIFAR-10.}
    \small
    \begin{tabular}{c|cccc}
    \toprule
    \multirow{2.5}{*}{Algorithm} & \multicolumn{4}{c}{VGG11 on CIFAR-10}\\ \cmidrule{2-5} 
                                 & $s$=2 & $s$=3 & $s$=5 & $s$=10 \\ \midrule
    FedAVG   & 74.24 & 77.29 & 81.08 & 81.97  \\ 
    FedBABU   & 75.05 & 77.73 & 81.04 & 82.16  \\     
    FedFR  & 76.14 & 77.89 & 81.61 & 82.18 \\ 
    FedFN    & \textbf{77.77} & \textbf{78.93} & \textbf{82.43} & \textbf{83.80} \\ \bottomrule                            
    \end{tabular}    \label{ch2apptab:cifar10_fedfr}
\end{table}

\begin{table}[h!]
    \centering
    \caption{Accuracy Comparison on the $s=10$ Setting of MobileNet on CIFAR-100.}
    \resizebox{\textwidth}{!}{
        \begin{tabular}{c|cccccccc}
        \toprule
        \multirow{2.5}{*}{FedFR} & \emph{$\mu$=0.5} & \emph{$\mu$=0.1} &\emph{$\mu$=0.05} & \emph{$\mu$=0.01} & \emph{$\mu$=0.005} & \emph{$\mu$=0.001} & \emph{$\mu$=0.0005} & \emph{$\mu$=0.0001}\\ \cmidrule{2-9}         
          & \emph{(Failed)} & 37.74 & \textbf{39.50} & 36.72 & 36.51 & 36.77 & 37.04 & 37.30\\ \midrule
        FedAVG &   \multicolumn{8}{c}{\textbf{36.30} (\emph{$\mu$=0.0})} \\ \midrule
        FedBABU &   \multicolumn{8}{c}{\textbf{45.73}} \\ \midrule
        FedFN &   \multicolumn{8}{c}{\textbf{46.98}} \\ \bottomrule
        
        \end{tabular}
        }
        \label{ch2apptab:cifar100_fedfr}
\end{table}
\newpage
\subsection{FN in the Centralized Learning}
\begin{figure}[h!]
    \centering    
    \includegraphics[width=1.0\textwidth]{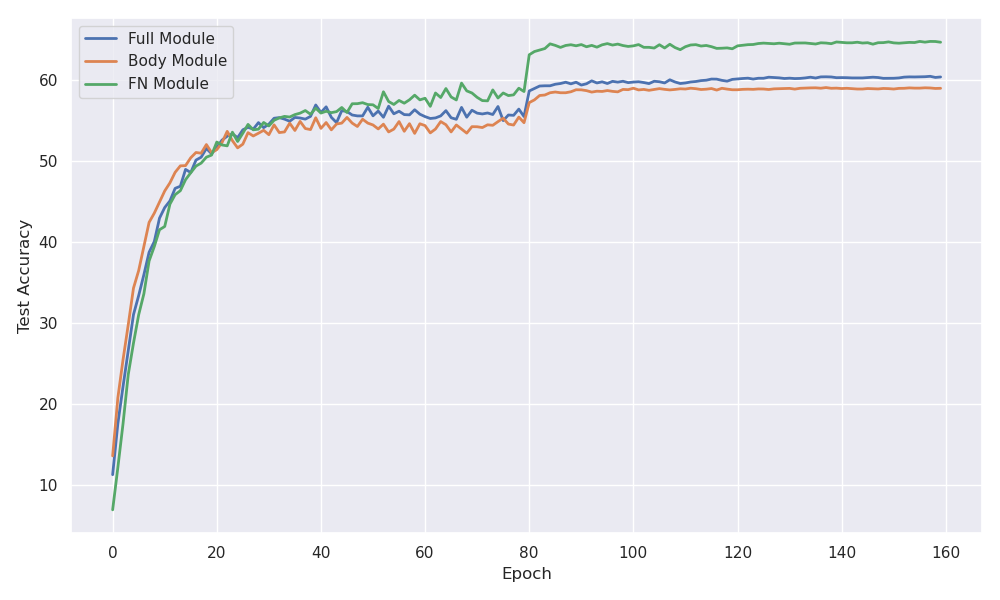}
    \caption{Module Comparison in Centralized Learning}
    \label{ch2fig:centralized_acc_cifar100}
\end{figure}
In FedBABU~\cite{oh2021fedbabu}, the authors comprehensively evaluate the performance of the Full and Body modules in a centralized learning setup. Expanding upon their analysis, we incorporate the FN module to assess its performance in comparison to that of the Full and Body modules. To ensure fair comparisons, we replicate the experimental settings outlined in \cite{oh2021fedbabu}. The experiments are carried out on the CIFAR-100 dataset, utilizing the MobileNet architecture. The results of the conducted experiments are depicted in Figure~\ref{ch2fig:centralized_acc_cifar100}.

As expected, our evaluation reveals a slight decline in performance within the Body module compared to the Full module, aligning with the findings reported in \cite{oh2021fedbabu}. This performance difference can be attributed to the partial update constraint imposed on the model during the training of the Body module. In contrast, the FN module, despite incorporating the constraint of feature normalization, exhibits significant improvements over the Full module. This enhancement is attributed to the capacity of FN module to empower the model to acquire more effective and discriminative representations, thereby enhancing overall performance.

\begin{table}[t]
    \centering
    \begin{tabular}{lrrr|rrr}
        \toprule
        Algorithm & shard 10 & shard 50 & shard 100 & lda 0.1 & lda 0.5 & lda 1.0 \\
        \midrule
        Local only        & 58.64 & 25.38 & 18.52 & 43.63 & 21.03 & 14.88 \\
        \midrule
        FedPer (2019)     & 70.92 & 33.73 & 22.77 & 55.93 & 27.32 & 18.93 \\
        Per-FedAVG (2020) & 32.57 & 43.09 & 45.00 & 35.12 & 42.95 & 42.82 \\
        FedRep (2021)     & 62.69 & 34.66 & 26.53 & 48.52 & 30.94 & 23.93 \\
        \midrule
        FedAVG (2017)     & 36.63 & 40.51 & 43.18 & 33.70 & 39.91 & 41.36 \\
        FedAVG-FT (2017)  & 70.39 & 43.54 & 44.08 & 55.06 & 45.07 & 43.66 \\
        FedBABU (2022)    & 44.92 & 40.70 & 39.81 & 38.71 & 39.87 & 37.55 \\
        FedBABU-FT (2022) & 79.11 & 51.36 & 45.12 & 69.05 & 51.14 & 43.39 \\
        FedFN (Ours)      & 47.23 & 49.72 & 51.33 & 41.30 & 45.73 & 44.00 \\
        \textbf{FedFN-FT} & \textbf{82.02} & \textbf{54.72} & \textbf{52.99} &
                           \textbf{69.50} & \textbf{51.76} & \textbf{46.85} \\
        FedDr+ (Ours)      & 48.69 & 51.49 & 53.23 & 45.83 & 47.85 & 47.35 \\
        \textbf{FedDr+FT} & \textbf{84.10} & \textbf{61.34} & \textbf{56.76} &
                           \textbf{74.79} & \textbf{56.01} & \textbf{51.62} \\

        \bottomrule
    \end{tabular}
\end{table}

\newpage

\section{Conclusion}\label{ch2sec: discussion}

In this study, we observe that increasing data heterogeneity leads to larger feature norm disparities between global and local models, which are influenced by feature norm bias within local models. 
We address this issue by introducing feature normalization techniques. Extensive experiments across various FL confirm the superior performance of feature normalization, emphasizing its role in enhancing the feature representations. Our experiments showcase the exceptional performance of FedFN, extending its effectiveness to pretrained ResNet18 models and confirming its applicability to foundational models.

\chapter{FedDr+: Stabilizing Dot-regression with Global Feature Distillation for Federated Learning}

\begin{tcolorbox}[
    colback=gray!10, 
    colframe=black, 
    width=0.9\textwidth, 
    boxrule=1pt, 
    arc=4pt, 
    left=5pt, 
    right=5pt, 
    center
]

Federated Learning (FL) has emerged as a pivotal framework for the development of effective global models (global FL) or personalized models (personalized FL) across clients with heterogeneous, non-iid data distribution. A key challenge in FL is client drift, where data heterogeneity impedes the aggregation of scattered knowledge. Recent studies have tackled the client drift issue by identifying significant divergence in the last linear (classifier) layer. To mitigate this divergence, strategies such as freezing the classifier weights and aligning the feature extractor accordingly have proven effective. 
Although the local alignment between classifier and feature extractor has been studied as a crucial factor in FL, we observe that it may lead the model to overemphasize the observed classes and underestimate the unobserved classes within each client. 
Therefore, our goals are twofold: (1) \emph{improving local alignment} and (2) \emph{maintaining the representation of unseen class samples}, ensuring that the solution seamlessly incorporates knowledge from individual clients, thus enhancing performance in both global and personalized FL.
To achieve this, we introduce a novel algorithm named \alg, which empowers local model alignment using dot-regression loss. \alg freezes the classifier as a simplex ETF to align the features and improves aggregated global models by employing a feature distillation mechanism to retain information about unseen/missing classes. Our empirical results demonstrate that \alg not only outperforms methods with a frozen classifier but also surpasses other state-of-the-art approaches, ensuring robust performance across diverse data distributions.

\end{tcolorbox}

\section{Introduction} \label{ch3sec:intro}

Federated Learning (FL)~\citep{mcmahan2017communication, he2020fedml} is a distributed learning strategy that enables multiple clients to collaboratively train a model while preserving data privacy. The foundational method, FedAvg~\citep{mcmahan2017communication}, involves distributing a global model, training local models on each client's private data, and aggregating these models without transmitting raw data. However, a major challenge in FL is data heterogeneity, or \emph{non-iidness}, where differing data distributions across clients lead to \emph{client drift}, hindering the convergence and effectiveness of the global model.
Addressing this challenge involves improving two key aspects: \emph{local alignment} and \emph{global knowledge preservation}. Local alignment refers to the cosine similarity between the features extracted by the local model and the classifier’s true class vectors, computed on the client’s training data, aiming to maximize alignment for improved local training. Global knowledge preservation aims to retain the global model’s knowledge of rare or unobserved classes in the client’s training data, preventing forgetting during local updates.

While both local alignment and global knowledge preservation are essential, they have generally been studied separately. Global knowledge preservation is crucial because it prevents the model from becoming overly biased toward the data of individual clients, ensuring decisions are based on a broader, shared understanding. This approach enables better generalization across all clients, particularly for unseen classes~\citep{lee2022preservation, lee2024fedsol}. The challenge of balancing global and local knowledge in FL resembles Catastrophic Forgetting in Continual Learning (CL)~\citep{mccloskey1989catastrophic}, where learning new tasks can cause models to forget previously learned ones. To achieve this, strategies like FedProx~\citep{MLSYS2020_1f5fe839}, MOON~\citep{li2021model}, and FedNTD~\citep{lee2022preservation} integrate global model regularization during local training. These methods align local models with the global objective through techniques like proximal terms, contrastive learning, and logit-based regularizers.

\begin{figure}
  \centering
  \small
  \includegraphics[width=0.73\linewidth]{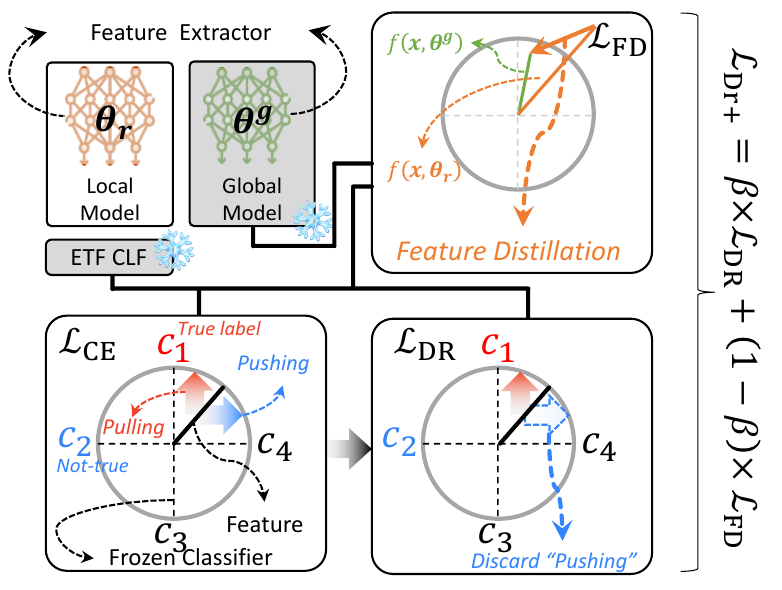}
  \caption{Overview of the proposed method, \alg trained with $\mc{L}_{\text{Dr+}}$. To enhance the local alignment, we employ dot-regression loss $\mc{L}_{\text{DR}}$, which discards the pushing term of cross-entropy loss, and propose a feature distillation $\mc{L}_{\text{FD}}$ to preserve the knowledge imbued in the global model. }
  \label{ch3fig:overview}
  \vspace{-16pt}
\end{figure}

On the other direction, to improve the local alignment, recent studies have extensively focused on freezing the last linear layer (called classifier) while updating only the feature extractor. This approach is motivated by the fact that the classifier is most sensitive to data heterogeneity~\citep{luo2021no, li2023no, fan2023federated}; freezing it ensures that all local models align their features to a consistent classifier across clients. For example, FedBABU~\citep{oh2021fedbabu} fixes the classifier after initialization, allowing only the feature extractor to adapt. Other methods~\citep{dong2022spherefed, li2023no, fan2023federated, huang2023neural, xiao2024fedloge} further enhance local alignment by modifying the loss function or by using robust initialization techniques, such as the Equiangular Tight Frame (ETF) classifier.

A frozen classifier is also extensively explored in other research areas, such as class imbalance~\citep{yang2022inducing} and class incremental learning~\citep{yang2023neural}, with a consistent objective similar to aforementioned FL studies---enhancing alignment. Recently, these fields have advanced by introducing and utilizing a novel type of loss, called dot-regression loss $\mc{L}_\text{DR}$, which aims to achieve alignment rapidly. In summary, $\mc{L}_\text{DR}$ originates from the decomposition analysis of cross-entropy (CE) loss, which includes \emph{pulling} and \emph{pushing}. As suggested in~\citep{yang2022inducing}, the \emph{pulling} component is a force that {attracts} features to the target class, whereas the \emph{pushing} component is a force that drives features {away from} other non-target classes.
$\mc{L}_\text{DR}$ discards the \emph{pushing} component, as it slows down convergence to the desired alignment (refer to Figure~\ref{ch3fig:overview}).

However, our findings indicate that while dot-regression loss enhances \emph{local alignment} as intended, it does not lead to sufficient performance improvement of the aggregated server-side model. The main issue arises from the insufficient \emph{global knowledge preservation} of unobserved classes during local training.  Specifically, the focus on improving alignment for classes present in the local training dataset induces significant feature dynamics, which inadvertently disrupt the representation of features associated with unobserved classes. This disruption leads to forgetting and deteriorated alignment for these classes, highlighting the need for better preservation of global knowledge during local training.

We emphasize the importance of addressing both local alignment and global knowledge preservation together. \alg provides a well-generalized global model by combining dot-regression loss with feature distillation, reducing the distance between feature vectors of local and global models. \textbf{{\alg}\,FT} extends this by fine-tuning the \alg global model using the same \alg loss function, enhancing local alignment for client-specific data. Starting with a well-trained global model is essential for achieving effective personalization while maintaining global generalization~\citep{nguyen2022begin,chen2022importance}.

\vspace{1pt}
\myparagraph{Contributions.} Our main contributions are summarized as follows:
\begin{itemize}[leftmargin=15pt]
    \item In high-heterogeneity FL settings, we observe a trade-off in the classifier-freezing setup: dot-regression loss improves local alignment with observed classes but leads to lower global model performance compared to CE loss, due to a significant loss of information on unseen classes, which is critical for the global model.
    \item To address this, we propose \alg, which preserves global knowledge through feature distillation while maintaining the advantages of dot-regression loss for local alignment. This contribution focuses on improving global federated learning (GFL).
    \item We extend \alg to personalized federated learning (PFL) via \textbf{{\alg}\,FT}, which fine-tunes the \alg global model using the \alg loss function for client-specific data. This highlights the importance of starting with a well-generalized global model for personalization.
    \item We demonstrate the superiority of our method across various datasets and non-iid settings.
\end{itemize}

\section{Related Work}
\label{ch3sec:related}


\myparagraph{Federated learning.}
Federated Learning (FL) is a decentralized approach to deep learning where multiple clients collaboratively train a global model using their own datasets~\citep{mcmahan2017communication, MLSYS2020_1f5fe839}. This approach faces challenges due to data heterogeneity across clients, causing instability in the learning process~\citep{karimireddy2020scaffold, luo2021no}. To address this problem, strategies like classifier variance reduction in FedPVR~\citep{li2022effectiveness} and virtual features in CCVR~\citep{luo2021no} have been proposed. Additionally, it is essential to distinguish between Global Federated Learning (GFL) and Personalized Federated Learning (PFL), as these are crucial concepts in FL. GFL aims to improve a single global model's performance across clients by addressing data heterogeneity through methods like client drift mitigation~\citep{MLSYS2020_1f5fe839, karimireddy2020scaffold, jhunjhunwala2023fedexp}, enhanced aggregation schemes~\citep{wang2020federated, wang2020tackling}, and data sharing techniques using public or synthesized datasets~\citep{lin2020ensemble, luo2021no}. Otherwise, PFL focuses on creating personalized models for individual clients by decoupling feature extractors and classifiers for unique updates~\citep{oh2021fedbabu, arivazhagan2019federated, collins2021exploiting}, modifying local loss functions~\citep{fallah2020personalized, li2021ditto}, and using prototype communication techniques~\citep{tan2022fedproto, xu2023personalized}.

\myparagraph{Frozen classifier in FL.}
By focusing on alignment, previous studies have attempted to mitigate data heterogeneity by freezing the classifier \citep{oh2021fedbabu, dong2022spherefed, li2023no}. Nevertheless, these methods have yet to effectively improve the alignment between features and their corresponding classifier weights. Motivated by this, we integrated the dot-regression method into FL to achieve a better-aligned local model by freezing the classifier. Dot-regression, proposed to address class imbalance, focuses on aligning feature vectors to a fixed classifier, demonstrating superior alignment performance compared to previous approaches. However, optimizing the dot-regression loss to align feature vectors with a fixed classifier caused the local model to lose information on unobserved classes, thereby degrading global model performance. To address these issues, FedLoGe~\citep{xiao2024fedloge} employing realignment techniques to ensure the well-aligned local model's performance translated to the global model. Additionally, in FedGELA~\citep{fan2023federated}, the classifier is globally fixed as a simplex ETF while being locally adapted to personal distributions. Also, FedPAC~\cite{xu2023personalized} addressed these challenges by leveraging global semantic knowledge for explicit local-global feature alignment. Besides alignment-focused methods, there have been various attempts to maintain good local model performance in the global model~\cite{jiang2023heterogeneous, an2024federated, chen2021bridging}. 

\myparagraph{Knowledge distillation in FL.}
Knowledge distillation (KD) has been widely studied in FL settings, such as in FedMD~\citep{li2019fedmd} and FedDF~\citep{lin2020ensemble}, where a pretrained teacher model transfers knowledge to a student model. Additional distillation-based methods, such as FedFed~\citep{yang2024fedfed} and co-distillation framework for PFL~\citep{chen2024spectral, cho2023communication}, have also been explored. In contrast to existing methods, we propose a loss function incorporating feature distillation to maintain the performance of both local and global models. To our knowledge, this is the first application of feature distillation in FL. This approach highlights the importance of distinguishing between GFL and PFL.

\section{Preliminaries}
\label{ch3sec:pre}

In this section, we describe the basic settings of Federated Learning (FL), including the \code{FedAvg} pipeline~\citep{mcmahan2017communication} and the dot-regression loss~\citep{yang2022inducing}, both of which are utilized in our framework. For clarity, we summarize the main notations in Table~\ref{ch3tab:main_notation_summary}.

\begin{table}[htp]
\centering
\caption{Main used throughout the paper.}
\label{ch3tab:main_notation_summary}
\resizebox{0.75\textwidth}{!}{
\begin{tabular}{ll}
\toprule
\textbf{Indices} & \\ 
$c \in [C]$  & Index for a class \\
$r \in [R]$ & Index for FL round \\
$i \in [N]$ & Index for a client \\ 
\midrule

\textbf{Dataset} & \\
$D_{\text{train}}^{i}$ & Training dataset for client $i$ \\
$D_{\text{test}}^{i}$ & Test dataset for client $i$ \\
$(x, y) \in D_{\text{train,test}}^{i}\,; (x, y) \sim \mc{D}^{i}$ & Data on client $i$ sampled from distribution $\mc{D}^{i}$ \\
 & ($x$: input data, $y$: class label) \\
$\mc{O}^i$ & Dataset consists of observed classes in client $i$ \\
$\mc{U}^i$ & Dataset consists of unobserved classes in client $i$ \\

\midrule
\textbf{Parameters} & \\ 
$\bm{\theta}$ & Feature extractor weight parameters \\
$\bm{V} = [v_1, \ldots, v_C] \in \mathbb{R}^{C \times d}$ & Classifier weight parameters (frozen during training) \\
$v_c, c \in [C]$ & $c$-th row vector of $\bm{V}$ \\
$\bm{\Theta} = (\bm{\theta}, \bm{V})$ & All model parameters \\
$\bm{\Theta}_{r}^{g} = (\bm{\theta}_{r}^{g}, \bm{V})$ & Aggregated global model parameters at round $r$ \\
$\bm{\Theta}_{r}^{i}= (\bm{\theta}_{r}^{i}, \bm{V})$ & Trained model parameters on client $i$ at round $r$ \\

\midrule
\textbf{Model Forward} & \\
$p(x; \bm{\theta}) \in \mathbb{R}^{C}$ & Softmax probability of input $x$ \\ 
$p_c(x; \bm{\theta}), c \in [C]$ & $c$-th element of $p(x; \bm{\theta})$ \\ 
$\mathcal{L}_{\text{CE}}(x; \theta) = -\log p_{y}(x; \bm{\theta})$ & Cross-entropy loss of input $x$ \\
$f(x; \bm{\theta}) \in \mathbb{R}^{d}$ & Feature vector of input $x$ \\ 
\bottomrule
\end{tabular}
}
\end{table}

\subsection{Basic Setup of Conventional FedAvg Pipeline}
\label{ch3subsec:fl}

\myparagraph{Basic FL setup.}
Let $[N] = \{1, \ldots, N\}$ denote the indices of clients, each with a unique training dataset $D_{\text{train}}^{i} = \{(x_m, y_m)\}_{m=1}^{|D_{\text{train}}^{i}|}$, where $(x_m, y_m) \sim \mc{D}^{i}$ for the $i^{\text{th}}$ client, $x_m$ is the input data, and $y_m \in [C]$ is the corresponding label among $C$ classes. Importantly, FL studies predominantly address the scenario where the data distributions are heterogeneous, \ie $\mc{D}^{i}$ varies across clients. Knowledge distributed among clients is collected over $R$ communication rounds.
The general objective of FL is to train a model fit to the aggregated knowledge, $\bigcup_{i \in [N]} \mc{D}^i$. This objective can be seen as solving the optimization problem:
\begin{equation*}
\label{ch3eq:flgoal}
\min_{\bm{\Theta} = (\bm{\theta}, \bm{V})} \sum_{i \in [N]} \frac{|D_{\text{train}}^{i}|}{\sum_{j \in [N]} |D_{\text{train}}^{j}|} \, \scalebox{1.45}{$\displaystyle\mathop{\mathbb{E}}$}_{(x, y) \sim \mc{D}^i} \Big[\mc{L}(x, y; \bm{\theta}, \bm{V})\Big]\,,
\end{equation*}
where $\mc{L}$ is the instance-wise loss function, $\bm{\theta}$ is the weight parameter for the feature extractor, and $\bm{V} = [v_1, \ldots, v_C]\in \mathbb{R}^{d \times C}$ is the classifier weight matrix. We use the notation $\bm{\Theta}$ to denote the entire set of model parameters.

At the beginning of each round $r \in [R]$, the server has access to only a subset of clients $\mc{S}_{r} \subset [N]$ participating in the $r^\text{th}$ round. At each round $r$, the server transmits the global model parameters $\bm{\Theta}_{r-1}^g$ to the participating clients. Each client then updates the parameters with their private data $D_{\text{train}}^{i}$ and uploads $\bm{\Theta}_{r}^i$ to the global server. By incorporating the locally trained weights, the server then updates the global model parameters to $\bm{\Theta}_{r}^g$.

\myparagraph{{FedAvg} pipeline.}
Our study follows the conventional {FedAvg}~\citep{mcmahan2017communication} framework to address the FL problem. {FedAvg} updates the global model parameters from locally trained parameters by aggregating these local models into $\bm{\Theta}_{r}^{g} = \sum_{i \in S_r} w_r^i \bm{\Theta}_{r}^{i}$, where $w_r^i = |D_{\text{train}}^{i}| \,/\, {\sum_{j \in S_{r}} |D_{\text{train}}^{j}|}$ is the importance weight of the $i^{\text{th}}$ client. 


\subsection{Dot-Regression Loss for Feature Alignment}
\myparagraph{Dot-regression loss $\mc{L}_{\text{DR}}$.}
This loss~\citep{yang2022inducing}
facilitates a faster alignment of feature vectors (penultimate layer
outputs) $f(x;\bm{\theta})\in \mathbb{R}^{d}$ to the true class direction of $v_y$, reducing the cosine angle as follows:
\begin{equation*}
    \mc{L}_{\text{DR}} (x, y; \bm{\theta}, \bm{V}) = \frac{1}{2}\Big(\cos\big(f(x; \bm{\theta}), v_y\big) - 1\Big)^2
\end{equation*}
where $\cos(\mathrm{vec}_1, \mathrm{vec}_2)$ denotes the cosine of the angle between two vectors  $\angle (\mathrm{vec}_1,\mathrm{vec}_2)$.

The main motivation is that the gradient of the cross-entropy (CE) loss for the feature vector can be decomposed into a \emph{pulling} and \emph{pushing} gradient, and recent work indicates that we can achieve better convergence by removing the pushing effect~\citep{yang2022inducing, li2021fedrs}. The \emph{pulling} gradient aligns $f(x; \bm{\theta})$ with $v_y$, while the \emph{pushing} gradient ensures $f(x; \bm{\theta})$ does not align with $v_c$ for all $c \neq y$ (Subsection~\ref{ch3app:prelim} details the exact form of pulling and pushing gradients).
Since $\mc{L}_{\text{DR}}$ directly attracts features to the true-class classifier, it drops the \emph{pushing} gradient, thereby increasing the convergence speed for maximizing $\cos(f(x; \bm{\theta}), v_y)$. 

\myparagraph{Frozen ETF classifier.}
Since $\mc{L}_{\text{DR}}$ focuses on aligning feature vectors with the true-class classifier, the classifier is not required to be trained. Instead, we construct the classifier to satisfy the simplex Equiangular Tight Frame (ETF) condition, a constructive way to achieve maximum angular separation between class vectors~\citep{yang2022inducing, yang2023neural}. Concretely, we initialize the classifier weight $\bm{V}$ as follows and freeze it throughout training: 
\begin{equation*}
\label{ch3eq:etf}
\bm{V} \longleftarrow \sqrt{\frac{C}{C-1}} \bm{U} \left( \bm{I}_C - \frac{1}{C} \bm{1}_C \bm{1}_C^\top \right),
\end{equation*}
where $\bm{U} \in \mathbb{R}^{d \times C}$ is a randomly initialized orthogonal matrix. Note that each $v_i$ in the classifier weight $\bm{V}$ satisfies $\cos(v_i, v_j) = -\frac{1}{C-1}$ for all $i \neq j \in [C]$\footnote{This relation for cosines holds if the $v_i$'s are symmetrically distributed such that $\bar{v} = \frac{1}{C} \sum_{i \in [C]} v_i = 0$, and $\cos(v_i, v_j)$ are all the same for $i \neq j$ .}.
\section{When Dot-Regression Loss Meets FL}
\label{ch3sec:prob}

Given our focus on applying $\mathcal{L}_{\text{DR}}$ to FL, we first examine its impact on FL models compared to the CE loss $\mathcal{L}_{\text{CE}}$. In summary, we find that while $\mathcal{L}_{\text{DR}}$ improves alignment-related performance on \textcolor{red}{observed} class labels, it faces challenge with \textcolor{blue}{unobserved} classes\footnote{While we use the term ``unobserved'' in this context, it also applies to ``rarely'' existing classes.}, which are essential for the generalization objective. To address this issue, we propose \alg, which integrates $\mathcal{L}_{\text{DR}}$ with a novel feature distillation loss. We then evaluate \alg by analyzing the effect of feature distillation and compare it with various FL algorithms and regularizers.

\myparagraph{Experimental configuration.} 
In this section, we conduct experiments on CIFAR-100~\citep{krizhevsky2009cifar} with a shard non-iid setting ($s$=10), where each client contains at most 10 classes. We additionally employ LDA setting ($\alpha$=0.1) in Section~\ref{ch3subsec:synergy_effect}. Refer to Section~\ref{ch3sec:exp} for more details on the dataset configuration. The model is trained for 320 communication rounds, randomly selecting 10\% of clients in each round, and the learning rate is decayed at $160^\text{th}$ and $240^\text{th}$ rounds. The experimental configuration for this section is detailed in subsection~\ref{ch3subsec:exp_setup}.

\subsection{Impact of $\mc{L}_{\text{DR}}$ on Local and Global Models}
\label{ch3subsec:dotreg}
 
We analyze the average performance of local and global models trained with $\mathcal{L}_{\text{DR}}$ compared to $\mathcal{L}_{\text{CE}}$, focusing on their ability to generalize. In Figure~\ref{ch3fig:alignment_gap}--\ref{ch3fig:problem_local_alignment}, we evaluate statistics on two datasets: the {\textcolor{red}{observed} class set} $\mathcal{O}^i$, containing classes in each client's training data $D_\text{train}^i$, and the {\textcolor{blue}{unobserved} class set} $\mathcal{U}^i$, representing unseen classes. Separately, Figure~\ref{ch3fig:problem_global_acc} reports the evaluation on all classes.

First, we examine the amount of change from the given global model to each local model in every communication round (Figure~\ref{ch3fig:alignment_gap}). The alignment gap is denoted by $\cos(f(x;\bm{\theta}_r^i), v_y) - \cos(f(x;\bm{\theta}_{r-1}^g),v_y)$. We then evaluate the feature-classifier alignment $\cos(f(x; \bm{\theta}_r^i), v_y)$ of each local model on the test data (Figure~\ref{ch3fig:problem_local_alignment}). Finally, we observe the test accuracy of the global model $\bm{\theta}_{r}^g$ (Figure~\ref{ch3fig:problem_global_acc}).


\begin{figure}[t!]
    \centering
        \begin{minipage}{\textwidth}
            \centering    
            \begin{subfigure}[b]{0.49\textwidth}
                \centering
                \includegraphics[width=\textwidth]{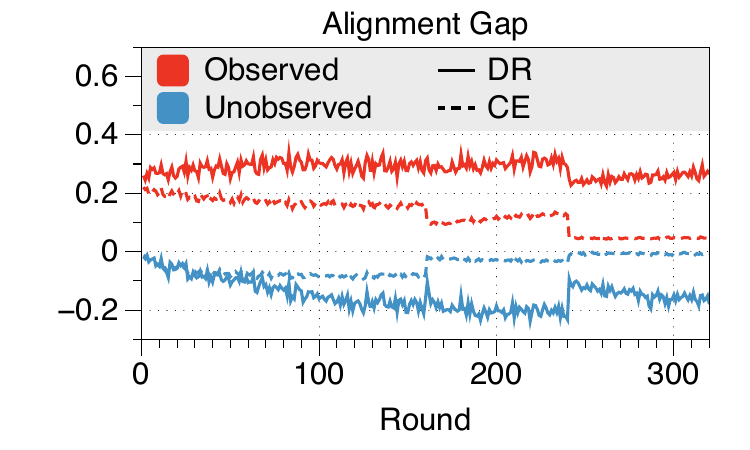}
                \vspace*{-15pt}
                \subcaption{Alignment gap}
                \label{ch3fig:alignment_gap}
            \end{subfigure}
            \hfill
            \begin{subfigure}[b]{0.49\textwidth}
                \centering
                    \includegraphics[width=\textwidth]{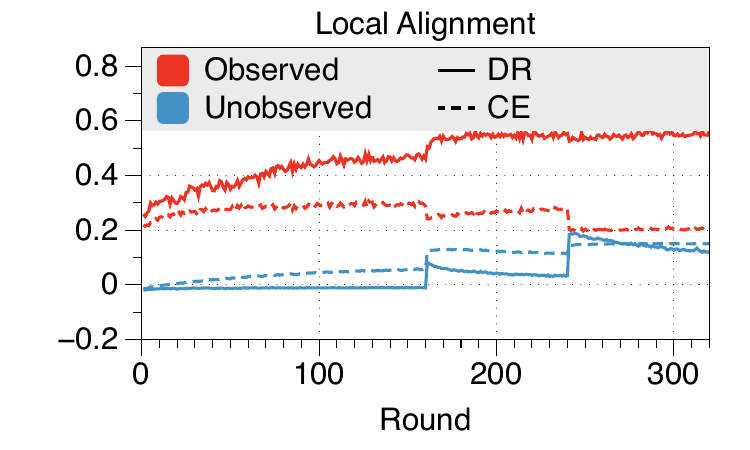}  
                \vspace*{-15pt}
                \subcaption{Local Alignment}
                \label{ch3fig:problem_local_alignment}
            \end{subfigure}
            \hspace{5pt}
            \vspace{-15pt}
            \caption{Comparison of (a) feature-classifier alignment gap and (b) feature-classifier alignment on the \textcolor{red}{observed} and \textcolor{blue}{unobserved} classes test data for $\bm{\theta}_r^i$ trained with $\mathcal{L}_\text{CE}$ and $\mathcal{L}_\text{DR}$.}
            \vspace*{-10pt}
            \label{ch3fig:local_align}
            \hspace{5pt}
        \end{minipage}
        \hspace{5pt}         
\end{figure}

\myparagraph{Alignment analysis of local models.} 
As shown in Figure~\ref{ch3fig:alignment_gap}--\ref{ch3fig:problem_local_alignment}, 
$\mathcal{L}_\text{DR}$ outperforms $\mathcal{L}_\text{CE}$ on \textcolor{red}{observed} classes in terms of alignment gap and alignment, while $\mathcal{L}_\text{CE}$ achieves better results on \textcolor{blue}{unobserved} classes for these metrics. The improvement on \textcolor{red}{observed} classes is attributed to $\mathcal{L}_\text{DR}$, which removes the pushing term present in $\mathcal{L}_\text{CE}$ and concentrates its pulling effects on these classes. However, this design inherently overlooks \textcolor{blue}{unobserved} classes, resulting in poorer performance on these classes compared to $\mathcal{L}_\text{CE}$.\footnote{A theoretical justification for why $\mathcal{L}_\text{DR}$ struggles with \textcolor{blue}{unobserved} classes is provided in Subsection~\ref{ch3app:ntk}, where we analyze feature gradients under the NTK framework.}


\myparagraph{Accuracy result of global models.}
Figure~\ref{ch3fig:problem_global_acc} shows that in the shard setting ($s=10$), $\mathcal{L}_\text{CE}$ consistently outperforms $\mathcal{L}_\text{DR}$. This difference is due to the higher proportion of \textcolor{blue}{unobserved} in this setting, where each client has access to at most 10 out of 100 classes. 

\begin{wrapfigure}[13]{r}{0.45\textwidth}
    \begin{minipage}{0.45\textwidth}
        \centering
        \vspace{-2pt}
        \includegraphics[width=\textwidth]{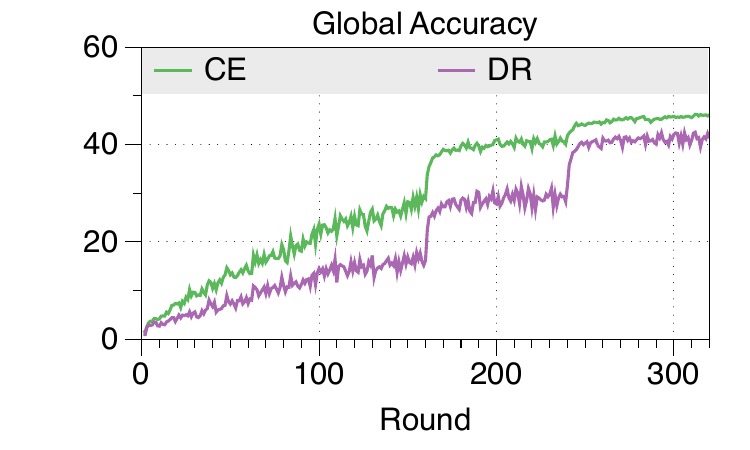}
        \vspace{-24pt}
        \caption{Comparison of the global test accuracy of $\bm{\theta}_r^g$ on all classes trained using $\mathcal{L}_\text{CE}$ and $\mathcal{L}_\text{DR}$.}
        \label{ch3fig:problem_global_acc}
    \end{minipage}
\end{wrapfigure}

As observed in the alignment analysis, $\mathcal{L}_\text{DR}$ performs particularly poorly on the \textcolor{blue}{unobserved}, with a significantly negative alignment gap and lower alignment compared to models trained with $\mathcal{L}_\text{CE}$, which contributes to the lower overall accuracy of the global model when using $\mathcal{L}_\text{DR}$.

\myparagraph{Importance of local alignment in observed classes.}
While $\mathcal{L}_\text{DR}$ struggles with local alignment on \textcolor{blue}{unobserved} classes, leading to lower global accuracy compared to $\mathcal{L}_\text{CE}$, local alignment in \textcolor{red}{observed} classes remains critical. At the last learning rate decay round ($240^\text{th}$), both methods achieve higher or comparable performance on \textcolor{blue}{unobserved} classes than in the final round ($320^\text{th}$). However, the final round shows higher global accuracy due to improved local alignment on \textcolor{red}{observed} classes. This underscores the need to preserve the advantages of $\mathcal{L}_\text{DR}$ in \textcolor{red}{observed} classes while mitigating its degradation on \textcolor{blue}{unobserved} classes.

\subsection{\alg: $\mc{L}_{\text{DR}}$ With Feature Distillation for FL}\label{ch3subsec:method}

We propose \alg to mitigate forgetting unobserved classes while retaining the strengths of dot-regression loss in aligning features of observed classes. Using $\mathcal{L}_\text{DR}$ with the frozen classifier $\bm{V}$, \alg includes a regularizer that fully distills the global model's feature vectors $f(x;\bm{\theta}^g)\in\mathbb{R}^{d}$ to the client features $f(x;\bm{\theta})$, to enhance generalization across all classes. The proposed loss function $\mc{L}_{\text{Dr+}}$, shown in~\eqref{ch3alg:FedDR}, combines $\mc{L}_\text{DR}$ with a regularizer $\mc{L}_\text{FD}(x; \bm{\theta}, \bm{\theta}^{g})=\frac{1}{d}\|f(x;\bm{\theta})-f(x;\bm{\theta}^g)\|_2^2$. Unless specified, we use a scaling parameter $\beta = 0.9$ throughout the paper. The overall pseudocode of \alg can be found in Appendix~\ref{ch3app:pseudo_code}.

\begin{equation}
    \label{ch3alg:FedDR}
    \mc{L}_{\text{Dr+}}(x,y; \bm{\theta}, \bm{\theta}^g, \bm{V})=\beta \cdot \mc{L}_{\text{DR}}(x,y;\bm{\theta}, \bm{V})
     +(1-\beta)\cdot \mc{L}_{\text{FD}}(x; \bm{\theta}, \bm{\theta}^{g})
\end{equation}

\myparagraph{Why feature distillation?}
To address data heterogeneity in FL, various distillation methods have been explored, including model parameters~\citep{oh2021fedbabu, MLSYS2020_1f5fe839, he2020group, li2019fedmd}, logit-related measurement~\citep{li2019fedmd, lee2022preservation, itahara2021distillation, ye2023fake, lin2020ensemble, chen2019knowledge, qian2022switchable}, and co-distillation~\citep{chen2024spectral, cho2023communication}. 
In contrast, we utilize the \textit{feature} distillation~\citep{heo2019comprehensive} technique because the feature directly concerns alignment. On the other hand, logits lose information from features when projected onto a frozen ETF classifier~\citep{heo2019comprehensive, li2017mimicking, li2023rethinking, ben2022s}.
 By distilling features, we leverage the global, differentiated knowledge for each data input $x$. This approach aims to minimize drift towards observed classes, and hence, we expect it to enhance overall generalization.

\subsection{Effect of Feature Distillation}\label{ch3subsec:effect_fd}
Our findings from Section~\ref{ch3subsec:dotreg} indicate that $\mathcal{L}_\text{DR}$ is unsuitable for the heterogeneous FL environment. This is primarily because there is a notable gap in how features align with the fixed classifier between $\mc{O}^i$ and $\mc{U}^i$. To assess the effect of feature distillation ($\mathcal{L}_{\text{FD}}$), which imposes a constraint on the feature distance $\|f(x; \bm{\theta}_{r}^i) - f(x; \bm{\theta}_{r-1}^g)\|_2$ for $x \in \mc{O}^i$, we measure this distance for both $\mc{O}^i$ and $\mc{U}^i$ from the models trained with $\mathcal{L}_{\text{DR}}$ and $\mathcal{L}_{\text{Dr+}}$. We additionally analyze the angle distance, $\angle(f(x; \bm{\theta}_{r}^i), f(x; \bm{\theta}_{r-1}^g))$, as it impacts feature-classifier alignment. These values are averaged over the selected client set $\mathcal{S}_r$.

\myparagraph{Feature distillation stabilizes the feature dynamics.}
By adding $\mathcal{L}_\text{FD}$, as revealed in Figure~\ref{ch3fig:feature_distance_dynamic}, the local model trained with $\mathcal{L}_\text{Dr+}$ shows a reduction in feature distance for \textcolor{red}{observed} classes, compared to the model trained with $\mathcal{L}_\text{DR}$. This reduction happens even for \textcolor{blue}{unobserved} classes. As demonstrated in Figure~\ref{ch3fig:feature_angle_dynamics}, the overall decrease in feature distance leads to a reduction in feature angle distance for both class sets. In both local models trained with $\mathcal{L}_\text{DR}$ and $\mathcal{L}_\text{Dr+}$, there is a trend where the angle distance is significantly larger for $\mc{U}^i$ than for $\mc{O}^i$(Figure~\ref{ch3fig:feature_angle_dynamics}). This large angle distance of $\mc{U}^i$ leads to the degradation of the feature-classifier alignment.

\begin{figure}[t]
    \centering
        \begin{minipage}{\textwidth}
            \centering    
            \begin{subfigure}[b]{0.47\textwidth}
                \centering
                \includegraphics[width=\textwidth]{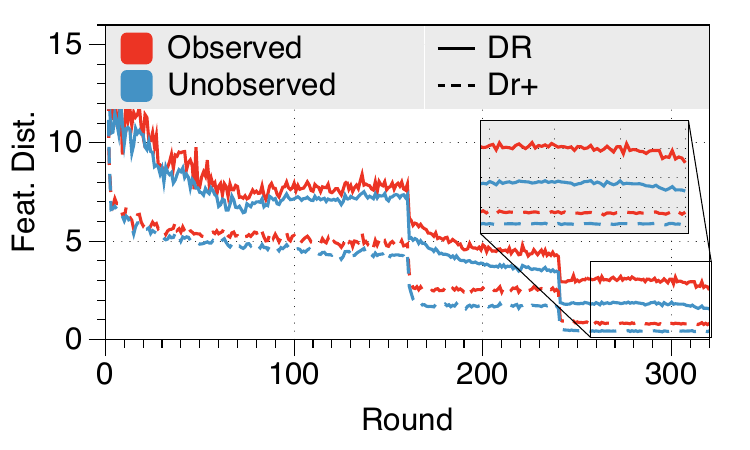}
                \subcaption{Feature distance}
                \label{ch3fig:feature_distance_dynamic}
            \end{subfigure}
            \hspace{15pt}
            \begin{subfigure}[b]{0.47\textwidth}
                \centering
                    \includegraphics[width=\textwidth]{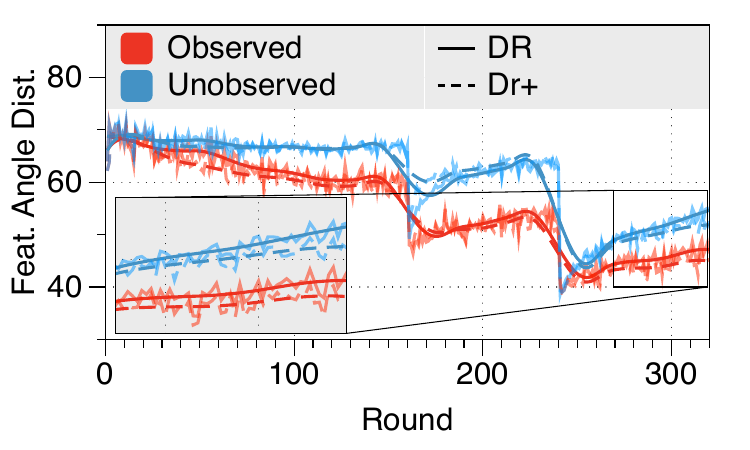}  
                    \subcaption{Feature angle distance}
                    \label{ch3fig:feature_angle_dynamics}
            \end{subfigure}
            \caption{We present (a) feature distance and (b) feature angle distance from $\bm{\theta}_{r-1}^g$ to $\bm{\theta}_{r}^i$ for \textcolor{red}{observed} and \textcolor{blue}{unobserved} classes  by training with $\mathcal{L}_\text{DR}$ and $\mathcal{L}_\text{Dr+}$.}
            \label{ch3fig:feature_dynamic_change}
        \end{minipage}
        
\end{figure}

\begin{figure}[t!]
    \centering
        \begin{minipage}{0.49\textwidth}
            \hspace{-15pt}
            \includegraphics[width=\textwidth]{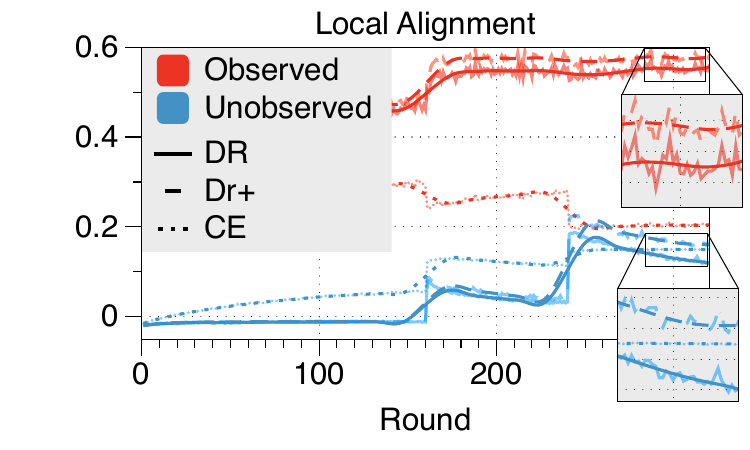}
            \caption{Comparison of feature-classifier alignment on the \textcolor{red}{observed} and \textcolor{blue}{unobserved} classes test data for $\bm{\theta}_r^i$ trained with $\mathcal{L}_\text{CE}, \mathcal{L}_\text{DR}$ and $\mathcal{L}_\text{Dr+}$.}
            \label{ch3fig:all_local_alignment}      
        \end{minipage}
        \hspace{3pt}
        \begin{minipage}{0.49\textwidth}
            \includegraphics[width=\textwidth]{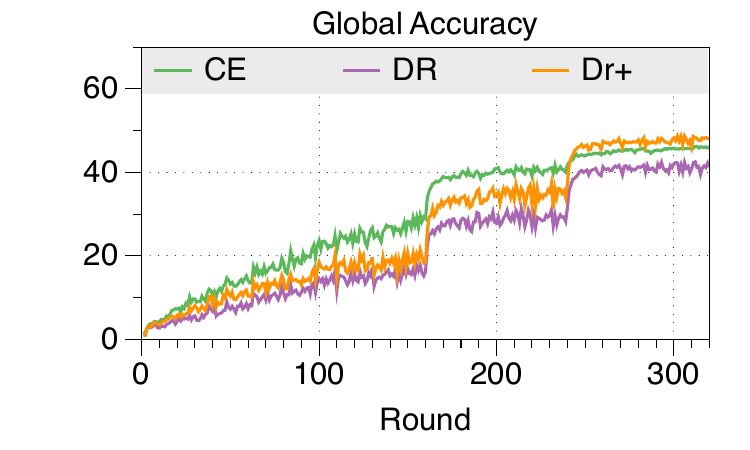}
            \caption{Comparison of the global test accuracy of $\bm{\theta}_r^g$ on all classes trained using $\mathcal{L}_{\text{CE}}, \mathcal{L}_{\text{DR}}$ and $\mathcal{L}_\text{Dr+}$.}
            \label{ch3fig:all_global_acc}   
        \end{minipage}
        
\end{figure}

\myparagraph{Feature distillation enhances local alignment and global accuracy.}
As in Figure~\ref{ch3fig:all_local_alignment}--\ref{ch3fig:all_global_acc}, our proposed algorithm, \ie $\mathcal{L}_{\text{Dr+}}$, improves feature-classifier alignment for both $\mathcal{O}^i$ and $\mathcal{U}^i$, along with enhanced global accuracy compared to $\mathcal{L}_{\text{DR}}$. We attribute this improvement to the enhanced knowledge of the global model which is preserved by preventing the forgetting of previously trained knowledge.  During the final convergence phase ($240^\text{th}$-$320^\text{th}$), $\mathcal{L}_{\text{Dr+}}$ achieves \textcolor{blue}{better feature-classifier alignment even on unobserved} classes compared to $\mathcal{L}_\text{CE}$. As a result, at the final round ($320^\text{th}$), $\mathcal{L}_{\text{Dr+}}$ demonstrates the best local alignment across all classes, surpassing both $\mathcal{L}_\text{CE}$ and $\mathcal{L}_{\text{DR}}$, contributing to its global model's superior performance. Even though the proposed regularizer demonstrates a reasonable regularizing effect, one question remains: ``\emph{Is it superior to other previously used regularizers?}''

\subsection{Synergistic Effect with Different Types of FL Algorithms and Regularizers}\label{ch3subsec:synergy_effect}

We answer the above question by evaluating the synergy effect of various FL algorithms by maintaining their original training loss and incorporating specific regularizers, following the approach suggested in Equation~\ref{ch3alg:FedDR}. To address the issue of differing loss scales between the baseline FL algorithms and the regularizers, we thoroughly tune the coefficient $\beta$ within the range $\{0.1, 0.3, 0.5, 0.7, 0.9, 0.99, 0.999, 0.9999\}$, and report the resulting performance in Table~\ref{ch3tab:synergy_effect}. The selected $\beta$ values are detailed in Subsection~\ref{ch3app:synergey}.

\myparagraph{FL algorithms and regularizers.}
We evaluate several baseline FL algorithms, including dot-regression. 
The baseline also include FedAvg~\citep{mcmahan2017communication}—using an unfrozen classifier, FedBABU~\citep{oh2021fedbabu}—extending FedAvg by freezing the classifier during local training, and SphereFed\citep{dong2022spherefed}—enhancing feature-classifier alignment by using MSE loss between one-hot encoded labels and cosine similarity based logits. We also considered FedGELA and FedETF—both applying client-specific adaptive loss functions tailored to data distribution.  

Alongside the FD regularizer, we evaluate a range of regularizers, including Prox~\citep{MLSYS2020_1f5fe839}—constraining the distance between local and global model parameters; MOON~\citep{li2021model}—minimizing the angle distance between feature vectors of global and local models through contrastive learning; and several logit-based regularizers—KD~\citep{hinton2015distilling}, NTD~\citep{lee2022preservation, zhao2022decoupled}, and LD~\citep{kim2021comparing}—keeping logit-related measurements of local models closely aligned with the global model. Specifically, KD applies softened softmax probability from the logit vector, NTD does the same but excludes the true class dimension, and LD distills the entire logit vector.

\begin{table*}[t!]
    \centering
    \caption{Synergy of various FL algorithms and regularizers. Baseline indicates training FL models without a regularizer. FD denotes feature distillation, which is the regularizer we use in \alg.}
    \label{ch3tab:synergy_effect}
    \small
    \addtolength{\tabcolsep}{-1pt}
    \resizebox{\textwidth}{!}{
    \begin{tabular}{l|ccccccc|ccccccc}
      \toprule 
       & \multicolumn{7}{c|}{Sharding ($s=10$)} & \multicolumn{7}{c}{LDA ($\alpha=0.1$)} \\ \cmidrule{2-15}
      Algorithm & Baseline & \!\!+Prox & \!\!+KD & \!\!+NTD & \!\!+LD & \!\!+MOON & +FD & Baseline & \!\!+Prox & \!\!+KD & \!\!+NTD & \!\!+LD & \!\!+MOON & +FD\\ \midrule
      FedAvg & 37.22 &  36.87 & 36.25 & 37.71  & 37.17 & 37.43 &  37.82 & 42.52 &  43.22 & 44.21 & 43.39 & 43.43 & 44.79 & 43.76\\ \midrule      
      FedBABU  & 46.20 &  46.03 & 46.37 & 47.22  & 46.71 & 46.49 &  46.95 & 47.37 &  46.62 & 47.60 & 46.48 & 45.78 & 46.27 & 46.49\\   
   
      SphereFed & 43.90 &  41.96 & 44.94 & 43.47  & 43.95 & 43.13 &  45.21 & 46.98 &  43.77 & 47.76 & 47.25 & 47.01 & 46.81 & 49.74\\   
      
      FedETF & 32.42 &  31.87 & 32.76 & 32.65  & 32.25 & 34.30 &  32.77 & 46.27 &  45.71 & 46.67 & 46.16 & 45.91 & 45.98 & 46.47\\   
      
      FedGELA & 29.17 &  28.69 & 29.11 & 28.84  & 29.36 & 28.80 &  30.33 & 27.11 &  29.03 & 28.45 & 29.62 & 29.41 & 28.09 & 29.75\\ \midrule  
      
      Dot-Regression & 42.52 &  41.95 & 47.45 & 48.32  & 47.52 & 44.72 &  \textbf{48.69} & 42.72 &  46.35 & 49.47 & 50.36 & 49.28 & 50.36 & \textbf{50.86}\\   
      
      \bottomrule
    \end{tabular}
    }
    \vspace{-5pt}
\end{table*}

\myparagraph{\alg shows best synergy.}
Table~\ref{ch3tab:synergy_effect} shows that \alg (dot-regression + FD) achieves the strongest performance among the tested combinations. FD performs exceptionally well when combined with dot-regression, SphereFed, and FedBABU. These baselines freeze the classifier while not using client-specific adaptive loss. Among them, the synergy is most effective in the order of dot-regression, SphereFed, and FedBABU—reflecting how each optimizes feature-classifier alignment well. FD outperforms other regularizers by effectively stabilizing feature dynamics for observed classes, which helps mitigate the misalignment and performance degradation for unobserved classes. MOON,on the other hand, prioritizes the cosine similarity between feature vectors from local and global models but fails to adequately control feature norm dynamics. Prox applies uniform regularization that is independent of specific data instances, leading to less refined control over feature dynamics. Logit-based regularizers lose effectiveness due to information loss when features are projected onto the classifier, as they focus on mitigating the dynamics of the less informative projected vectors rather than the richer original feature vectors.

For baselines like FedGELA and FedETF, which apply client-specific loss functions, none of the regularizers, including FD, lead to consistently lead to significant performance, particularly in the sharding setting. In non-frozen settings, FedAvg, the FD regularizer does not offer a significant performance boost. Under LDA, FedAvg combined with MOON outperforms FedAvg with FD, consistent with the claim of MOON~\citep{li2021model}.

\section{Experiments and Results}
\label{ch3sec:exp}

In this section, we first present the experimental results of \alg in the context of global federated learning (GFL). We then analyze the elapsed time and conduct a sensitivity analysis for GFL, investigating the effects of varying local epochs, client sampling ratios, and different $\beta$ values on the performance of \alg. Finally, we propose \alg\ FT, which fine-tunes the \alg GFL model with $\mathcal{L}_\text{Dr+}$, and report the results comparing it with existing PFL methods.

\subsection{Experimental Setup}\label{ch3subsec:exp_setup}

\myparagraph{Dataset and models.}
To simulate a realistic FL scenario involving 100 clients, we conduct extensive studies on three widely used datasets: CIFAR-10~\citep{krizhevsky2009cifar}, CIFAR-100~\citep{krizhevsky2009cifar} and ImageNet-100~\citep{deng2009imagenet}. We use VGG11~\citep{simonyan2014very} for CIFAR-10, MobileNet~\citep{howard2017mobilenets} for CIFAR-100, and ResNet-18~\citep{he2016identity} for ImageNet-100. The training data is distributed among the 100 clients using sharding and the LDA (Latent Dirichlet Allocation) partition strategies. 

Following the convention, sharding distributes the data into non-overlapping shards of equal size, each shard encompassing $\frac{|D_\text{train}|}{100\times s}$ and $\frac{|D_\text{test}|}{100\times s}$ samples per class, where $s$ denotes the number of shards per client. On the other hand, LDA involves sampling a probability vector from Dirichlet distribution, $p_c = (p_{c,1}, p_{c,2}, \cdots, p_{c,100}) \sim \text{Dir}(\alpha)$, and allocating a proportion $p_{c,k}$ of instances of class $c \in [C]$ to each client $k \in [100]$. Smaller values of $s$ and $\alpha$ increase the level of data heterogeneity. For CIFAR-10 and CIFAR-100, we explore a range of $s$ and $\alpha$ values to assess the impact of different data heterogeneity levels. For ImageNet-100, we focus on experiments with $s=20$ and $\alpha=0.1$.

\myparagraph{Implementation details.} In each round of communication, a random 10\% of clients are selected to participate in the training process. The total number of communication rounds is set to 320. The initial learning rate and the number of local epochs for CIFAR-10, CIFAR-100, and ImageNet-100 are determined through grid searches, with the detailed process and results provided in Subsection~\ref{ch3app:grid_search}. The learning rate $\eta$ is decayed by a factor of 0.1 at the 160th and 240th communication rounds.

\begin{table*}[t!]
    \centering
    \small
    \caption{Accuracy comparison in the GFL setting. 
    The entries are based on results obtained from three different seeds, indicating the mean and standard deviation of the accuracy of the global model, represented as X{\tiny $\pm$Y}. The best performance in each case is highlighted in \textbf{bold}.}
    \label{ch3tab:gfl_acc_all}
    \vspace{-5pt}
    \resizebox{\textwidth}{!}{
    \begin{tabular}{l|cccc|ccc|c}
      \toprule 
      \multicolumn{9}{c}{\textbf{NIID Partition Strategy: Sharding}}\\ \midrule
       & \multicolumn{4}{c|}{CIFAR-100} & \multicolumn{3}{c|}{CIFAR-10} & \multirow{2}{*}{\centering ImageNet-100}\\ \cmidrule{2-8} 
       & $s$=10 & $s$=20 & $s$=50 & $s$=100  & $s$=2 & $s$=5 & $s$=10 & \\ \midrule
      FedAvg  & 36.63{\tiny $\pm$ 0.22} & 42.25{\tiny $\pm$ 1.42} & 45.57{\tiny $\pm$ 0.22} & 48.20{\tiny $\pm$ 1.36} & 72.08{\tiny $\pm$ 0.67} & 81.53{\tiny $\pm$ 0.35} & 82.38{\tiny $\pm$ 0.40} & 67.78{\tiny $\pm$ 0.41}   \\ 

      FedProx & 37.07{\tiny $\pm$ 0.21} & 42.35{\tiny $\pm$ 0.83} & 45.18{\tiny $\pm$ 1.07} & 47.78{\tiny $\pm$ 0.79} & 71.92{\tiny $\pm$ 0.51} & 81.29{\tiny $\pm$ 0.40} & 82.45{\tiny $\pm$ 0.35} & 67.81{\tiny $\pm$ 0.65}  \\      
      SCAFFOLD\,($\times 2$)\!\!
      & 46.08{\tiny $\pm$ 0.37} & 48.15{\tiny $\pm$ 1.21} & 49.31{\tiny $\pm$ 0.62} & 50.73{\tiny $\pm$ 0.42} & 75.49{\tiny $\pm$ 0.42} & \textbf{84.14}{\tiny $\pm$ 0.13} & \textbf{85.11}{\tiny $\pm$ 0.29} & 70.47{\tiny $\pm$ 0.46} \\

      MOON & 36.95{\tiny $\pm$ 0.37} & 43.05{\tiny $\pm$ 0.27} & 43.95{\tiny $\pm$ 0.12} & 46.92{\tiny $\pm$ 0.08} & 67.55{\tiny $\pm$ 1.16} & 80.90{\tiny $\pm$ 0.26} & 82.62{\tiny $\pm$ 0.31} & 68.19{\tiny $\pm$ 0.32} \\      

      FedNTD  & 34.05{\tiny $\pm$ 1.19} & 41.78{\tiny $\pm$ 0.31} & 46.42{\tiny $\pm$ 0.63} & 47.17{\tiny $\pm$ 0.32} & 72.21{\tiny $\pm$ 0.59} & 69.96{\tiny $\pm$ 17.10} & 81.99{\tiny $\pm$ 0.42} & 67.51{\tiny $\pm$ 0.25} \\
      FedExP  & 36.85{\tiny $\pm$ 0.11} & 42.49{\tiny $\pm$ 1.22} & 45.07{\tiny $\pm$ 0.92} & 48.09{\tiny $\pm$ 1.00} & 72.31{\tiny $\pm$ 0.60} & 81.41{\tiny $\pm$ 0.19} & 82.47{\tiny $\pm$ 0.16} & 63.34{\tiny $\pm$ 0.51} \\
      
      FedSOL & 32.18{\tiny $\pm$ 0.18} & 41.54{\tiny $\pm$ 1.03} & 47.42{\tiny $\pm$ 0.76} & 47.70{\tiny $\pm$ 1.13} & 54.93{\tiny $\pm$ 3.52} & 75.73{\tiny $\pm$ 0.10} & 77.00{\tiny $\pm$ 0.41} & 66.61{\tiny $\pm$ 1.17} \\\midrule

      FedBABU  & 45.97{\tiny $\pm$ 0.48} & 45.53{\tiny $\pm$ 0.79} & 46.52{\tiny $\pm$ 0.51} & 46.02{\tiny $\pm$ 0.28} & 71.99{\tiny $\pm$ 0.52} & 81.07{\tiny $\pm$ 0.60} & 82.32{\tiny $\pm$ 0.06} & 68.82{\tiny $\pm$ 0.46} \\
      
      SphereFed & 42.71{\tiny $\pm$ 0.65} & 48.63{\tiny $\pm$ 0.90} & {\textbf{52.16}}{\tiny $\pm$ 0.22} & \textbf{53.41}{\tiny $\pm$ 0.19} & \textbf{76.33}{\tiny $\pm$ 0.33} & 83.67{\tiny $\pm$ 0.18} & 84.36{\tiny $\pm$ 0.30} & 69.71{\tiny $\pm$ 0.39} \\
      FedETF & 31.37{\tiny $\pm$ 0.72} & 42.22{\tiny $\pm$ 0.77} & 47.47{\tiny $\pm$ 0.67} & 49.00{\tiny $\pm$ 0.74} & 67.81{\tiny $\pm$ 0.94} & 80.78{\tiny $\pm$ 0.68} & 82.60{\tiny $\pm$ 0.46} & 70.81{\tiny $\pm$ 0.28} \\ 
      FedGELA & 27.95{\tiny $\pm$ 0.81} & 38.63{\tiny $\pm$ 0.66} & 44.67{\tiny $\pm$ 0.51} & 47.95{\tiny $\pm$ 0.85} & 63.77{\tiny $\pm$ 2.34} & 79.05{\tiny $\pm$ 0.14} & 81.56{\tiny $\pm$ 0.09}& 67.08{\tiny $\pm$ 0.18} \\
      \alg \textbf{(Ours)} & \textbf{48.21}{\tiny $\pm$ 0.56} & \textbf{50.77}{\tiny $\pm$ 0.14} & {\textbf{52.15}}{\tiny $\pm$ 0.03} & \textbf{52.41}{\tiny $\pm$ 0.81} & \textbf{76.57}{\tiny $\pm$ 0.51} & 83.22{\tiny $\pm$ 0.34} & 84.14{\tiny $\pm$ 0.27} & \textbf{71.47}{\tiny $\pm$ 0.45} \\ \bottomrule
      \midrule
      \multicolumn{9}{c}{\textbf{NIID Partition Strategy: LDA}}\\ \midrule
       & \multicolumn{4}{c|}{CIFAR-100} & \multicolumn{3}{c|}{CIFAR-10} & \multirow{2}{*}{\centering ImageNet-100}\\ \cmidrule{2-8} 
       & $\alpha$=0.05 & $\alpha$=0.1 & $\alpha$=0.2 & $\alpha$=0.3  & $\alpha$=0.1 & $\alpha$=0.2 & $\alpha$=0.3 & \\ \midrule
      FedAvg  & 35.58{\tiny $\pm$ 1.35} & 42.10{\tiny $\pm$ 0.60} & 44.78{\tiny $\pm$ 0.72} & 45.73{\tiny $\pm$ 0.88} & 68.71{\tiny $\pm$ 1.82} & 77.75{\tiny $\pm$ 0.26} & 80.76{\tiny $\pm$ 0.51} & 65.11{\tiny $\pm$ 0.25}   \\ 
      
      FedProx  & 37.07{\tiny $\pm$ 0.21} & 42.35{\tiny $\pm$ 0.83} & 45.18{\tiny $\pm$ 1.06} & 48.18{\tiny $\pm$ 0.51} & 69.00{\tiny $\pm$ 2.27} & 77.81{\tiny $\pm$ 0.24} & 80.55{\tiny $\pm$ 0.19} & 40.48{\tiny $\pm$ 1.28}   \\ 
      SCAFFOLD\,($\times 2$)\!\! & 40.54{\tiny $\pm$ 0.48} & 46.14{\tiny $\pm$ 0.70} & 47.98{\tiny $\pm$ 0.93} & 48.06{\tiny $\pm$ 1.08} & \textit{(Failed)} & 80.15{\tiny $\pm$ 0.29} & {\textbf{82.63}}{\tiny $\pm$ 0.23} & 66.84{\tiny $\pm$ 0.77}   \\ 
      MOON  & 23.97{\tiny $\pm$ 1.15} & 30.86{\tiny $\pm$ 0.21} & 33.60{\tiny $\pm$ 0.62} & 35.54{\tiny $\pm$ 0.45} & 66.44{\tiny $\pm$ 3.28} & 77.36{\tiny $\pm$ 0.08} & 80.15{\tiny $\pm$ 0.22} & 41.25{\tiny $\pm$ 0.60}   \\ 
      FedNTD  & 31.78{\tiny $\pm$ 3.14} & 40.41{\tiny $\pm$ 0.96} & 43.10{\tiny $\pm$ 2.03} & 43.04{\tiny $\pm$ 0.82} & 70.22{\tiny $\pm$ 0.40} & 77.16{\tiny $\pm$ 0.20} & 79.50{\tiny $\pm$ 0.56} & 64.87{\tiny $\pm$ 0.20}   \\ 
      FedExP  & 34.39{\tiny $\pm$ 1.77} & 40.85{\tiny $\pm$ 1.32} & 44.47{\tiny $\pm$ 0.28} & 45.44{\tiny $\pm$ 0.14} & 70.14{\tiny $\pm$ 0.53} & 78.09{\tiny $\pm$ 0.21} & 80.40{\tiny $\pm$ 0.54}& 59.40{\tiny $\pm$ 0.36}   \\

      FedSOL & 34.49{\tiny $\pm$ 0.80} & 41.19{\tiny $\pm$ 0.30} & 43.55{\tiny $\pm$ 1.51} & 44.85{\tiny $\pm$ 0.54} & 59.51{\tiny $\pm$ 1.77} & 67.55{\tiny $\pm$ 0.41} & 70.96{\tiny $\pm$ 0.32} & 62.70{\tiny $\pm$ 0.89}   \\ \midrule
      
      FedBABU  & 41.97{\tiny $\pm$ 1.01} & 45.77{\tiny $\pm$ 0.28} & 44.28{\tiny $\pm$ 0.45} & 44.80{\tiny $\pm$ 0.63} & 65.15{\tiny $\pm$ 3.66} & 77.03{\tiny $\pm$ 0.25} & 79.91{\tiny $\pm$ 0.13} & 66.54{\tiny $\pm$ 0.30}   \\ 
      
      SphereFed & 39.56{\tiny $\pm$ 0.48} & 46.54{\tiny $\pm$ 0.58} & \textbf{49.41}{\tiny $\pm$ 0.78} & 49.22{\tiny $\pm$ 0.86} & 67.49{\tiny $\pm$ 3.49} & 80.05{\tiny $\pm$ 0.40} & {\textbf{82.62}}{\tiny $\pm$ 0.66} & 67.03{\tiny $\pm$ 0.30}   \\ 
      FedETF & 40.71{\tiny $\pm$ 0.90} & 45.63{\tiny $\pm$ 0.33} & 46.28{\tiny $\pm$ 1.05} & 46.69{\tiny $\pm$ 0.87} & 70.75{\tiny $\pm$ 0.36} & 77.86{\tiny $\pm$ 0.46} & 79.95{\tiny $\pm$ 0.34} & 68.98{\tiny $\pm$ 0.21}   \\ 

      FedGELA & 16.72{\tiny $\pm$ 1.91} & 27.12{\tiny $\pm$ 1.58} & 33.68{\tiny $\pm$ 0.19} & 36.17{\tiny $\pm$ 0.26} & 50.69{\tiny $\pm$ 7.55} & 66.04{\tiny $\pm$ 14.87} & 77.89{\tiny $\pm$ 0.97} & 55.57{\tiny $\pm$ 0.42}   \\ 
      \alg \textbf{(Ours)} & \textbf{45.12}{\tiny $\pm$ 1.00} & \textbf{49.48}{\tiny $\pm$ 0.50} & \textbf{50.67}{\tiny $\pm$ 0.88} & \textbf{51.15}{\tiny $\pm$ 0.65} & \textbf{72.07}{\tiny $\pm$ 2.26} & \textbf{80.90}{\tiny $\pm$ 0.02} & {\textbf{82.42}}{\tiny $\pm$ 0.10} & \textbf{70.20}{\tiny $\pm$ 0.09}   \\  \bottomrule
    \end{tabular}
    }
\end{table*}

\vspace{-5pt}
\subsection{Global Federated Learning Results}
\label{ch3subsec:gfl_results}

We compare \alg with a range of GFL algorithms, considering both non-freezing and freezing classifier approaches. Among non-freezing classifiers, \alg competes with FedAvg~\citep{mcmahan2017communication}, FedProx~\citep{MLSYS2020_1f5fe839}, SCAFFOLD~\citep{karimireddy2020scaffold}, MOON~\citep{li2021model}, FedNTD~\citep{lee2022preservation}, FedExP~\citep{jhunjhunwala2023fedexp}, and FedSOL~\citep{lee2024fedsol}. \alg is also evaluated against freezing classifier algorithms such as FedBABU~\citep{oh2021fedbabu}, SphereFed~\citep{dong2022spherefed}, FedETF~\citep{li2023no}, and FedGELA~\citep{fan2023federated}. Among the baseline algorithms, SCAFFOLD incurs a communication cost two times higher per round, denoted as ($\times2$).  Our experiments encompass heterogeneous settings involving sharding and LDA non-IID environments. 

Table~\ref{ch3tab:gfl_acc_all} summarizes the accuracy comparison between various GFL methods proposed in recent literature and FedAvg under various conditions.
While specific methods demonstrated effectiveness in particular scenarios, some of these underperformed relative to the robustness of FedAvg. For example, SCAFFOLD shown strong performance in the less heterogeneous sharding setting on CIFAR-10; however, it failed in model training under the highly heterogeneous LDA condition with $\alpha=0.1$.
Notably, \alg consistently outperformed all baselines in highly heterogeneous settings, achieving a 3.15\% improvement in CIFAR-100 LDA with $\alpha=0.05$ and 3.17\% in ImageNet-100 LDA.

\subsection{Elapsed Time Results}
\label{ch3subsec:elapsed_time}

We compare \alg with various GFL algorithms for the elapsed time per communication round on CIFAR-100 ($s$=10). As shown in Table~\ref{ch3apptab:elapsed_time}, incorporating a global model during the local update generally results in higher computation costs, leading to longer elapsed times compared to updates without a global model. \alg exhibits a slightly longer elapsed time than most other algorithms, yet still requires less time than FedSOL and MOON.

\begin{table}[t!]
    \centering
    \small
    \caption{Elapsed time per round (in seconds) for various GFL algorithms.}
    \label{ch3apptab:elapsed_time}
    \vspace{-5pt}
    \addtolength{\tabcolsep}{-2pt}
    \resizebox{\textwidth}{!}{
    \begin{tabular}{c|cccccc|cccccc}
      \toprule
    & \multicolumn{6}{c|}{Local update without global model} & \multicolumn{6}{c}{Local update with global model} \\ \midrule      
      & FedAvg & FedBABU & FedETF & SphereFed & FedGELA & FedExP & FedProx & SCAFFOLD & FedNTD  & FedSOL & MOON & \alg (Ours) \\ \midrule
      Elapsed time  & 20.4 & 20.5 & 20.7 & 20.4 & 20.5 & 20.1 & 22.6 & 21.4 & 21.2  & 27.2 & 56.7 & 24.9 \\ \bottomrule
    \end{tabular}
    }
\end{table}

\subsection{Sensitivity Analysis}\label{ch3subsec:sensitivity}

\begin{figure*}[t!]\label{ch3fig: Sensitivity}
    \centering
    \vspace{5pt}
    \includegraphics[width=0.6\textwidth]{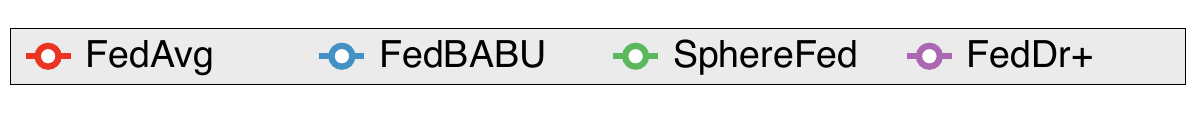} \\
    \begin{subfigure}[b]{0.32\textwidth}
        \centering
        \includegraphics[width=1\textwidth]{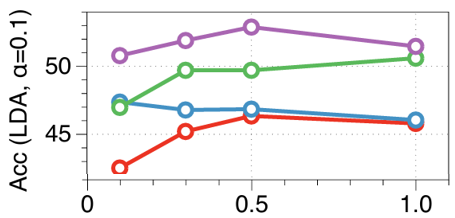}
        \vspace{-7pt}
        \subcaption{Client sampling ratio}
    \end{subfigure}
    \hfill
    \begin{subfigure}[b]{0.32\textwidth}
        \centering
        \includegraphics[width=1\textwidth]{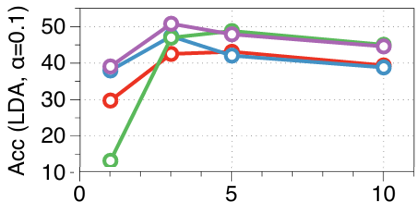}
        \vspace{-7pt}
        \subcaption{Local epochs}
    \end{subfigure}
    \hfill
    \begin{subfigure}[b]{0.32\textwidth}
        \centering
        \includegraphics[width=0.95\textwidth]{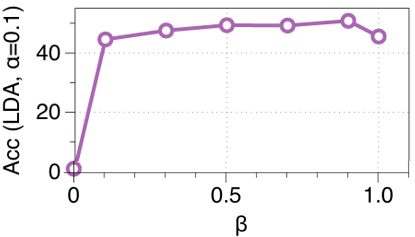}
        \vspace{-3pt}
        \subcaption{$\beta$ sensitivity}
    \end{subfigure}
    \caption{Performance of baselines and \alg on CIFAR-100 ($\alpha$=0.1) with various analyses: (a) client sampling ratio, (b) the number of local epochs, and (c) sensitivity to $\beta$.}
    \vspace*{-5pt}
    \label{ch3fig:sensitivity}
\end{figure*}

We explore the impact of varying client sampling ratio and local epochs on performance, as well as the effect of different $\beta$ values in \alg, as detailed in Figure~\ref{ch3fig:sensitivity}. All experiments are conducted on MobileNet using the CIFAR-100 dataset with LDA ($\alpha$=0.1).
\newline

\myparagraph{Effect of client sampling ratio and local epochs.} We evaluate the sensitivity of hyperparameters in \alg by comparing it to baselines under varying client sampling ratio and local epochs, starting from the default setting of client sampling ratio of 0.1 and local epoch of 3. Compared to FedAvg (without classifier freezing), FedBABU and SphereFed (all with classifier freezing) generally show performance improvements with increasing fraction ratios, but \alg consistently outperforms the baselines. The number of local epochs is crucial in FL; too few epochs result in underfitting, while too many cause client drift, degrading global model performance. The default setting of local epochs 3 is optimal for all baselines, with \alg achieving the best performance. Although performance generally declines when deviating from this peak point, \alg remains the best or highly competitive.
\newline

\myparagraph{Weight ratio $\beta$ analysis.} We analyze the effect of scaling parameter in \alg by varying $\beta$ while keeping other hyperparameters constant. The performance is evaluated for $\beta\in\{0, 0.1, 0.3, 0.5, 0.7, 0.9, 1.0\}$. When $\beta=0$, only feature distillation is applied, and when $\beta=1$, only dot-regression is used. $\beta\in\{0, 1\}$ are generally less effective, whereas $\beta\in\{0.3, 0.5, 0.7, 0.9\}$ show consistently good performance, indicating a balanced approach is beneficial.

\begin{table}[t!]
    \centering
    \caption{PFL accuracy comparison with MobileNet on CIFAR-100. Results are reported in the format X{\tiny $\pm$Y}, representing the mean and standard deviation of the average personalized accuracies across all clients, computed over five seeds. The best performance in each case is highlighted in \textbf{bold}.}
    
    \label{ch3tab:pfl_acc}
    \small
    \resizebox{\textwidth}{!}{
    \begin{tabular}{l|ccc|ccc}
    \toprule
    Algorithm & $s$=10 & $s$=20 & $s$=100 & $\alpha$=0.05 & $\alpha$=0.1 & $\alpha$=0.3 \\ \midrule
    Local only ($\mathcal{L}_\text{CE}$)                              & 58.42{\tiny $\pm$ 0.22}  & 42.37{\tiny $\pm$ 0.29} & 19.02{\tiny $\pm$ 0.34} & 55.71{\tiny $\pm$ 0.19}   & 44.02{\tiny $\pm$ 0.39} & 27.98{\tiny $\pm$ 0.08} \\
    
    Local only ($\mathcal{L}_\text{CE}$+ETF)
     & 58.05{\tiny $\pm$ 0.25}  & 41.72{\tiny $\pm$ 0.26} & 19.06{\tiny $\pm$ 0.18} & 55.41{\tiny $\pm$ 0.21}   & 43.56{\tiny $\pm$ 0.31} & 27.69{\tiny $\pm$ 0.17} \\

    Local only ($\mathcal{L}_\text{DR}$)                               & 61.05{\tiny $\pm$ 0.37}  & 44.28{\tiny $\pm$ 0.21} & 21.06{\tiny $\pm$ 0.21} & 58.56{\tiny $\pm$ 0.14}   & 47.05{\tiny $\pm$ 0.13} & 31.16{\tiny $\pm$ 0.15} \\
  
    FedPer                            & 70.62{\tiny $\pm$ 0.71}  & 55.65{\tiny $\pm$ 1.35} & 25.57{\tiny $\pm$ 0.59} & 63.35{\tiny $\pm$ 1.96}   & 51.90{\tiny $\pm$ 2.13} & 35.84{\tiny $\pm$ 2.16} \\
    Per-FedAvg                           & 31.71{\tiny $\pm$ 1.08}  & 38.64{\tiny $\pm$ 0.40} & 45.71{\tiny $\pm$ 0.81} & 28.85{\tiny $\pm$ 0.27}   & 36.00{\tiny $\pm$ 0.42} & 42.41{\tiny $\pm$ 0.32} \\
    
    FedRep                            & 62.59{\tiny $\pm$ 0.30}  & 51.18{\tiny $\pm$ 1.00} & 26.51{\tiny $\pm$ 0.27} & 57.73{\tiny $\pm$ 0.41}   & 49.59{\tiny $\pm$ 0.40} & 36.22{\tiny $\pm$ 0.86} \\
    
    Ditto                            & 38.39{\tiny $\pm$ 0.54}  & 42.16{\tiny $\pm$ 1.14} & 44.04{\tiny $\pm$ 0.81} & 34.86{\tiny $\pm$ 1.18}   & 38.67{\tiny $\pm$ 1.30} & 42.05{\tiny $\pm$ 0.58} \\

    \midrule

    FedAvg-FT & 70.20{\tiny $\pm$ 0.54}  & 56.26{\tiny $\pm$ 0.51} & 48.67{\tiny $\pm$ 0.99} & 61.08{\tiny $\pm$ 1.86}   & 56.34{\tiny $\pm$ 1.18} & 49.74{\tiny $\pm$ 1.08} \\

    FedBABU-FT & 80.73{\tiny $\pm$ 0.65}  & 71.02{\tiny $\pm$ 0.34} & 51.70{\tiny $\pm$ 0.21} & 76.12{\tiny $\pm$ 0.55}   & 69.94{\tiny $\pm$ 0.34} & 57.40{\tiny $\pm$ 1.50} \\
    
    SphereFed-FT & 81.34{\tiny $\pm$ 0.64}  & 72.22{\tiny $\pm$ 0.56} & 56.58{\tiny $\pm$ 0.89} & 74.49{\tiny $\pm$ 0.86}   & 69.39{\tiny $\pm$ 1.04} & 59.51{\tiny $\pm$ 1.03} \\
    
    FedETF-FT & 53.32{\tiny $\pm$ 0.60}  & 53.05{\tiny $\pm$ 0.49} & 49.74{\tiny $\pm$ 0.85} & 52.31{\tiny $\pm$ 0.40}   & 53.70{\tiny $\pm$ 0.35} & 50.80{\tiny $\pm$ 0.65} \\
    
    FedGELA-FT & 75.75{\tiny $\pm$ 0.57}  & 68.96{\tiny $\pm$ 0.37} & 52.23{\tiny $\pm$ 0.59} & 58.26{\tiny $\pm$ 5.78}   & 60.12{\tiny $\pm$ 0.71} & 53.09{\tiny $\pm$ 0.82} \\ \midrule

    \textbf{{\alg}\,FT (ours)} & \textbf{83.08{\tiny $\pm$ 0.27}}  & \textbf{74.80{\tiny $\pm$ 0.66}} & \textbf{56.56{\tiny $\pm$ 1.04}} & \textbf{78.40{\tiny $\pm$ 0.40}}  & \textbf{73.23{\tiny $\pm$ 0.89}} & \textbf{62.22{\tiny $\pm$ 0.86}} \\ 
    \bottomrule
    \end{tabular}
    }
    \vspace*{-5mm}
\end{table}

\subsection{Personalized Federated Learning Results}
\label{ch3subsec:pfl_results}

We introduce {\alg}\,FT, inspired by prior work~\citep{oh2021fedbabu, dong2022spherefed, li2023no, kim2023fedfn, fan2023federated}, which enhances personalization by leveraging local data to fine-tune the global federated learning (GFL) model. We fine-tune the \alg GFL model using $\mathcal{L}_\text{Dr+}$ to create {\alg}\,FT, \ie 2-step approach. The overall pseudocode of {\alg}\,FT can be found in Subsection~\ref{ch3app:pseudo_code}. For a comprehensive analysis, we compare {\alg}\,FT with existing personalized federated learning (PFL) methods, including 1-step approaches, \ie creating PFL models from scratch, such as FedPer~\citep{arivazhagan2019federated}, Per-FedAvg~\citep{fallah2020personalized}, FedRep~\citep{collins2021exploiting}, and Ditto~\citep{li2021ditto}, as well as 2-step methods such as FedAVG-FT, FedBABU-FT~\citep{oh2021fedbabu}, SphereFed-FT~\citep{dong2022spherefed}, FedETF-FT~\citep{li2023no}, and FedGELA-FT~\citep{fan2023federated}. Additionally, we compare these methods with various simple local models that have not undergone federated learning: (1) Local only ($\mathcal{L}_\text{CE}$), trained with $\mathcal{L}_\text{CE}$, (2) Local only ($\mathcal{L}_\text{CE}$\,+\,ETF), trained with $\mathcal{L}_\text{CE}$ and initializing the classifier with an ETF classifier, and (3) Local only ($\mathcal{L}_\text{DR}$), trained using $\mathcal{L}_\text{DR}$.

In Table~\ref{ch3tab:pfl_acc}, we first compare the performance of simple local models in PFL by examining $\mathcal{L}_\text{DR}$ and $\mathcal{L}_\text{CE}$. While methods using $\mathcal{L}_\text{CE}$ show no significant differences, utilizing $\mathcal{L}_\text{DR}$ leads to substantial performance improvements in PFL across all settings. The ``Local only ($\mathcal{L}_\text{CE}$)'' and ``Local only ($\mathcal{L}_\text{CE}$\,+\,ETF)'' methods exhibit similar performance due to the nearly classwise orthogonal nature of randomly initialized classifiers~\citep{oh2021fedbabu, saxe2013exact, glorot2010understanding, he2015delving, lezama2018ole}. With a large number of classes ($C$=100), the ETF classifier, which is also nearly classwise orthogonal, performs similarly to random initialization. When comparing {\alg}\,FT with other 2-step methods, {\alg}\,FT consistently demonstrates superior performance. This aligns with previous research~\citep{nguyen2022begin,chen2022importance} suggesting that fine-tuning from a well-initialized model yields better PFL performance. Additionally, compared with 1-step algorithms, {\alg}\,FT continues to show superiority, outperforming all baseline methods across all settings.

\newpage
\section{Supplemental Evidence}
\label{ch3sec:suppl}
\subsection{Notations}
\label{ch3app:notations}
In this subsection, we first introduce all key notations used throughout this paper.

\begin{table}[htp]
\centering

\caption{Notations used throughout the paper.}
\label{ch3tab:notation_summary}
\resizebox{0.9\textwidth}{!}{
\begin{tabular}{ll}
\toprule
\textbf{Indices} & \\ 
$c \in [C]$  & Index for a class \\
$r \in [R]$ & Index for FL round \\
$i \in [N]$ & Index for a client \\ 
\midrule

\textbf{Dataset} & \\
$D_{\text{train}}^{i}$ & Training dataset for client $i$ \\
$D_{\text{test}}^{i}$ & Test dataset for client $i$ \\
$(x, y) \in D_{\text{train,test}}^{i}\,; (x, y) \sim \mc{D}^{i}$ & Data on client $i$ sampled from distribution $\mc{D}^{i}$ \\
 & ($x$: input data, $y$: class label) \\
$\mc{O}^i$ & Dataset consists of observed classes in client $i$ \\
$\mc{U}^i$ & Dataset consists of unobserved classes in client $i$ \\

\midrule
\textbf{Parameters} & \\ 
$\bm{\theta}$ & Feature extractor weight parameters \\
$\bm{V} = [v_1, \ldots, v_C] \in \mathbb{R}^{C \times d}$ & Classifier weight parameters (frozen during training) \\
$v_c, c \in [C]$ & $c$-th row vector of $\bm{V}$ \\
$\bm{\Theta} = (\bm{\theta}, \bm{V})$ & All model parameters \\
$\bm{\Theta}_{r}^{g} = (\bm{\theta}_{r}^{g}, \bm{V})$ & Aggregated global model parameters at round $r$ \\
$\bm{\Theta}_{r}^{i}= (\bm{\theta}_{r}^{i}, \bm{V})$ & Trained model parameters on client $i$ at round $r$ \\

\midrule
\textbf{Model Forward} & \\
$p(x; \bm{\theta}) \in \mathbb{R}^{C}$ & Softmax probability of input $x$ \\ 
$p_c(x; \bm{\theta}), c \in [C]$ & $c$-th element of $p(x; \bm{\theta})$ \\ 
$\mathcal{L}_{\text{CE}}(x; \theta) = -\log p_{y}(x; \bm{\theta})$ & Cross-entropy loss of input $x$ \\
$f(x; \bm{\theta}) \in \mathbb{R}^{d}$ & Feature vector of input $x$ \\ 
$z(x; \bm{\theta}) = f(x; \bm{\theta}) \bm{V}^\top \in \mathbb{R}^{C}$ & Logit vector of input $x$ \\ 
$z_c(x; \bm{\theta}), c \in [C]$  & $c$-th element of $z(x; \bm{\theta})$ \\
\bottomrule
\end{tabular}
}
\end{table}

\newpage
\subsection{Pseudo Code of \alg and \textbf{{\alg}\,FT}} \label{ch3app:pseudo_code}

We now present the pseudocode for \alg and \textbf{{\alg}\,FT}, outlining their key operations for global and personalized federated learning. The algorithm consists of two main stages:

\begin{figure}[H] 
\centering
\resizebox{\textwidth}{!}{ 
\begin{minipage}{\textwidth} 
\begin{algorithm}[H]
\caption{\alg, \textbf{{\alg}\,FT}}
\label{ch3alg:feddr+}
\KwIn{Total rounds $R$, local epochs $E$, training dataset $D_{\text{train}}^{i}$ for client $i$, sampled client set $N^{(r)} \subset [N]$ at round $r$, learning rate $\eta^{(r)}$ at round $r$}

\textbf{Initial Parameters:}  ETF Classifier $\mathbf{V}$, 
Initial global model parameters $\boldsymbol{\Theta}_{0}^{g} = (\boldsymbol{\theta}_{0}^{g}, \mathbf{V})$\\

\For{$i = 1, \dots, N$}{
Server broadcasts $\mathbf{V}$ to client $i$}
\quad \textbf{\code{/** STEP 1: Get a GFL Model $\boldsymbol{\Theta}_{R}^{g}$ of \alg **/}}\\ 
\For{$r = 1,\dots, R$}{
    Server samples clients $N^{(r)}$ and broadcasts $\boldsymbol{\theta}_{r}^{i} \leftarrow \boldsymbol{\theta}_{r-1}^{g}$\;
    \For{each client $i \in N^{(r)}$ \textbf{in parallel}}{
        \For{Local Steps $e = 1, \dots, E$}{
            \For{Batches $j = 1, \dots, B$}{
                $\boldsymbol{\theta}_{r}^{i} \leftarrow \boldsymbol{\theta}_{r}^{i} - \eta^{(r)} \nabla \mc{L}_{\text{Dr+}}([D_{\text{train}}^{i}]_j; \boldsymbol{\theta}_{r}^{i}, \boldsymbol{\theta}_{r-1}^{g}, \bm{V})$ \hfill \textit{Using [Equation \eqref{ch3alg:feddr+}]}
            }
        }
        Upload $\boldsymbol{\theta}_{r}^{i}$ to server\;
    }
    \textbf{Server Aggregation:} 
    $\boldsymbol{\theta}_{r}^{g} \leftarrow \frac{1}{|N^{(r)}|} \sum_{i \in N^{(r)}} \boldsymbol{\theta}_{r}^{i}$\;
}
\textbf{GFL output:} $\boldsymbol{\Theta}_{R}^{g}=(\boldsymbol{\theta}_{r}^{g}, \boldsymbol{V})$\\

\quad \textbf{\code{/** STEP 2: Get a PFL Models $\{\boldsymbol{\Theta}_{R+1}^{i}\}_{i=1}^{N}$ of \textbf{{\alg}\,FT} **/}}\\ 
\For{$i = 1, \dots, N$}{
Server broadcasts $\boldsymbol{\theta}_{R+1}^{i} \leftarrow \boldsymbol{\theta}_{R}^{g}$ to client $i$\\
        \For{Local Steps $e = 1, \dots, E$}{
            \For{Batches $j = 1, \dots, B$}{
                $\boldsymbol{\theta}_{R+1}^{i} \leftarrow \boldsymbol{\theta}_{R+1}^{i} - \eta^{(R)} \nabla \mc{L}_{\text{Dr+}}([D_{\text{train}}^{i}]_j; \boldsymbol{\theta}_{R}^{i}, \boldsymbol{\theta}_{R}^{g}, \bm{V})$ \hfill \textit{Using [Equation \eqref{ch3alg:feddr+}]}
            }}
}
\textbf{PFL outputs:} $\{\boldsymbol{\Theta}_{R+1}^{i}=(\boldsymbol{\theta}_{R+1}^{i}, \boldsymbol{V})\}_{i=1}^{N}$\\

\end{algorithm}
\end{minipage}
}
\end{figure}

\newpage
\subsection{Preliminaries: Pulling and Pushing Feature Gradients in CE}
\label{ch3app:prelim}

In this subsection, we first compute the classifier's gradient with respect to the features. Next, we explain how the cross-entropy loss draws the pulling and pushing effects.

\subsubsection{Feature Gradient of $\mathcal{L}_{\text{CE}}$} \label{ch3appsubsec:prop}

We begin by presenting two lemmas that support Proposition~\ref{ch3appprop:pull_push} and clarify pulling and pushing feature gradients in the cross-entropy (CE) loss.

\begin{lem}\label{ch3applem:derivative_logit}    
    For all $c,c'\in[C], \displaystyle \frac{\partial p_{c'}(x;\bm{\theta})}{\partial z_c(x;\bm{\theta})}=
    \begin{cases}
        p_c(x;\bm{\theta}) \cdot (1-p_c(x;\bm{\theta})) &\text{if}\:c=c'\\
        -p_c(x;\bm{\theta}) \cdot p_{c'}(x;\bm{\theta}) &\text{otherwise}   
    \end{cases}.$
\end{lem}

\begin{proof}
    Note that $\displaystyle p(x;\bm{\theta})=\bigg[\frac{\exp(z_j(x;\bm{\theta}))}{\sum_{i=1}^{C}\exp(z_i(x;\bm{\theta}))}\bigg]_{j=1}^{C}\in\mathbb{R}^{C}$.
    Then,

    \begin{enumerate}[label=(\roman*)]
        \item $c=c'$ case:
            \begin{align*}
                \frac{\partial p_c(x;\bm{\theta})}{\partial z_c(x;\bm{\theta})}&=\frac{\partial }{\partial z_c(x;\bm{\theta})}\left\{\frac{\exp(z_c(x;\bm{\theta}))}{\sum_{i=1}^{C}\exp(z_i(x;\bm{\theta}))}\right\}=\frac{\exp(z_c(x;\bm{\theta}))\left(\sum_{i=1}^{C}\exp(z_i(x;\bm{\theta}))\right)-{\exp(z_c(x;\bm{\theta}))}^2}{\left(\sum_{i=1}^{C}\exp(z_i(x;\bm{\theta}))\right)^2}\\
                              &=p_c(x;\bm{\theta})-{p_c(x;\bm{\theta})}^2 =p_c(x;\bm{\theta})(1-p_c(x;\bm{\theta})).
                \end{align*}
        \item $c\neq c'$ case:
            \begin{align*}
                \frac{\partial p_{c'}(x;\bm{\theta})}{\partial z_c(x;\bm{\theta})}&=\frac{\partial }{\partial z_c(x;\bm{\theta})}\left\{\frac{\exp(z_{c'}(x;\bm{\theta}))}{\sum_{i=1}^{C}\exp(z_i(x;\bm{\theta}))}\right\}=\frac{-\exp(z_{c}(x;\bm{\theta}))\exp(z_{c'}(x;\bm{\theta}))}{\left(\sum_{i=1}^{C}\exp(z_i(x;\bm{\theta}))\right)^2}\\
                              &=-p_c(x;\bm{\theta})p_{c'}(x;\bm{\theta}).
            \end{align*}   
    \end{enumerate} 
\end{proof}

\begin{lem}\label{ch3applem:gradient_logit}
 $\nabla_{z(x;\bm{\theta})}\mathcal{L}_\text{CE}(x,y;\bm{\theta})=p(x;\bm{\theta})-\mathbf{e}_{y}$, where $\mathbf{e}_{y}\in\mathbb{R}^{C}$ is the unit vector with its $y$-th element as 1. 
\end{lem}

\begin{proof}
    \begin{align*}
        \frac{\partial\mathcal{L}_\text{CE}(x,y;\bm{\theta})}{\partial z_c(x;\bm{\theta})}&=- \frac{\partial}{\partial z_c(x;\bm{\theta})}\log p_y(x;\bm{\theta})=-\frac{1}{p_y(x;\bm{\theta})}\frac{\partial p_y(x;\bm{\theta})}{\partial z_c(x;\bm{\theta})}\notag\\
        &=\begin{cases}
            p_c(x;\bm{\theta})-1 &\text{if }c=y\\
            p_c(x;\bm{\theta}) &\text{else }   
            \end{cases} =p_c(x;\bm{\theta})-\mathds{1}\{c=y\}. 
    \end{align*}
    The last equality holds by Lemma~\ref{ch3applem:derivative_logit}.
    Therefore, the desired result is satisfied.
\end{proof}

\begin{prop}\label{ch3appprop:pull_push}
    Given $(x,y)$, the gradient of the $\mathcal{L}_\text{CE}$ with respect to $f(x;\bm{\theta})$ is given by:
    \begin{equation}
    \nabla_{f(x;\bm{\theta})}\mathcal{L}_\text{CE}(x,y;\bm{\theta})= - (1-p_{y}(x;\bm{\theta}))v_{y} + \sum_{c\in [C]\setminus \{y\}} p_{c}(x;\bm{\theta})v_{c} \\ 
    \end{equation}
\end{prop}

\begin{proof}
\begin{align*}   
        \nabla_{f(x;\bm{\theta})}\mathcal{L}_\text{CE}(x,y;\bm{\theta})& 
        {=}\Big[\nabla_{f(x;\bm{\theta})}z_{1}(x;\bm{\theta})\,\big|\,\cdots\,\big|\,\nabla_{f(x;\bm{\theta})}z_{C}(x;\bm{\theta})\Big]\nabla_{z(x;\bm{\theta})}\mathcal{L}_\text{CE}(x,y;\bm{\theta})\\
            &=\sum_{c=1}^{C}\frac{\partial \mathcal{L}_\text{CE}(x,y;\bm{\theta})}{\partial z_{c}(x;\bm{\theta})}\nabla_{f(x;\bm{\theta})}z_{c}(x;\bm{\theta})\\
            &=\frac{\partial \mathcal{L}_\text{CE}(x,y;\bm{\theta})}{\partial z_{y}(x;\bm{\theta})}\nabla_{f(x;\bm{\theta})}z_{y}(x;\bm{\theta})+\sum_{c\in [C]\setminus \{y\}}\frac{\partial \mathcal{L}_\text{CE}(x,y;\bm{\theta})}{\partial z_{c}(x;\bm{\theta})}\nabla_{f(x;\bm{\theta})}z_{c}(x;\bm{\theta})\\
            &=\frac{\partial \mathcal{L}_\text{CE}(x,y;\bm{\theta})}{\partial z_{y}(x;\bm{\theta})}v_{y}+\sum_{c\in [C]\setminus \{y\}}\frac{\partial \mathcal{L}_\text{CE}(x,y;\bm{\theta})}{\partial z_{c}(x;\bm{\theta})}v_{c}\\
            &= - (1-p_{y}(x;\bm{\theta}))v_{y} + \sum_{c\in [C]\setminus \{y\}} p_{c}(x;\bm{\theta})v_{c}.\\   
\end{align*}
Applying the chain rule for the second step and invoking Lemma~\ref{ch3applem:gradient_logit} for the final equality confirms the result.

\subsubsection{Physical Meaning of $\nabla_{f(x;{\theta})}\mathcal{L}_\text{CE}(x,y;{\theta})$}


The gradient $ \nabla_{f(x;\bm{\theta})}\mathcal{L}_\text{CE}(x,y;\bm{\theta}) $ consists of two components:  
\begin{align*}
    \mathbf{F}_\text{Pull} &= \big(1 - p_{y}(x;\bm{\theta})\big) v_{y}, \\
    \mathbf{F}_\text{Push} &= -\kern-1em\sum_{c \in [C] \setminus \{y\}} p_{c}(x;\bm{\theta}) v_{c}.
\end{align*}
$\mathbf{F}_\text{Pull}$ moves the feature vector towards the classifier vector $ v_y $ of the true class, promoting alignment.  
In contrast, $\mathbf{F}_\text{Push}$ moves it away from the classifier vectors $ v_c $ for $ c \in [C] \setminus \{y\} $, inducing misalignment.

\end{proof}

\newpage

\subsection{Theoretical Perspective of Dot-Regression (DR)}\label{ch3app:theory_dr}

In this subsection, we provide a theoretical analysis of dot-regression (DR) loss in the context of feature-classifier alignment. We first derive the feature gradient of $\mathcal{L}_{\mathrm{DR}}$ and analyze its effect on feature updates. We then present an NTK-based perspective explaining why dot-regression struggles with unobserved classes in FL. Finally, we compare DR with cross-entropy (CE) loss to highlight its limitations and the necessity of feature distillation.

\subsubsection{Feature Gradient of $\mathcal{L}_{\mathrm{DR}}$}

We derive the gradient of dot-regression loss with respect to the feature vector on the observed classes. 

\begin{thm}\label{ch3appthm:feature_gradient_dr}
    Given $(x,y)$, the gradient of the $\mathcal{L}_{\mathrm{DR}}$ with respect to $f(x;\bm{\theta}_{f})$ is given by:
    \[
    \nabla_{f(x;\bm{\theta}_{f})}\mathcal{L}_{\mathrm{DR}}(x,y;\theta)=-\frac{1-\cos\alpha}{{\|f(x;\bm{\theta}_{f})\|}_{2}}\left\{V_{y}-\cos{\alpha}\,\frac{f(x;\bm{\theta}_{f})}{{\|f(x;\bm{\theta}_{f})\|}_{2}}\right\}\,,
    \]
where $\displaystyle \cos{\alpha} = \frac{f(x;\bm{\theta}_{f})^\top}{{\|f(x;\bm{\theta}_{f})\|}_{2}} V_y$.
\end{thm}

\begin{proof}
    {
    \begin{align*}   
        \nabla_{f(x;\bm{\theta}_{f})}\mathcal{L}_{\mathrm{DR}}(x,y;\bm{\theta}) &=\nabla_{f(x;\bm{\theta}_{f})}\left\{\frac{1}{2} \left(\frac{f(x;\bm{\theta}_{f})^{T}}{{\|f(x;\bm{\theta}_{f})\|}_{2}} V_{y}-1\right)^{2} \right\}\\            
            &=\left(\frac{f(x;\bm{\theta}_{f})^{T}}{{\|f(x;\bm{\theta}_{f})\|}_{2}} V_{y}-1\right)\nabla_{f(x;\bm{\theta}_{f})}\frac{f(x;\bm{\theta}_{f})^{T}}{{\|f(x;\bm{\theta}_{f})\|}_{2}} V_{y}\\
            &=\left(\frac{f(x;\bm{\theta}_{f})^{T}}{{\|f(x;\bm{\theta}_{f})\|}_{2}} V_{y}-1\right)\left[\frac{1}{{\|f(x;\bm{\theta}_{f})\|}_{2}}\left\{I-\frac{f(x;\bm{\theta}_{f}){f(x;\bm{\theta}_{f})}^{T}}{{\|f(x;\bm{\theta}_{f})\|}_{2}^{2}} \right\}V_{y}\right]\\
            &= -\frac{1-\cos\alpha}{{\|f(x;\bm{\theta}_{f})\|}_{2}}\left\{V_{y}-\cos{\alpha}\frac{f(x;\bm{\theta}_{f})}{{\|f(x;\bm{\theta}_{f})\|}_{2}}\right\}.\\                
        \end{align*}
    }
\end{proof}

\subsubsection{Physical Meaning of $\nabla_{f(x;{\theta})}\mathcal{L}_\text{DR}(x,y;{\theta})$} 

\begin{figure}[h!]
    \centering
    \includegraphics[width=0.75\linewidth]{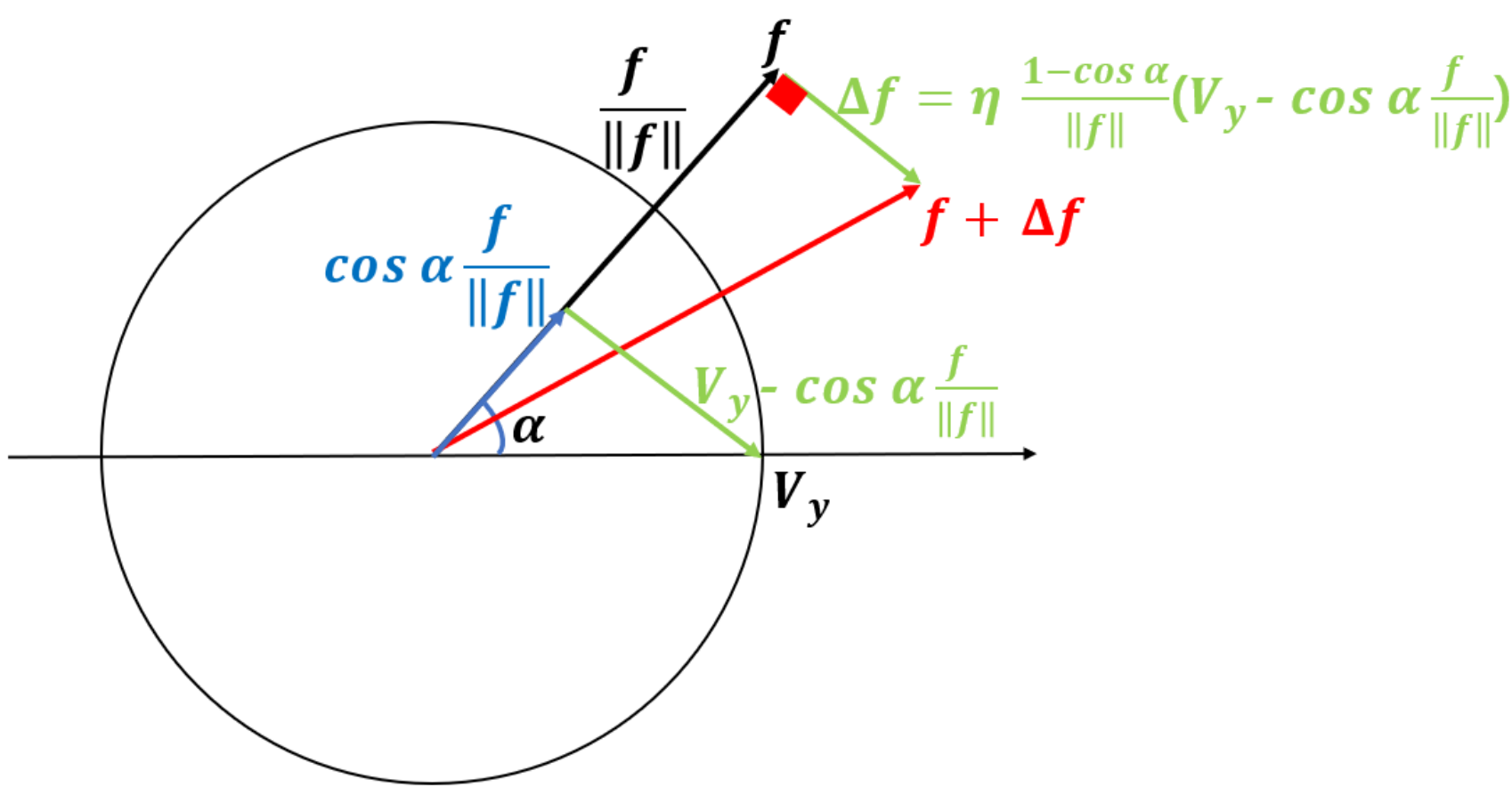}
    \caption{Feature gradient of $\mathcal{L}_{\text{DR}}$. The gradient update rotates $f(x;\bm{\theta}_f)$ toward $V_y$ while increasing its norm. As training progresses, the update magnitude decreases, leading to convergence.}
    \label{ch3appfig:feature_gradient_dr}
\end{figure}





According to Theorem~\ref{ch3appthm:feature_gradient_dr}, the change in the feature vector $ \Delta f(x;\bm{\theta}_{f}) $ is given by:
\begin{equation*}
    \Delta f(x;\theta_f) = \eta\,\frac{1-\cos\alpha}{\|f(x;\theta_f)\|_2} \left( V_y - \cos\alpha \frac{f(x;\theta_f)}{\|f(x;\theta_f)\|_2} \right),
\end{equation*}
where $ \eta $ is the learning rate and $ \alpha $ is the angle between the feature vector $ f(x;\bm{\theta}_{f}) $ and the target vector $ V_y $.  

The term inside the parentheses, $ V_y - \cos{\alpha} \frac{f(x;\bm{\theta}_{f})}{\|f(x;\bm{\theta}_{f})\|_2} $, represents a component orthogonal to $ f(x;\bm{\theta}_{f}) $ that points towards $ V_y $. This component adjusts $ f(x;\bm{\theta}_f) $ to increase its cosine similarity with $ V_y $ while also expanding its norm.

The scaling factor $ \frac{1-\cos\alpha}{\|f(x;\bm{\theta}_f)\|_2} $ determines the update magnitude. As training progresses, $ f(x;\bm{\theta}_f) $ aligns more closely with $ V_y $, reducing $ 1-\cos\alpha $ and increasing $ \|f(x;\bm{\theta}_f)\|_2 $. Consequently, $ \Delta f(x;\bm{\theta}_f) $ diminishes over time, reflecting convergence as the cosine similarity with $ V_y $ approaches its maximum.

Figure~\ref{ch3appfig:feature_gradient_dr} illustrates this process, showing how the orthogonal component drives both the rotation and scaling of $ f(x;\bm{\theta}_f) $ toward alignment with $ V_y $.

\subsection{NTK Perspective: Why Dot-Regression in FL Fails on Unobserved Classes} \label{ch3app:ntk}

In this subsection, we analyze why dot-regression struggles with unobserved classes under the Neural Tangent Kernel (NTK) regime~\citep{NEURIPS2018_5a4be1fa}. In the NTK regime, the feature gradient of any input is a weighted sum of the feature gradients from training samples. Assuming the network width is sufficiently wide, these weights depend only on the pair of inputs, the initialization distribution, such as He initialization~\citep{pmlr-v9-glorot10a}, and the activation functions.\footnote{In practice, finite-width effects cause deviations from the ideal NTK behavior.} The NTK regime holds when the setting where every layer in the neural network has infinite width, with parameters initialized i.i.d. This section explains how NTK-based gradient updates fail to align feature vectors with unobserved class directions, which leads to poor generalization in FL.

\subsubsection{Gradient Flow in the NTK Regime}

We treat gradient descent as a continuous process. $P$ is the number of trainable parameters in the feature extractor, and $\theta_p$ ($p \in [P]$) denote each parameter. We focus on a specific client, denoted by $i$.

During training, gradient descent updates the model parameters to minimize the loss function. As below, we can see the evolution of the function $f(x; \bm{\theta}(t))$ can be analyzed using the kernel $\Theta^{(L)}(t)(x, x_i)$, which evolves along the training process:

\begin{align*}
\frac{\mathrm{d} f(x;\bm{\theta}(t))}{\mathrm{d} t} &= \sum_{p=1}^P \Big(\frac{\partial f(x;\bm{\theta}(t))}{ \partial\theta_p} \Big)^\top \frac{\mathrm{d} \theta_p}{\mathrm{d} t}
\tag{Chain Rule}\\&=- \sum_{p=1}^P  \Big(\frac{\partial f(x;\bm{\theta}(t))}{ \partial\theta_p} \Big)^\top \frac{1}{|D_{\text{train}}^{i}|} \kern-0.2em\sum_{(\tilde{x},\tilde{y})\in D_{\text{train}}^{i}} \kern-1em\frac{\partial f(\tilde{x};\bm{\theta}(t))}{ \partial\theta_p} \nabla_{f(x_i;\theta)} \mathcal{L}(\tilde{x},\tilde{y};\bm{\theta}) \tag{Gradient Descent}\\
&= -\frac{1}{|D_{\text{train}}^{i}|}\kern-0.2em\sum_{(\tilde{x},\tilde{y})\in D_{\text{train}}^{i}} \kern-0.5em \Big( {\underbrace{\sum_{p=1}^P \Big( \frac{\partial f(x;\bm{\theta}(t))}{ \partial\theta_p} \Big)^\top \frac{\partial f(\tilde{x};\bm{\theta}(t))}{ \partial\theta_p}}_{\bm{\Theta}^{(L)}(t)(x,\tilde{x}) \in \mathbb{R}^{d \times d}} }\Big) \nabla_{f(\tilde{x};\theta)} \mathcal{L}(\tilde{x},\tilde{y};\bm{\theta}) \\
&= -\frac{1}{|D_{\text{train}}^{i}|} \kern-0.2em\sum_{(\tilde{x},\tilde{y})\in D_{\text{train}}^{i}} \kern-1em \bm{\Theta}^{(L)}(t)(x,\tilde{x}) \nabla_{f(\tilde{x};\theta)} \mathcal{L}(\tilde{x},\tilde{y};\bm{\theta})\,.
\end{align*} 

In the NTK regime with infinitely large widths, the matrix $\bm{\Theta}^{(L)}(t)(x,\tilde{x})$ converges to a scalar multiple of the identity matrix, $\Theta^{(L)}_\infty(x,\tilde{x}) \mathbf{I}$. Furthermore, with the same condition, this scalar kernel remains constant throughout training~\citep{NEURIPS2018_5a4be1fa, yang2020scalinglimitswideneural,belfer2021spectralanalysisneuraltangent}, though finite-width effects may introduce small variations. In the NTK regime, the gradient descent dynamics are given by:

\begin{equation}
\label{ch3eq:ntkgd}
\frac{\mathrm{d} f(x;\bm{\theta}(t))}{\mathrm{d} t} =   -\frac{1}{|D_{\text{train}}^{i}|} \kern-0.2em\sum_{(\tilde{x},\tilde{y})\in D_{\text{train}}^{i}} \kern-1em  \underbrace{\Theta^{(L)}_\infty (x,\tilde{x})}_{\,\in\, \mathbb{R}} 
 \nabla_{f(x_i;\theta)} \mathcal{L}(\tilde{x},\tilde{y};\bm{\theta})\,. \tag{in NTK Regime}
\end{equation}

Thus, $ \Theta^{(L)}_\infty (x,\tilde{x})$ determines how each training sample $\tilde{x}$ influences an arbitrary input $x$, and in NTK regime, this weight depends only on the initialization distribution.

In Federated Learning (FL), local models are independently updated on different clients before aggregation. Under the NTK regime, each client follows the gradient flow during local training. FL aggregation then combines feature representations learned from different data distributions, leading to shifts in the global feature representation. By aggregating updates from multiple clients, FL integrates feature information from clients that have observed missing classes, thereby improving feature alignment.

\subsubsection{Limitations of Dot-Regression loss in FL under the NTK Regime}

Dot-regression loss~\citep{yang2022inducing} speeds up the alignment of feature vectors
$f(x; \bm{\theta}) \in \mathbb{R}^{d}$ (pre-classifier layer outputs) with the true class direction $v_y$ by minimizing the cosine angle:
\[
    \mc{L}_{\text{DR}} (x, y; \bm{\theta}, \bm{V}) = \frac{1}{2}\Big(\cos\big(f(x; \bm{\theta}), v_y\big) - 1\Big)^2\,.
\]

This loss function is motivated by the decomposition of cross-entropy (CE) loss gradients into \emph{pulling} and \emph{pushing} components. Prior work suggests that removing the pushing effect in CE can improve convergence~\citep{yang2022inducing, li2021fedrs}.

Let $c$ be an unobserved class for a specific client $i$ with the classifier vector $v_c$. From Theorem\autoref{ch3appthm:feature_gradient_dr}, it follows that under the NTK regime, the gradient descent process on the client $i$ is independent of $v_c$ for arbitrary input $x$.

To analyze this, we first express the feature gradient under the dot-regression loss $\mathcal{L}_{\mathrm{DR}}$ in the local learning stage. For simplicity, we omit the dependence on $\bm{\theta}(t)$ in the feature notation and write $\cos(f(\tilde{x};\bm{\theta}(t)), v_y)$ as $\cos(f(\tilde{x}), v_y)$. The feature gradient is given by:

\[
\frac{\mathrm{d} f(x)}{\mathrm{d} t} =   \frac{1}{|D_{\text{train}}^{i}|} \sum_{y\in\mathcal{O}^i} \kern-0.6em\sum_{\kern0.6em(\tilde{x},y)\in D_{\text{train}}^{i}} \kern-1.5em \Theta^{(L)}_\infty (x,\tilde{x}) \frac{1-\cos(f(\tilde{x}),\!v_y)}{\| f(\tilde{x})\|_2} \Big( v_y - \cos(f(\tilde{x}),\!v_y) \frac{f(\tilde{x})}{\| f(\tilde{x})\|_2} \Big). \tag{in NTK Regime}
\]

Since $c \notin \mathcal{O}^i$, the feature gradient evaluated on training data does not depend on $v_c$. Given that feature gradients are a weighted sum over training data in the NTK regime, this implies that the learned feature representation for an arbitrary input remains unaffected by $v_c$ during local training.

Therefore, dot-regression cannot align features with unobserved classes in local training. To examine this effect more closely, consider two cases $f_1(x)$ and $f_2(x)$ with the same input $x$ with label $c$, whose settings and initialization at time $t=0$ are identical except for the classifier vector $v_c$ of class $c$, fixed with $w$ and $-w$ ($\|w\|=1$, $\forall y\in\mathcal{O}^i: w \perp v_y$). In the NTK regime under the dot-regression loss, we have:

\[
\frac{\mathrm{d}}{\mathrm{d} t} \langle f(x), v_c \rangle  =   -\frac{1}{|D_{\text{train}}^{i}|} \sum_{y\in\mathcal{O}^i} \kern-0.6em\sum_{\kern0.6em(\tilde{x},y)\in D_{\text{train}}^{i}} \kern-1.5em \Theta^{(L)}_\infty (x,\tilde{x}) \frac{\cos(f(\tilde{x}),\!v_y)(1-\cos(f(\tilde{x}),\!v_y))}{\| f(\tilde{x})\|_2^2} \, \langle f(\tilde{x}), v_c \rangle. \tag{in NTK Regime}
\]

Since every term in the update equation is identical for $f_1(x)$ and $f_2(x)$, except for $\langle f(\tilde{x}), v_c \rangle$, which takes opposite values in each case, it follows that $ \langle f_1(x), v_c \rangle = - \langle f_2(x), v_c \rangle $ for all time $t \geq 0$. This demonstrates that classifier initialization strongly determines alignment in the local learning stage. Consequently, the global aggregation stage is the only way to generalize to classes that haven't been observed yet. This slows down the overall accuracy of the FL server. 

\subsubsection{Cross-Entropy Loss and Feature-Classifier Alignment}

In contrast, cross-entropy (CE) loss explicitly guides feature gradients toward $v_c$, weighted by the softmax probability $p_c$ and the NTK weight. This ensures that even when class $c$ is absent, local training still produces meaningful updates. After each global aggregation, the refined $p_c$ further strengthens alignment, allowing CE to maintain consistent feature-classifier alignment across all classes.

This observation aligns with our empirical findings: without feature distillation, dot-regression struggles to generalize to unobserved classes, whereas CE enables continuous feature updates, leading to improved generalization.


\newpage
\subsection{Experimental Setup}\label{ch3appsec:exp_setup}

This subsection details the code implementation, dataset descriptions, model specifications, optimizer settings, and non-IID (NIID) partitioning used in our experiments.

\subsubsection{Code Implementation}

Our implementations are conducted using the PyTorch framework. Specifically, the experiments presented in Table~\ref{ch3tab:gfl_acc_all} and Table~\ref{ch3apptab:elapsed_time} are executed on a single NVIDIA RTX 3090 GPU, based on the code structure from the following repository: \url{https://github.com/Lee-Gihun/FedNTD}. The other parts of our study are carried out on a single NVIDIA A5000 GPU, utilizing the code framework from \url{https://github.com/jhoon-oh/FedBABU}.

\subsubsection{Datasets, Model, and Optimizer}

To simulate a realistic FL scenario, we conduct extensive studies on three widely used datasets: CIFAR-10~\citep{krizhevsky2009cifar}, CIFAR-100~\citep{krizhevsky2009cifar} and ImageNet-100~\cite{deng2009imagenet}. For each dataset, appropriate models are employed: VGG11~\citep{simonyan2014very} for CIFAR-10, MobileNet~\citep{howard2017mobilenets} for CIFAR-100, and ResNet-18~\citep{he2016identity} for ImageNet-100. A momentum optimizer is utilized for all experiments. The data preprocessing pipeline for the training phase includes \texttt{RandomResizedCrop}, \texttt{RandomHorizontalFlip}, and \texttt{Normalize} transformations for all datasets. During testing, only the \texttt{Normalize} transformation is applied for CIFAR-10 and CIFAR-100, while for ImageNet-100, \texttt{Resize}, \texttt{CenterCrop}, and \texttt{Normalize} are applied. Unless otherwise noted, the basic setting of our experiments follows the dataset statistics, FL scenario specifications, and optimizer hyperparameters summarized in Table~\ref{ch3apptab:dataset_optimizer_details}.

\begin{table}[h!]
    \centering
    \small
    \caption{Summary of Dataset, Model, FL System, and Optimizer Specifications}
    \label{ch3apptab:dataset_optimizer_details}
    \vspace{5pt}
    \begin{tabular}{l|cccccccccc}
        \toprule
        {Datasets} & $C$ & $|D_\text{train}|$  & $|D_\text{test}|$ & $N$ & $R$ & $r$ & $E$ & $B$ & $m$ & $\lambda$ \\
        \midrule
        CIFAR-10 & 10 & 50000 & 10000 & 100 & 320 & 0.1 & 10 & 50 & 0.9 & 1e-5 \\
        CIFAR-100 & 100 & 50000 & 10000 & 100 & 320 & 0.1 & 3 & 50 & 0.9 & 1e-5 \\
        ImageNet-100 & 100 & 130000 & 5000 & 100 & 320 & 0.1 & 5 & 50 & 0.9 & 1e-5 \\
        \bottomrule
    \end{tabular}
\end{table}

Note: In terms of dataset information, $C$ represents the number of classes in the dataset, with $|D_{\text{train}}|$ and $|D_{\text{test}}|$ indicating the total numbers of training and test data used, respectively. For the federated learning (FL) system specifics, $R$ indicates the total number of FL rounds, $r$ is the ratio of clients selected for each round, and $E$ denotes the number of local epochs. Local model training utilizes a momentum optimizer where $B$ is the batch size, and $m$ and $\lambda$ represent the momentum and weight decay parameters, respectively. The initial learning rate $\eta$ is decayed by a factor of 0.1 at the 160th and 240th communication rounds. The initial learning rate $\eta$ and the number of local epochs $E$ were determined via extensive grid search for each algorithm, with details outlined in Subsection~\ref{ch3app:grid_search}.

\subsubsection{Non-IID Partition Strategies}\label{ch3appsubsec:niid_strategy}

To induce heterogeneity in each client's training and test data ($D_\text{train}^i, D_\text{test}^i$), we distribute the entire class-balanced datasets, $D_\text{train}$ and $D_\text{test}$, among 100 clients using both sharding and Latent Dirichlet Allocation (LDA) partitioning strategies:

\begin{itemize}[leftmargin=10pt]
\item \textbf{Sharding}~\citep{mcmahan2017communication, oh2021fedbabu}: We organize the $D_\text{train}$ and $D_\text{test}$ by label and divide them into non-overlapping shards of equal size. Each shard encompasses $\frac{|D_\text{train}|}{100\times s}$ and $\frac{|D_\text{test}|}{100\times s}$ samples of the same class, where $s$ denotes the number of shards per client. This sharding technique is used to create $D_\text{train}^i$ and $D_\text{test}^i$, which are then distributed to each client $i$, ensuring that each client has the same number of training and test samples. The data for each client is disjoint. As a result, each client has access to a maximum of $s$ different classes. Decreasing the number of shards per user $s$ increases the level of data heterogeneity among clients.

\item \textbf{Latent Dirichlet Allocation (LDA)}~\citep{luo2021no, wang2020federated}: We utilize the LDA technique to create $D_\text{train}^i$ from $D_\text{train}$. This involves sampling a probability vector $p_c = (p_{c,1}, p_{c,2}, \cdots, p_{c,100}) \sim \text{Dir}(\alpha)$ and allocating a proportion $p_{c,k}$ of instances of class $c \in [C]$ to each client $k \in [100]$. Here, $Dir(\alpha)$ represents the Dirichlet distribution with the concentration parameter $\alpha$. The parameter $\alpha$ controls the strength of data heterogeneity, with smaller values leading to stronger heterogeneity among clients. For $D_\text{test}^i$, we randomly sample from $D_\text{test}$ to match the class frequency of $D_\text{train}^i$ and distribute it to each client $i$. 
\end{itemize}

\subsection{Hyperparameter Search for $\eta$ and $E$} \label{ch3app:grid_search}

To optimize the initial learning rate ($\eta$) and the number of local epochs ($E$) for our algorithm, we conduct a grid search on the CIFAR-10, CIFAR-100, and ImageNet-100 datasets. The process and reasoning are outlined below.

\subsubsection{Rationale for Varying Initial Learning Rate ($\eta$)}

The algorithms used in our experiments differ in handling feature normalization within the loss function. Some algorithms apply feature normalization, while others do not. When features $f(x;\bm{\theta})$ are normalized, the resulting gradient is scaled by $\frac{1}{\|f(x;\bm{\theta})\|_2}$. This scaling effect necessitates a grid search across various learning rates to account for the differences in learning behavior.

\subsubsection{Rationale for Varying Local Epochs $E$}

In FL, choosing the appropriate number of local epochs is crucial. Too few epochs can lead to underfitting, while too many can cause client drift. Therefore, finding the optimal number of local epochs is essential by exploring a range of values. 

\subsubsection{Grid Search Process and Results}

Considering the above reasons, we perform grid search for $\eta$ and $E$ on CIFAR-10, CIFAR-100, and ImageNet-100 datasets. The grid search for CIFAR-10 uses a shard size of 2, while for CIFAR-100, a shard size of 10 is used. Additionally, for ImageNet-100, a shard size of 20 is used. The detailed procedures for each dataset are provided below. These optimal settings have also been confirmed to yield good performance in less heterogeneous settings.

\myparagraph{CIFAR-10.}
We examine $\eta$ values from \{0.01, 0.05, 0.1, 0.15, 0.2, 0.25, 0.3, 0.35, 0.4, 0.45, 0.5, 0.55, 0.6\}. For $E$, we consider \{1, 3, 5, 10, 15\}. A default initial learning rate of 0.01 is used unless specified otherwise. The optimal learning rates vary by algorithm, and the results are summarized in Table~\ref{ch3apptab:hyper_summary_10}. Table~\ref{ch3apptab:hyper_summary_10} also includes the additional hyperparameters used for each algorithm. The notation for these additional hyperparameters follows the conventions used throughout this paper~\citep{MLSYS2020_1f5fe839, lee2022preservation, jhunjhunwala2023fedexp, li2023no, lee2024fedsol}. The optimal number of local epochs is found to be 10 for every algorithm.

\myparagraph{CIFAR-100.}
We examine $\eta$ values from \{0.1, 0.5, 1.0, 1.5, 2.0, 2.5, 3.0, 3.5, 4.0, 4.5, 5.0, 5.5, 6.0, 6.5, 7.0\}. For $E$, we consider \{1, 3, 5, 10\}. A default initial learning rate of 0.1 is used unless specified otherwise. The optimal learning rates differ by algorithm, and the results are listed in Table~\ref{ch3apptab:hyper_summary_100}. Table~\ref{ch3apptab:hyper_summary_100} also includes the additional hyperparameters used for each algorithm. The notation for these additional hyperparameters follows the conventions used throughout this paper~\citep{MLSYS2020_1f5fe839, lee2022preservation, jhunjhunwala2023fedexp, li2023no, lee2024fedsol}. The optimal number of local epochs is found to be 3 for every algorithm.

\myparagraph{ImageNet-100.}
We examine $\eta$ values from $\{0.01, 0.1, 1.0, 10.0\}$, which are chosen to maintain a consistent logarithmic scale difference.  A default initial learning rate of 0.1 is used unless specified otherwise. The optimal learning rates differ by algorithm, and the results are listed in Table~\ref{ch3apptab:hyper_summary_imagenet}. Table~\ref{ch3apptab:hyper_summary_imagenet} also includes the additional hyperparameters used for each algorithm. The notation for these additional hyperparameters follows the conventions used throughout this paper~\citep{MLSYS2020_1f5fe839, lee2022preservation, jhunjhunwala2023fedexp, li2023no, lee2024fedsol}. The optimal number of local epochs is fixed at 5, following the setting of \citep{lee2024fedsol}.

\begin{table}[h!]
    \centering
    \small
    \caption{Hyperparameters for VGG11 training on CIFAR-10.}
    \label{ch3apptab:hyper_summary_10}
    \vspace{5pt}
    \resizebox{\textwidth}{!}{
    \begin{tabular}{c|ccccccccc|ccc}
      \toprule
      & \multicolumn{9}{c|}{Feature un-normalized algorithms} & \multicolumn{3}{c}{Feature normalized algorithms} \\ \midrule
      Hyperparameters & FedAvg & FedBABU & FedProx & SCAFFOLD & MOON & FedNTD & FedExP & FedSOL & FedGELA & FedETF & SphereFed & \alg (Ours) \\ \hline
      $\eta$ & 0.01 & 0.01 & 0.01 & 0.01 & 0.01 & 0.01 & 0.01  & 0.01 & 0.01 & 0.05 & 0.55 & 0.35 \\ 
      Additional  & None & None & $\mu$=0.001 & None & $(\mu, \tau)$=(1,0.5) & $(\beta, \tau)$=(1,3) & $\epsilon$=0.001 & $\rho$=2.0 & None & $(\beta, \tau)$=(1,1) & None & $\beta$=0.9 \\ \bottomrule
    \end{tabular}
    }
\end{table}

\begin{table}[h!]
    \centering
    \small
    \caption{Hyperparameters for MobileNet training on CIFAR-100.}
    \label{ch3apptab:hyper_summary_100}
    \vspace{5pt}
    \resizebox{\textwidth}{!}{
    \begin{tabular}{c|ccccccccc|ccc}
      \toprule
      & \multicolumn{9}{c|}{Feature un-normalized algorithms} & \multicolumn{3}{c}{Feature normalized algorithms} \\ \midrule
      Hyperparameters & FedAvg & FedBABU & FedProx & SCAFFOLD & MOON & FedNTD & FedExP & FedSOL & FedGELA & FedETF & SphereFed & \alg (Ours) \\ \hline
      $\eta$ & 0.1 & 0.1 & 0.1 & 0.1 & 0.1 & 0.1 & 0.1  & 0.1 & 0.1 & 0.5 & 6.5 & 5.0 \\ 
      Additional  & None & None & $\mu$=0.001 & None & $(\mu, \tau)$=(1,0.5) & $(\beta, \tau)$=(1,3) & $\epsilon$=0.001 & $\rho$=2.0 & None & $(\beta, \tau)$=(1,1) & None & $\beta$=0.9 \\ \bottomrule
    \end{tabular}
    }
\end{table}

\begin{table}[h!]
    \centering
    \small
    \caption{Hyperparameters for ResNet-18 training on ImageNet-100.}
    \label{ch3apptab:hyper_summary_imagenet}
    \vspace{5pt}
    \resizebox{\textwidth}{!}{
    \begin{tabular}{c|ccccccccc|ccc}
      \toprule
      & \multicolumn{9}{c|}{Feature un-normalized algorithms} & \multicolumn{3}{c}{Feature normalized algorithms} \\ \midrule
      Hyperparameters & FedAvg & FedBABU & FedProx & SCAFFOLD & MOON & FedNTD & FedExP & FedSOL & FedGELA & FedETF & SphereFed & \alg (Ours) \\ \hline
      $\eta$ & 0.1 & 0.1 & 0.1 & 0.1 & 0.1 & 0.1 & 0.1  & 0.1 & 0.1 & 1.0 & 1.0 & 1.0 \\ 
      Additional  & None & None & $\mu$=0.001 & None & $(\mu, \tau)$=(1,0.5) & $(\beta, \tau)$=(1,3) & $\epsilon$=0.001 & $\rho$=2.0 & None & $(\beta, \tau)$=(1,1) & None & $\beta$=0.9 \\ \bottomrule
    \end{tabular}
    }
\end{table}

\newpage

\newpage

\subsection{Synergy Effect Details}\label{ch3app:synergey}
 We evaluate the synergy effect of various FL algorithms by maintaining their original training loss while incorporating specific regularizers, as detailed in Equation~\ref{ch3alg:FedDR} of the main text. To manage the differing loss scales between the baseline FL algorithms and the regularizers, we systematically tune the coefficient $\beta$ across a range of values (${0.1, 0.3, 0.5, 0.7, 0.9, 0.99, 0.999, 0.9999}$). The resulting performance and optimal $\beta$ values are shown in Table~\ref{ch3apptab:synergy_effect_tunned}
 and Table~\ref{ch3apptab:synergy_effect_argmax}. However, when we set $\beta=0.9$ without addressing the issue of differing loss scales, the performance results, presented in Table~\ref{ch3apptab:synergy_effect_plain}, reveal that several synergies are significantly inferior due to this oversight.

\begin{table*}[h!]
    \centering
    \caption{Synergy of various FL algorithms and regularizers. Baseline indicates training FL models without a regularizer. FD denotes feature distillation, which is the regularizer we use in \alg.}
    \label{ch3apptab:synergy_effect_tunned}
    \vspace{5pt}
    \small
    \addtolength{\tabcolsep}{-1pt}
    \resizebox{\textwidth}{!}{
    \begin{tabular}{l|ccccccc|ccccccc}
      \toprule 
       & \multicolumn{7}{c|}{Sharding ($s=10$)} & \multicolumn{7}{c}{LDA ($\alpha=0.1$)} \\ \cmidrule{2-15}
      Algorithm & Baseline & \!\!+Prox & \!\!+MOON & \!\!+KD & \!\!+NTD & \!\!+LD & +FD & Baseline & \!\!+Prox & \!\!+MOON & \!\!+KD & \!\!+NTD & \!\!+LD & +FD\\ \midrule
      FedAvg & 37.22 &  36.87 & 37.43 & 36.25 & 37.71  & 37.17 &  37.82 & 42.52 &  43.22 & 44.79 & 44.21 & 43.39 & 43.43 & 43.76\\ \midrule      
      FedBABU  & 46.20 &  46.03 & 46.49 & 46.37 & 47.22  & 46.71 &  46.95 & 47.37 &  46.62 & 46.27 & 47.60 & 46.48 & 45.78 & 46.49\\   
   
      SphereFed & 43.90 &  41.96 & 43.13 & 44.94 & 43.47  & 43.95 &  45.21 & 46.98 &  43.77 & 46.81 & 47.76 & 47.25 & 47.01 & 49.74\\   
      
      FedETF & 32.42 &  31.87 & 34.30 & 32.76 & 32.65  & 32.25 &  32.77 & 46.27 &  45.71 & 45.98 & 46.67 & 46.16 & 45.91 & 46.47\\   
      
      FedGELA & 29.17 &  28.69 & 28.80 & 29.11 & 28.84  & 29.36 &  30.33 & 27.11 &  29.03 & 28.09 & 28.45 & 29.62 & 29.41 & 29.75\\ \midrule  
      
      Dot-Regression & 42.52 &  41.95 & 44.72 & 47.45 & 48.32  & 47.52 &  \textbf{48.69} & 42.72 &  46.35 & 50.36 & 49.47 & 50.36 & 49.28 & \textbf{50.86}\\   
        
      \bottomrule
    \end{tabular}
    }
    \vspace{-5pt}
\end{table*}

\begin{table*}[h!]
    \centering
    \caption{Optimal $\beta$ value selected through grid search to achieve the best synergy of various FL algorithms and regularizers.}
    \label{ch3apptab:synergy_effect_argmax}
    \vspace{5pt}
    \small
    \addtolength{\tabcolsep}{-1pt}
    \resizebox{\textwidth}{!}{
    \begin{tabular}{l|ccccccc|ccccccc}
      \toprule 
       & \multicolumn{7}{c|}{Sharding ($s=10$)} & \multicolumn{7}{c}{LDA ($\alpha=0.1$)} \\ \cmidrule{2-15}
      Algorithm & Baseline & \!\!+Prox & \!\!+MOON & \!\!+KD & \!\!+NTD & \!\!+LD & +FD & Baseline & \!\!+Prox & \!\!+MOON & \!\!+KD & \!\!+NTD & \!\!+LD & +FD\\ \midrule
      FedAvg & None &  0.999 & 0.5 & 0.9999 & 0.9999  & 0.999  & 0.9 & None  &  0.999 & 0.99 & 0.999 & 0.99 & 0.999 &  0.9999\\ \midrule

      FedBABU  & None &  0.9999 & 0.9 & 0.999 & 0.99  & 0.999 &  0.999 & None &  0.999 & 0.999 & 0.999 & 0.999 & 0.99 & 0.9999\\   
   
      SphereFed & None &  0.9999 & 0.9999 & 0.9999 & 0.9  & 0.99 &  0.9 & None &  0.9999 & 0.999 & 0.999 & 0.9999 & 0.999 & 0.99\\   
      
      FedETF & None &  0.999 & 0.3 & 0.5 & 0.999  & 0.9 &  0.9 & None &  0.9999 & 0.9999 & 0.5 & 0.999 & 0.99 & 0.99\\   
      
      FedGELA & None &  0.9999 & 0.7 & 0.5 & 0.7  & 0.5 &  0.7 & None &  0.99 & 0.9 & 0.5 & 0.5 & 0.5 & 0.3\\ \midrule  
      
      Dot-Regression & None &  0.9999 & 0.9 & 0.5 & 0.5  & 0.5 &  0.9 & None &  0.9999 & 0.5 & 0.5 & 0.5 & 0.5 & 0.9\\    
           
      \bottomrule
    \end{tabular}
    }
    \vspace{-5pt}
\end{table*}

\begin{table*}[h!]
    \centering
    \caption{Synergy of various FL algorithms and regularizers at $\beta=0.9$.}
    \label{ch3apptab:synergy_effect_plain}
    \vspace{5pt}
    \small
    \addtolength{\tabcolsep}{-1pt}
    \resizebox{\textwidth}{!}{
    \begin{tabular}{l|ccccccc|ccccccc}
      \toprule 
       & \multicolumn{7}{c|}{Sharding ($s=10$)} & \multicolumn{7}{c}{LDA ($\alpha=0.1$)} \\ \cmidrule{2-15}
      Algorithm & Baseline & \!\!+Prox & \!\!+MOON & \!\!+KD & \!\!+NTD & \!\!+LD & +FD & Baseline & \!\!+Prox & \!\!+MOON & \!\!+KD & \!\!+NTD & \!\!+LD & +FD\\ \midrule
      FedAvg & 37.22 &  30.27 & 36.67 & 35.14 &  35.56 & 34.83 &  37.82 & 42.52 &  36.09 & 42.09 & 41.48 & 41.34 & 43.36 & 43.10\\ \midrule      
      FedBABU & 46.20 &  36.71 & 46.49 & 45.50 &  45.09 & 45.81 &  45.31 & 47.37 &  39.04 & 45.92 & 45.58 & 45.56 & 46.46 & 44.77\\ 
   
      SphereFed & 43.90 &  1.36 & 1.89 &  41.01 & 43.47 &  41.73 & 45.21 &  46.98 & 1.46 & 2.21 &  45.22 & 46.25 & 43.84 &  48.61\\
      
      FedETF & 32.42 &  25.18 & 32.58 & 32.76 &  31.98 & 32.25 &  32.77 & 46.27 &  34.92 & 45.38 & 44.94 & 45.77 & 44.36 & 45.92 \\ 
      
      FedGELA & 29.17 &  25.52 & 28.57 & 28.84 & 28.67  & 28.37 &  29.07 & 27.11 &  26.84 & 28.09 & 27.78 & 28.27 & 27.97 & 27.60\\ \midrule  
      
      Dot-Regression & 42.52 & 5.42 & 44.72 &  46.60 & 45.78 &  47.52 & \textbf{48.69} &  42.72  &  7.47 & 30.69 & 48.19 & 33.08 & 49.09 & \textbf{50.79} \\

      \bottomrule
    \end{tabular}
    }
    \vspace{-5pt}
\end{table*}

\newpage

\subsection{Personalized Federated Learning Results}
\label{ch3subsec:pfl results}

We introduce \alg\,FT and dot-regression\,FT, inspired by prior work~\citep{oh2021fedbabu, dong2022spherefed, li2023no, kim2023fedfn}. These methods enhance personalization by leveraging local data to fine-tune the GFL model. We investigate the impact of fine-tuning using $\mathcal{L}_\text{Dr+}$ and $\mathcal{L}_\text{DR}$ loss for each GFL model to assess their effectiveness on personalized accuracy. Performance metrics without standard deviations indicate results on $D_\text{test}$, obtained from the GFL model after the initial step in the 2-step method. Our experiments involve heterogeneous settings with sharding and LDA non-IID environments, using MobileNet on CIFAR-100 datasets. We set $s$ as 10, 20, and 100, and the LDA concentration parameter ($\alpha$) as 0.05, 0.1, and 0.3. Table~\ref{ch3apptab:pfl_acc} provides detailed personalized accuracy results.

Our 2-step process involves first developing the GFL model either using dot-regression or \alg. In the second step, we fine-tune this model to create the PFL model, again using $\mathcal{L}_\text{DR}$ or $\mathcal{L}_\text{Dr+}$. This results in four combinations: Dot-Regression\,FT\,($\mathcal{L}_\text{DR}$), Dot-Regression\,FT\,($\mathcal{L}_\text{Dr+}$), \alg\,FT\,($\mathcal{L}_\text{DR}$), and \alg\,FT\,($\mathcal{L}_\text{Dr+}$). When the GFL model is fixed, using $\mathcal{L}_\text{DR}$ for fine-tuning consistently outperforms $\mathcal{L}_\text{Dr+}$ across all settings, because dot-regression focuses on local alignment which advantages personalized fine-tuning. Conversely, when the fine-tuning method is fixed, employing $\mathcal{L}_\text{Dr+}$ for the GFL model consistently outperforms $\mathcal{L}_\text{DR}$ across all settings. This aligns with previous research~\citep{nguyen2022begin,chen2022importance} suggesting that fine-tuning from a well-initialized model yields better PFL performance. 

\begin{table}[t!]
    \centering
    \caption{PFL accuracy comparison with MobileNet on CIFAR-100. For PFL, we denote the entries in the form of X{\tiny $\pm$(Y)}, representing the mean and standard deviation of personalized accuracies across all clients derived from a single seed.}
    \label{ch3apptab:pfl_acc}
    \small
    \vspace{5pt}
    \resizebox{\textwidth}{!}{
    \begin{tabular}{l|ccc|ccc}
    \toprule
    Algorithm & $s$=10 & $s$=20 & $s$=100 & $\alpha$=0.05 & $\alpha$=0.1 & $\alpha$=0.3 \\ \midrule

    Dot-Regression                                   & 42.52  & 49.02 & 52.86  & 30.31{\tiny $\pm$7.95} & 37.52{\tiny $\pm$5.60} & 47.08{\tiny $\pm$3.69} \\
    
    Dot-Regression\,FT ($\mathcal{L}_\text{DR}$) & 80.84{\tiny $\pm$(5.99)}  & 74.18{\tiny $\pm$(5.78)} & 56.84{\tiny $\pm$(5.04)} & 72.02{\tiny $\pm$(6.80)}  & 66.96{\tiny $\pm$(5.36)} & 60.34{\tiny $\pm$(3.66)} \\ 
    {Dot-Regression\,FT ($\mathcal{L}_\text{Dr+}$)} & 80.82{\tiny $\pm$(6.12)}  & 73.73{\tiny $\pm$(5.75)} & 56.69{\tiny $\pm$(4.95)} & 71.85{\tiny $\pm$(7.03)}  & 66.59{\tiny $\pm$(5.32)} & 59.87{\tiny $\pm$(3.65)} \\ \midrule
    \alg (ours)                                   & 48.69  & 51.00 & 53.23   & 39.63{\tiny $\pm$9.12}  & 45.83{\tiny $\pm$6.18} & 48.04{\tiny $\pm$3.44} \\    
    \textbf{\alg\,FT ($\mathcal{L}_\text{DR}$) (ours)} & \textbf{84.23}{\tiny $\pm$(5.44)}  & \textbf{75.73}{\tiny $\pm$(4.79)} & \textbf{56.90}{\tiny $\pm$(4.85)} & \textbf{78.65}{\tiny $\pm$(6.17)}  & \textbf{74.86}{\tiny $\pm$(4.77)} & \textbf{62.47}{\tiny $\pm$(3.72)} \\

    \alg\,FT ($\mathcal{L}_\text{Dr+}$) (ours) & 84.10{\tiny $\pm$(5.43)}  & 75.42{\tiny $\pm$(4.80)} & 56.76{\tiny $\pm$(4.91)} & 78.55{\tiny $\pm$(6.16)}  & 74.75{\tiny $\pm$(4.75)} & 62.16{\tiny $\pm$(3.73)} \\ \bottomrule
    \end{tabular}
    }
\end{table}

\subsection{IID Data Performance}

To address the question regarding the performance of \alg or dot-regression loss in Federated Learning (FL) settings with IID data, we conducted experiments on CIFAR-100 with 100 clients, distributing data IID and ensuring a fair number of samples per client. We evaluated FedAvg, FedBABU, Dot-regression, and \alg across 5 seeds, calculating the mean and standard deviation of the global model accuracy for each algorithm.

\begin{table}[h]
    \centering
    \caption{Global model accuracy (\%) in IID data settings.}
    \begin{tabular}{lc}
        \toprule
        \textbf{Algorithm} & \textbf{Accuracy (mean $\pm$ std)} \\
        \midrule
        \textbf{FedAvg}         & 47.19 $\pm$ 1.06 \\
        \textbf{FedBABU}        & 45.18 $\pm$ 0.61 \\
        \textbf{Dot-regression} & 51.48 $\pm$ 0.99 \\
        \alg         & \textbf{51.10 $\pm$ 0.61} \\
        \bottomrule
    \end{tabular}
    \label{ch3tab:iid}
\end{table}

From the Table~\ref{ch3tab:iid}, it is evident that Dot-regression and \alg achieve the highest performance, significantly outperforming both FedAvg and FedBABU. The performance of Dot-regression and \alg is nearly identical under IID settings.

This similarity arises because, in the IID scenario, there are no \textbf{unobserved classes} across clients. As a result, the feature distillation mechanism in \alg, which is specifically designed to mitigate forgetting on unobserved classes, does not provide additional benefits. Instead, both Dot-regression and \alg excel in improving local alignment across all classes, fully achieving the global model's objective of enhancing local alignment for all clients.

\subsection{Performance in Stochastic Client Data Settings}

While our original experiments on \textbf{CIFAR-100 (s=10) with 100 clients} assumed a static client dataset, we conducted additional experiments where each client randomly removed one class from its dataset every 10 FL rounds. As expected, global model accuracy decreased for all methods, as shown in Table~\ref{ch3tab:stochastic}. However, \alg consistently outperformed CE and Dot-regression, demonstrating its robustness in handling dynamic class distributions. The round-wise global test accuracy trends for CE, Dot-regression, and \alg in the stochastic setting are presented in Figure~\ref{ch3fig:stochastic_global_acc}, further confirming \alg’s stability and superior performance across training rounds.

\begin{table}[h]
    \centering
    \caption{Global model accuracy (\%) in static and stochastic client data settings.}
    \begin{tabular}{lcc}
        \toprule
        \textbf{Algorithm} & Static Setting & Stochastic Setting \\
        \midrule
        \textbf{CE}        & 46.20 & 43.59 \\
        \textbf{Dot-regression} & 42.52 & 38.13 \\
        \alg         & \textbf{48.69} & \textbf{44.96} \\
        \bottomrule
    \end{tabular}
    \label{ch3tab:stochastic}
\end{table}

\begin{figure}[h!]
    \centering
        \begin{minipage}{\textwidth}
            \centering    
            \begin{subfigure}[b]{0.32\textwidth}
                \centering
                \includegraphics[width=\textwidth]{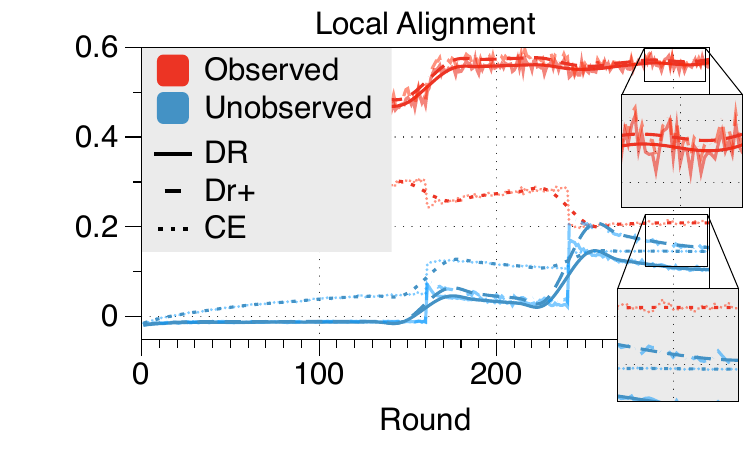}
                \subcaption{Local alignment}
                \label{ch3fig:stochastic_alignment}
            \end{subfigure}
            \hfill
            \begin{subfigure}[b]{0.32\textwidth}
                \centering
                \includegraphics[width=\textwidth]{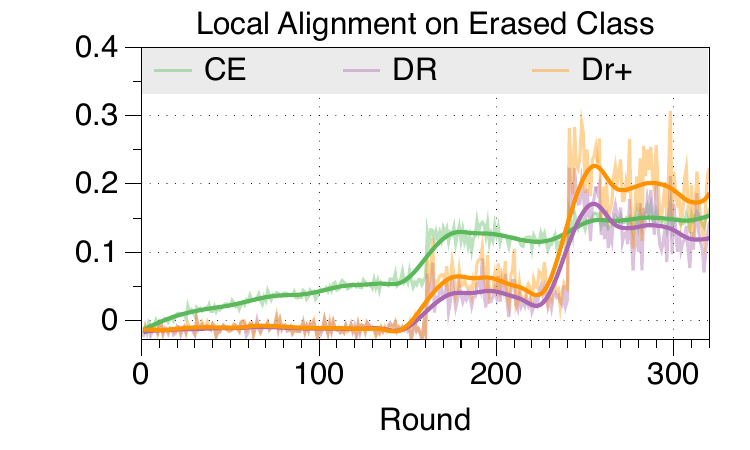}
                \subcaption{Local alignment on erased class}
                \label{ch3fig:stochastic_erased_alignment}
            \end{subfigure}
            \hfill
            \begin{subfigure}[b]{0.32\textwidth}
                \centering
                    \includegraphics[width=\textwidth]{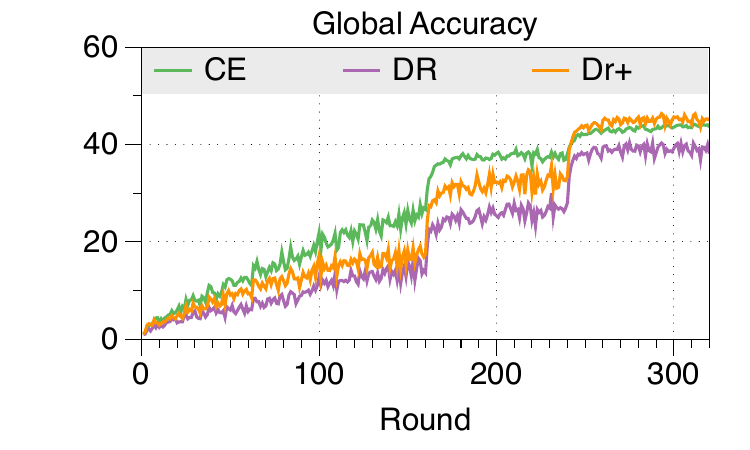}  
                \subcaption{Global test accuracy}
                \label{ch3fig:stochastic_global_acc}
            \end{subfigure}
            \hspace{5pt}
            \caption{Comparison of (a) feature-classifier alignment on the \textcolor{red}{observed} and \textcolor{blue}{unobserved} classes test data, (b) feature-classifier alignment on erased-class test data for $\bm{\theta}_r^i$, and (c) global test accuracy of $\bm{\theta}_r^g$ on all classes. Models are trained using $\mathcal{L}_\text{CE}$, $\mathcal{L}_\text{DR}$, and $\mathcal{L}_\text{Dr+}$.}
            \vspace*{-10pt}
            \label{ch3fig:stochastic_setting}
            \hspace{5pt}
        \end{minipage}
        \hspace{5pt}         
\end{figure}

To further investigate why \alg maintains superior global accuracy in the stochastic setting, we analyzed the feature-classifier alignment for both observed/unobserved classes and erased classes. 

\begin{itemize}
    \item \textbf{Local alignment for observed/unobserved classes (Fig~\ref{ch3fig:stochastic_alignment}):}
    \begin{itemize}
        \item \alg maintains superior feature-classifier alignment for both observed and unobserved classes compared to Dot-regression, consistently outperforming it across all rounds.
        \item During the final convergence phase, \alg surpasses even CE in unobserved class alignment, confirming its effectiveness in preserving global knowledge.
    \end{itemize}
    \newpage
    \item \textbf{Local alignment for erased class (Fig~\ref{ch3fig:stochastic_erased_alignment}):}
    \begin{itemize}
        \item Even for erased class (those removed during training), \alg retains stronger feature-classifier alignment than Dot-regression.
        \item During the final convergence phase, \alg also surpasses CE in erased class alignment, further demonstrating its ability to mitigate forgetting of removed class knowledge.
    \end{itemize}
    
\end{itemize}

These results suggest that the \textbf{feature distillation mechanism in \alg effectively enhances global knowledge preservation while also enabling effective learning of observed classes, even when class distributions change dynamically}.

\subsection{Scaling to Larger Numbers of Clients and Training Rounds}

We conducted experiments on \textbf{CIFAR-100 (s=10)} with \textbf{1,000 communication rounds}, increasing the number of clients to 100, 200, 500, and 1,000. All algorithms used previously grid-searched optimal hyperparameters, and results are averaged over three independent seeds. All algorithms used previously grid-searched optimal hyperparameters, and results are averaged over three independent seeds.

\begin{table}[h]
    \centering
    \caption{Global model accuracy (\%) for different numbers of clients with 1,000 communication rounds.}
    \begin{tabular}{lcccc}
        \toprule
        \textbf{Algorithm} & N=100 & N=200 & N=500 & N=1,000 \\
        \midrule
        \textbf{FedAvg}   & 50.50 $\pm$ 0.57 & 42.51 $\pm$ 1.47 & 33.02 $\pm$ 0.74 & 26.63 $\pm$ 1.31 \\
        \textbf{FedBABU}  & 58.19 $\pm$ 1.07 & 48.75 $\pm$ 1.99 & 37.40 $\pm$ 0.41 & 25.10 $\pm$ 1.08 \\
        \alg   & \textbf{64.21 $\pm$ 1.24} & \textbf{59.78 $\pm$ 0.71} & \textbf{43.27 $\pm$ 0.31} & \textbf{28.99 $\pm$ 0.98} \\
        \bottomrule
    \end{tabular}
    \label{ch3tab:scalability}
\end{table}

Table~\ref{ch3tab:scalability} confirms that \alg consistently outperforms FedAvg and FedBABU across all settings, demonstrating robust scalability in large-scale FL.

\newpage

\section{Limitations}
\label{ch3sec:limitations}

While our proposed method, \alg, effectively enhances both local alignment and global knowledge preservation in Federated Learning (FL), it has certain limitations. First, our approach builds upon dot-regression to improve local alignment, but this is just one possible strategy. Alternative methods, such as directly maximizing local alignment without relying on dot-regression, could be explored to further enhance performance in FL settings. Second, although \alg effectively mitigates forgetting of unobserved classes by incorporating feature distillation, dot-regression alone remains less effective in preserving alignment for unobserved classes. While our empirical results demonstrate that \alg alleviates this issue, further theoretical investigation is needed to develop a more principled approach to ensuring alignment across both observed and unobserved classes. These limitations highlight opportunities for future work to extend and refine our method, improving its robustness and generalizability in diverse FL environments.

\section{Conclusion}
\label{ch3sec:conclusion}




Motivated by the recent FL methods enhancing feature alignment with a fixed classifier, we first investigate the effects of applying dot-regression loss for FL. Since the dot-regression is the most direct method for feature-classifier alignment, we find it improves alignment and accuracy in local models but degrades the performance of the global model. This happens because local clients trained with dot-regression tend to forget classes that have not been observed. To address this, we propose \alg, combining dot-regression with 
a feature distillation method. By regularizing the deviation of local features from global features, \alg allows local models to maintain knowledge about all classes during training, thereby ultimately preserving general knowledge of the global model. Our method achieves top performance in global and personalized FL experiments, even when data is distributed unevenly across devices (non-IID settings).

\chapter{Are Multiple Global Models Necessary in Federated Personalized Reward Model Learning?}
\begin{tcolorbox}[
    colback=gray!10, 
    colframe=black, 
    width=0.9\textwidth, 
    boxrule=1pt, 
    arc=4pt, 
    left=5pt, 
    right=5pt, 
    center
]

Large language models (LLMs) are increasingly aligned to human preferences via reward
modeling and RLHF, but preference data are often decentralized and heterogeneous across
users. Federated learning (FL) offers a natural way to train reward models without
centralizing data, and personalization is typically achieved by fine-tuning a shared
model on each client. A common intuition under preference heterogeneity is that
maintaining multiple global models (e.g., via clustering) during FL should provide
better starting points for such personalized models than a single global model.

We empirically test this intuition for federated personalized reward model learning
under preference heterogeneity. We adopt a fixed two-stage protocol: the server first
trains reward models via FL on pairwise preference data, and each client then obtains a
personalized reward model by locally fine-tuning from the assigned global model. Under
this protocol, we compare a \textsc{Single} FL model against \textsc{Hard} and
\textsc{Soft} multi-global variants across both real-world and synthetic preference
distributions. Our results show that a single FL-trained reward model used as
initialization consistently outperforms all multi-global designs in terms of
personalized performance. Moreover, personalization from a \textsc{Single} FL model
also exceeds personalization from a centrally trained model, suggesting that
concentrating preference conflicts at the FL aggregation stage is more favorable than
mixing conflicting preferences in every centralized batch. These findings indicate that
a carefully trained single global reward model is a strong and often superior backbone
for personalized reward modeling under preference heterogeneity, despite the appeal of
more complex multi-global FL topologies.
\end{tcolorbox}
\section{Introduction}

Large language models (LLMs) are increasingly deployed in human-facing systems,
where alignment with human preferences can improve safety and user satisfaction
\citep{bai2022constitutional,askell2021general,casper2023open}. A standard alignment
pipeline first trains a \emph{reward model} that predicts which response humans prefer,
and then uses this signal to optimize an LLM via RLHF methods~\citep{schulman2017proximal,ouyang2022training}. However, preference
data are highly sensitive and often cannot be centralized due to data-protection
regulations (e.g., GDPR~\citep{regulation2016regulation}, CCPA~\citep{illman2019california}) and cross-jurisdictional sharing constraints~\citep{kopf2023openassistant}. Federated learning (FL)~\citep{mcmahan2017communication,MLSYS2020_1f5fe839,kairouz2021advances}
offers a practical alternative: data remain local while a coordinating server aggregates
model updates from decentralized clients. 

Existing FL-based preference alignment methods~\citep{fan2024fedrlhf,wu2024towards}
aim at a \emph{single} aligned LLM. FedRLHF~\citep{fan2024fedrlhf} trains a single global LLM on data from clients with diverse preferences and reports that the performance of the single global LLM degrades as preference heterogeneity increases. FedBiscuit~\citep{wu2024towards} clusters clients, trains multiple server-side reward models via FL, and then uses these clustered reward models as a voting committee: for each prompt–response pair, their votes determine the final preference label used to supervise a single server-side LLM. As a result, the server learns a single global LLM that produces one response per prompt, effectively optimizing for an aggregate client preference rather than client-specific ones. Under cross-client preference heterogeneity, such a single LLM cannot satisfy all clients, motivating \emph{personalized reward models} that adapt to client-specific preferences and serve as better initializations for downstream personalized LLMs.

Recent personalized FL studies~\citep{oh2021fedbabu,dong2022spherefed,kim2023fedfn,kim2025feddr} demonstrate that initializing each client from a shared global model and then fine-tuning on its local data is a simple yet effective. Under preference heterogeneity, a key question is whether a \textsc{Single} server-side model remains sufficient during federated training. In settings such as preference or task heterogeneity, where clients may favor different predictions for similar inputs, prior  methods~\citep{wu2024towards,wang2025adaptive} maintain multiple server-side models by \textsc{Hard} clustering clients into groups and training a separate model per group: FedBiscuit~\citep{wu2024towards} clusters clients by preference similarity to train clustered reward models, while FedLEASE~\citep{wang2025adaptive} clusters them by task or domain and trains a model per cluster. These designs show that \textsc{Hard}-clustered multi-global FL has been explored as a way to cope with preference heterogeneity during the federated phase. Yet they ultimately deploy and evaluate only a single aligned LLM, \emph{without studying how the server-side design (single-global vs.\ multiple-global) affects the personalized reward models obtained by client-side fine-tuning,} leaving open whether multiple-global reward models are actually needed or a single-global reward model with local fine-tuning already suffices.

\noindent Our contributions are:
\begin{itemize}
    \item We compare FL designs that use either a single global model or multiple
    global models as initializations for personalized reward models under
    cross-client preference heterogeneity.
    \item We find that the single-global design consistently yields better
    personalized performance than multi-global designs, and even outperforms a
    centralized baseline trained on pooled data.
    \item Our findings indicate that the global model obtained by averaging client models trained on their own preference-derived data serves as a strong initialization for personalization.
\end{itemize}

\section{Related Work}
\label{sec:related}

\subsection{Reward Modeling Under Preference Heterogeneity}
Reward modeling from human preference is central to modern alignment pipelines. Preference-labeled summarization data are used to train reward models in \citep{stiennon2020learning}, and RLHF is applied to train instruction-following language models in \citep{ouyang2022training}. Subsequent work proposes alternative preference optimization objectives, including Direct Preference Optimization (DPO)~\citep{rafailov2023direct} and RRHF~\citep{yuan2023rrhf}. To address robustness and diverse user needs, recent studies have introduced more detailed control mechanisms. Robustness and group-wise disparity are addressed by methods such as Group Robust Preference Optimization (GRPO)~\citep{ramesh2024group}, while DPA~\citep{wang2024arithmetic} enables a single global model to satisfy diverse preferences by training multi-objective reward models. Similarly, MiCRo~\citep{shen2025micro} adopts a mixture of experts strategy, training multiple reward components corresponding to latent subgroups and employing a router to dynamically adjust reward weights based on user context. Other studies model latent user-specific preferences via variational formulations~\citep{poddar2024personalizing}. Despite these advances in personalized alignment, most prior work assumes centralized access to preference data,
which is often infeasible due to privacy constraints~\cite{regulation2016regulation, illman2019california, kopf2023openassistant}. Recent federated RLHF studies~\citep{wu2025towards,fan2024fedrlhf}
therefore focus on learning a \emph{single global} aligned policy from decentralized data. However, under explicit
preference heterogeneity, such a global model may not satisfy all clients. Motivated by this gap, we study
\emph{personalized} alignment in federated settings with heterogeneous client preferences.

\subsection{Server-Side Multi-Model Designs in Federated Learning}

Under strong heterogeneity across clients—arising from differences in preference signals,
data distributions, tasks, or objectives—several federated learning studies have explored
maintaining \emph{multiple server-side models} during training.
A common approach clusters clients and trains a separate global model per cluster.
FedBiscuit~\citep{wu2024towards} clusters clients by preference similarity to train clustered
reward models, while FedLEASE~\citep{wang2025adaptive} clusters clients by task or domain and
trains one model per group. Related methods, including IFCA~\citep{ghosh2020efficient}, CFL~\cite{sattler2020clustered},
and FeSEM~\citep{long2023multi}, similarly partition clients and train cluster-specific global
models under general data or task heterogeneity.
Other approaches introduce architectural specialization through partial parameter sharing,
such as mixture-of-experts based designs that decompose models into shared components and
specialized experts~\citep{yi2024pfedmoe}. Such designs introduce \emph{multiple global models} during the federated phase. However, prior work primarily evaluates the resulting \emph{global} models or policies and
typically deploys a single aligned LLM.
The effect of the server-side design choice—\emph{single-global versus multiple-global}—on
the \emph{personalized reward models} obtained via client-side fine-tuning has not been
systematically examined.
In this work, we study this design choice by comparing single and multiple global
initializations for personalized reward modeling under preference heterogeneity.

\section{Preliminaries}
\label{sec:prelim}

This section outlines the FL framework used in this work. 
Clients collaboratively train global reward models from decentralized pairwise preference data and later fine-tune these global models locally to obtain personalized reward models. 
We first describe the FL procedure and then present the experimental setup, including datasets, model configurations, and hyperparameter settings.

\subsection{FL Procedure}
We consider FL with pairwise preference data distributed over $M$ clients. 
Each client $m$ holds a local dataset
$D_{\text{train}}^m = \{(x_i^m, y_{i}^{m,+}, y_{i}^{m,-})\}_{i=1}^{N_m}$,
where $x_i^m$ is a prompt (or question), $y_{i}^{m,+}$ is the preferred response, and $y_{i}^{m,-}$ is the less-preferred response. FL proceeds over $R$ communication rounds. 
At the beginning of round $r$, the server broadcasts the current global parameters $\theta^{(r-1)}$ to a sampled subset of clients $S_r \subset [M]$. 
Each selected client $m \in S_r$ performs $\tau$ local training iterations on $D_{\text{train}}^m$ with batch size $B$, and returns updated parameters $\theta_{m}^{(r)}$. 
The server aggregates these local updates via weighted averaging~\cite{li2019convergence, wu2024towards} to obtain the new global parameters $\theta^{(r)}$. 

The description above corresponds to a single-global design in which the server maintains a single parameter vector $\theta^{(r)}$. In our multi-global settings, the server instead maintains $K$ global models $\{\theta_k^{(r)}\}_{k=1}^K$ and applies the same broadcast–update–aggregate procedure independently to each model over its associated subset of clients. 

After $R$ rounds of federated training (the \emph{global} phase), we run a separate \emph{personalization} phase. 
The server maintains $K$ global models $\{\theta_k^{(R)}\}_{k=1}^K$, and each client $m$ receives a single initialization $\theta_m^{\text{init}}$ constructed from these $K$ models. 
Client $m$ then fine-tunes $\theta_m^{\text{init}}$ on its own preference data $D_{\text{train}}^m$ without further communication to obtain a personalized reward model $\theta_m^{\text{PFL}}$. 
Throughout the paper, we refer to $\{\theta_k^{(R)}\}_{k=1}^K$ as \emph{global} models and to $\theta_m^{\text{PFL}}$ as the resulting \emph{personalized} reward model for client $m$.

\subsection{Datasets and Preference Heterogeneity Design}

We evaluate our method using two types of preference datasets: a real-world dataset that captures naturally occurring annotator-specific preferences, and a synthetic dataset designed to systematically control style-based preference heterogeneity. In both datasets, we apply client-level splitting by shuffling each client's local data and allocating 50 samples to validation, 50 to test, and using the remaining samples for training.

\begin{itemize}
\item \textbf{Real-World Dataset.} 
We use the Reddit TL;DR summarization dataset~\citep{stiennon2020learning, volske2017tl}, which includes human preference annotations with worker identifiers. We select only workers who appear in all three splits (train, valid1, valid2) and have at least 100 samples in total, resulting in 34 clients and 144{,}502 samples. Each client corresponds to a real annotator, resulting in naturally imbalanced data quantities across clients and preserving real-world preference heterogeneity without any synthetic manipulation.

\item \textbf{Synthetic Dataset.} 
To construct controlled preference heterogeneity, we merge prompts from UltraFeedback~\citep{cui2023ultrafeedback} and p-Soups~\citep{jang2023personalized}, yielding 52K unique prompts. Using GPT-4o-mini\footnote{\url{https://platform.openai.com/docs/models/gpt-4o-mini}}, we generate two responses per prompt along two preference style axes: (1) \emph{Elementary} vs.\ \emph{PhD-level}, and (2) \emph{Humorous} vs.\ \emph{Non-humorous}, with 26K prompts assigned to each axis. We simulate 40 clients, each receiving 650 samples per axis. Clients are grouped into four dominant preference groups: (G1) Elementary or Humorous, (G2) Elementary or Non-humorous, (G3) PhD or Humorous, and (G4) PhD or Non-humorous, with 10 clients in each group, forming four equally sized preference-aligned clusters. For clarity, we index clients so that clients 1–10 belong to G1, 11–20 to G2, 21–30 to G3, and 31–40 to G4.

\end{itemize}

\subsection{Implementation Details} 
Unless explicitly mentioned otherwise, all experiments share the following configuration. We run 200 communication rounds of FL, sampling $|S_r|=5$ clients per round, and each selected client performs exactly 30 local training iterations per round with a batch size of 16, using AdamW~\cite{loshchilov2018decoupled} with $(\beta_1, \beta_2) = (0.9, 0.95)$ and a constant learning rate of $1\times10^{-5}$. In configurations with multiple global models and dynamic client--model assignment, the assignments are refreshed every $T=20$ rounds. Qwen2-0.5B~\cite{yang2024qwen2technicalreport} is used as the base model, and LoRA-based parameter-efficient tuning~\cite{hu2022lora, houlsby2019parameter} is applied with rank $r=8$, scaling factor $\alpha=16$, and dropout rate $0.05$. Following the standard Transformer architecture~\cite{vaswani2017attention}, LoRA is injected into the attention projection layers (\texttt{q\_proj}, \texttt{k\_proj}, \texttt{v\_proj}, \texttt{o\_proj}) and the MLP projection layers (\texttt{gate\_proj}, \texttt{up\_proj}). All experiments are implemented on top of the FedBiscuit~\cite{wu2024towards} codebase.\footnote{We build on the official FedBiscuit repository (\texttt{https://github.com/HarliWu/FedBiscuit}) and directly reuse its binary-selector reward model without modification. Given a prompt and two candidate responses $(y^{m,+}, y^{m,-})$, a single forward pass over the concatenated input produces a 2-dimensional preference logit over special option tokens corresponding to the two responses, which is trained with binary cross-entropy~\cite{wu2025towards} against a binary label indicating whether $y^{m,+}$ or $y^{m,-}$ is preferred.}

\vspace{-15pt}
\section{Empirical Motivation: Does Increasing Global Models Help?}
\label{sec:motivation}

Given our focus on personalized reward models under preference heterogeneity, we ask whether a multi-global FL design is beneficial for downstream personalization. We therefore compare a single-global model (\textsc{Single}) with a multi-global variant based on hard client clustering (\textsc{Hard}). In summary, from the perspective of providing an initialization for personalization, \textsc{Single} is preferable to \textsc{Hard}.

\paragraph{Hard multi-global design.}
We instantiate the multi-global setting using a dynamic hard-clustering design~(\textsc{Hard}) inspired by the multi-model design in FedBiscuit~\citep{wu2024towards}. Every $T$ communication rounds, the server recomputes client--model assignments by evaluating all $K$ global models on each client's validation set and assigning each client $m$ to the single model index
$k(m) \in [K]$ that attains the lowest validation loss, forming hard clusters
$\{\mathcal{A}_k\}_{k=1}^K$ with $\mathcal{A}_k = \{m : k(m)=k\}$. Between these reassignment steps, each round $r$ proceeds as follows: the server samples a subset of participating clients $S_r \subset [M]$, and each client $m \in S_r$ receives only the global model of its current cluster, $\theta_{k(m)}^{(r-1)}$, updates it locally on $\mathcal{D}_{\text{train}}^{m}$, and returns the updated parameters. The server then aggregates updates \emph{cluster-wise}: for each $k$, it updates $\theta_k^{(r)}$ using only the clients in $S_{r,k} = S_r \cap \mathcal{A}_k$. After the global phase, each client again selects the global model that yields the lowest validation loss on its data and fine-tunes it locally to obtain $\theta_m^{\text{PFL}}$. A full algorithmic description is deferred to Subsection~\ref{app:alg_hard}.
\vspace{-5pt}
\subsection{\textsc{Single} Beats Multi-Global FL}
\label{sec:motivation-single-vs-hard}

We find that \textsc{Single} provides a \emph{better initial point} for personalized reward models than \textsc{Hard}. Figure~\ref{fig:pfl-single-vs-hard} reports, for both the real-world and synthetic datasets, the mean preference-prediction accuracy of $\theta_m^{\text{PFL}}(\tau)$ on $\mathcal{D}_{\text{test}}^{m}$, averaged over all clients $m$, as a function of the number of local fine-tuning updates $\tau$. For each design, we take the final global model(s) from the federated phase and run client-side fine-tuning for $\tau \in \{0,80,160,240\}$ local updates on $\mathcal{D}_{\text{train}}^{m}$, where $\tau=0$ corresponds to the mean personalized accuracy of the global model assigned to each client before any local updates. 
\begin{figure*}[t]
    \centering
    \begin{subfigure}[t]{0.47\textwidth}
        \centering
        \includegraphics[height=0.38\textheight]{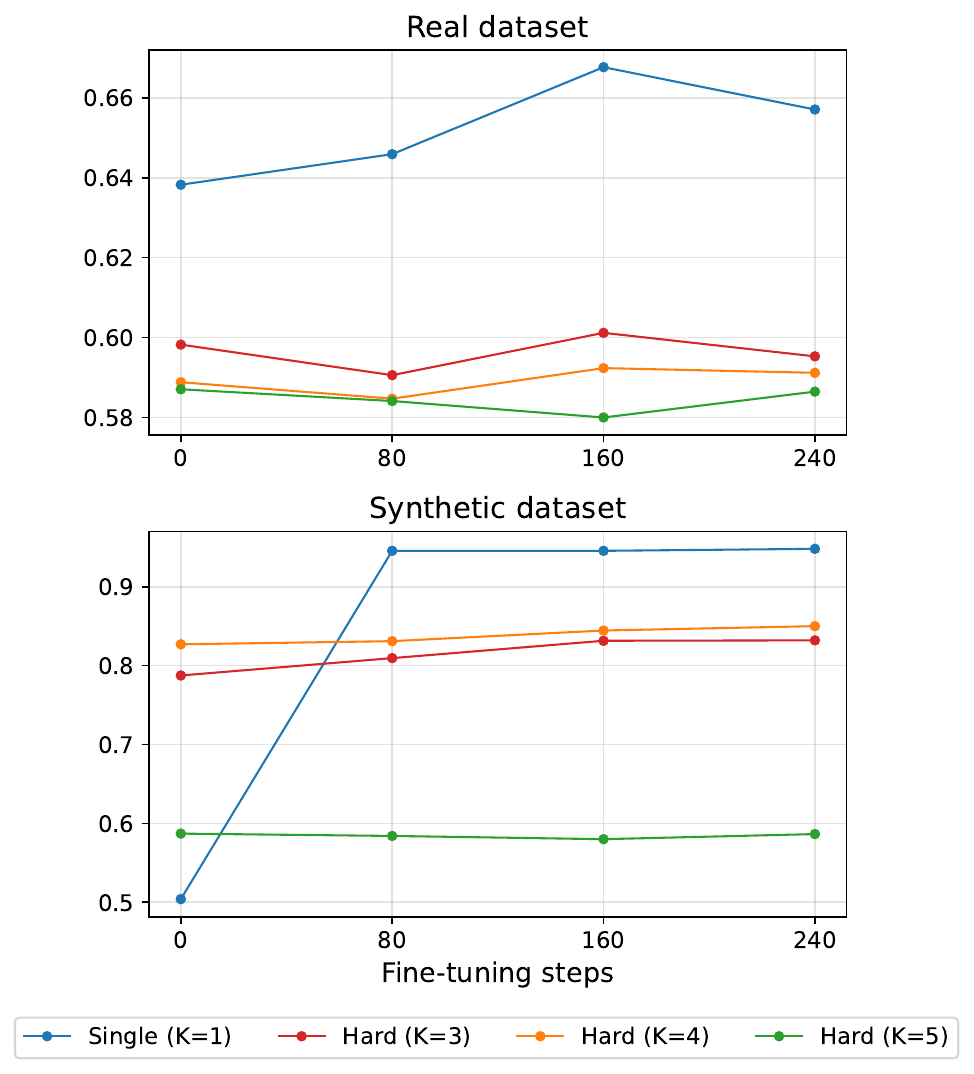}
        \caption{PFL test accuracy}
        \label{fig:pfl-single-vs-hard}
    \end{subfigure}
    \hfill
    \begin{subfigure}[t]{0.47\textwidth}
        \centering
        \includegraphics[height=0.38\textheight]{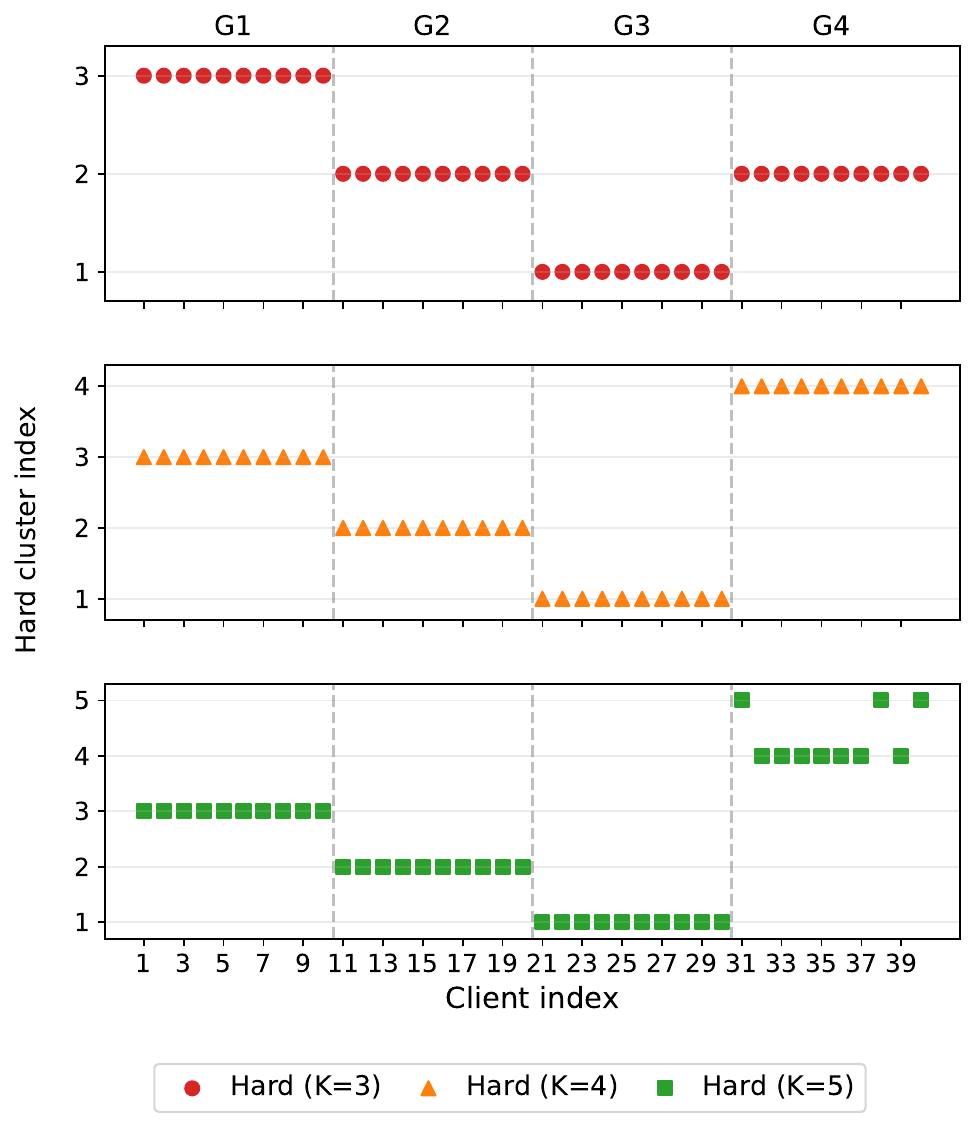}
        \caption{Hard-cluster assignment}
        \label{fig:pfl-hard-assignment}
    \end{subfigure}

    \caption{Comparison of (a) mean preference-prediction accuracy of $\theta_m^{\text{PFL}}$ on $\mathcal{D}_{\text{test}}^{m}$ as a function of fine-tuning steps for the real and synthetic datasets, comparing a \textsc{single} with \textsc{hard}~($K\in\{3,4,5\}$), and (b) hard-cluster indices assigned to
each client on the synthetic dataset for the same
\textsc{hard} settings.}
    \label{fig:pfl-single-hard-overall}
\end{figure*}

Across fine-tuning budgets $\tau>0$, \textsc{Single} consistently achieves higher personalized accuracy than \textsc{Hard} with $K \in \{3,4,5\}$ on both the real-world and synthetic datasets. This advantage holds even at $\tau=0$ (before fine-tuning) on the real-world dataset. On the synthetic dataset, we simulate 40 clients split into four equally sized preference groups G1--G4 (10 clients each). Clients within a group share a dominant style preference, while preferences across groups are deliberately conflicting, so a single global model cannot align well with all groups simultaneously. Consistent with this construction, \textsc{Single} starts around $50\%$ accuracy at $\tau=0$, close to random guessing and below all \textsc{Hard} settings with $K \in \{3,4,5\}$.
\begin{wrapfigure}[17]{r}{0.47\textwidth}
    \begin{minipage}{0.47\textwidth}
        \centering
        \vspace{-15pt}
        \includegraphics[width=\textwidth]{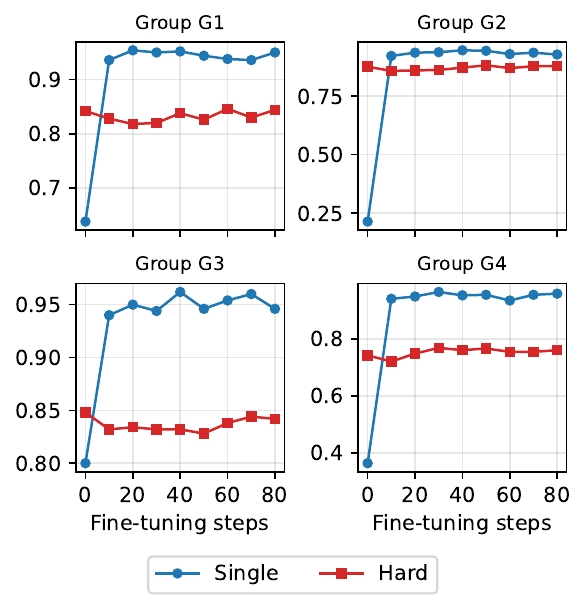}
        \vspace{-24pt}
        \caption{Mean preference-prediction accuracy of
    $\theta_m^{\text{PFL}}$ on $\mathcal{D}_{\text{test}}^{m}$ as a function of fine-tuning steps on the synthetic dataset, averaged over clients
    within each preference group (G1–G4), comparing \textsc{Single} and
    \textsc{Hard}.}
    \label{fig:synthetic-single-vs-hard-groups}
    \end{minipage}
\end{wrapfigure}
However, this initialization advantage of \textsc{Hard} on the synthetic dataset is modest and is quickly overturned once local adaptation is allowed.

\subsection{Why More Global Models Hurt Personalization?}
\label{sec:motivation-synthetic-analysis}

On the synthetic dataset, \textsc{Hard} performs best at $K=4$ because the hard clusters recover the four ground-truth preference groups (G1--G4). Figure~\ref{fig:pfl-hard-assignment} shows the final client-to-cluster assignments for $K \in \{3,4,5\}$ alongside the true groups. When $K=4$, each group maps to a distinct cluster, giving the highest personalization accuracy among \textsc{Hard} variants. 
When $K=3$, assignments are consistent within each group, but the fewer clusters merge different groups (e.g., G2 and G4), reducing accuracy relative to $K=4$. 
When $K=5$, extra clusters split some groups (notably G4) across multiple clusters, hurting personalization accuracy.

Even in the most favorable \textsc{Hard} setting with $K=4$, \textsc{Single} starts from a worse initial point but adapts better. Figure~\ref{fig:synthetic-single-vs-hard-groups} compares \textsc{Single} and \textsc{Hard} at the group level, plotting personalized accuracy for $\tau \in \{0,10,20,\dots,80\}$. At $\tau=0$, \textsc{Single} is a poor compromise (over $60\%$ accuracy on G1/G3 but under $40\%$ on G2/G4) and is worse than \textsc{Hard} in every group. After about 10 local updates, however, \textsc{Single} already surpasses \textsc{Hard} in all four groups, and this gap generally widens as $\tau$ increases.

\section{Why Can a Single Global Model Be Enough?}
\label{sec:analysis}

We compare a broader set of global initialization strategies and find that, on both the real-world and synthetic datasets, \textsc{Single} yields the strongest personalized reward models after fine-tuning.

\paragraph{Global initialization topologies.}
In addition to \textsc{Single} and \textsc{Hard}, we consider four further strategies for constructing global initializations.
(1) \textsc{Soft} (Soft-Fusion) maintains $K$ global experts together with client-specific soft weights. Every $T$ rounds, these soft assignments are refreshed from validation performance; between updates, each round broadcasts to client $m$ a single fused model formed from its current soft weights and aggregates the returned update back into all experts using the same weights. In the personalization phase, each client starts from its final fused model and then fine-tunes locally.
(2) \textsc{Hard-Oracle} fixes client clusters in advance, runs independent FL within each cluster to obtain cluster-specific global models, and then personalizes by sending each client the global model of its cluster as the initialization for local fine-tuning. On the synthetic dataset, we exploit the known preference groups G1–G4 and set $K=4$, assigning each group to its own cluster. On the real-world dataset, we instead derive fixed clusters by first training local-only models per client and then applying $K$-medoids clustering~\citep{kaufman2009finding} to a distance matrix built from Matthews correlation coefficients (MCC)~\citep{matthews1975comparison} between their predictions.
(3) \textsc{Centralized} breaks the FL constraint by pooling all preference data on a single server, training one centralized global reward model with a batch size scaled by $|S_r|$ to mirror the effective per-round averaging in \textsc{Single}, and then fine-tuning it separately on each client’s data.
(4) \textsc{Base} removes federated training altogether and obtains each personalized model by locally fine-tuning a randomly initialized reward model on that client’s data.
Full algorithmic details for \textsc{Soft} and \textsc{Hard-Oracle} are provided in Subsections~\ref{app:alg_softfusion} and~\ref{app:hard_oracle}.

\subsection{Within FL: Single vs.\ Multi-Global}
\label{sec:single-vs-multi}

\textsc{Single} provides the best initialization for personalization, outperforming all \emph{multi-global FL topologies}---\textsc{Hard}, \textsc{Soft}, and \textsc{Hard-Oracle}---on both datasets (Figure~\ref{fig:global-init-topologies}). We fix $K$ based on the best-performing \textsc{Hard} choice in each dataset: $K=3$ on the real-world dataset and $K=4$ on the synthetic dataset. Across all fine-tuning budgets $\tau$, \textsc{Single} achieves the highest mean personalized accuracy, while the multi-global methods follow a consistent ordering \textsc{Hard-Oracle} $>$ \textsc{Hard} $\gtrsim$ \textsc{Soft} that is already visible at $\tau=0$. On the real dataset, all multi-global variants remain strictly below \textsc{Single} for every $\tau$. On the synthetic dataset, \textsc{Hard-Oracle} (with preference-aligned clusters) is strongest at $\tau=0$, but \textsc{Single} overtakes all multi-global variants after only a small amount of local fine-tuning. Together, these results suggest that pooling all training signal into one global model provides a better initialization for personalization than splitting it across multiple globals.

\begin{figure}[t]
    \centering
    \includegraphics[width=0.85\linewidth]{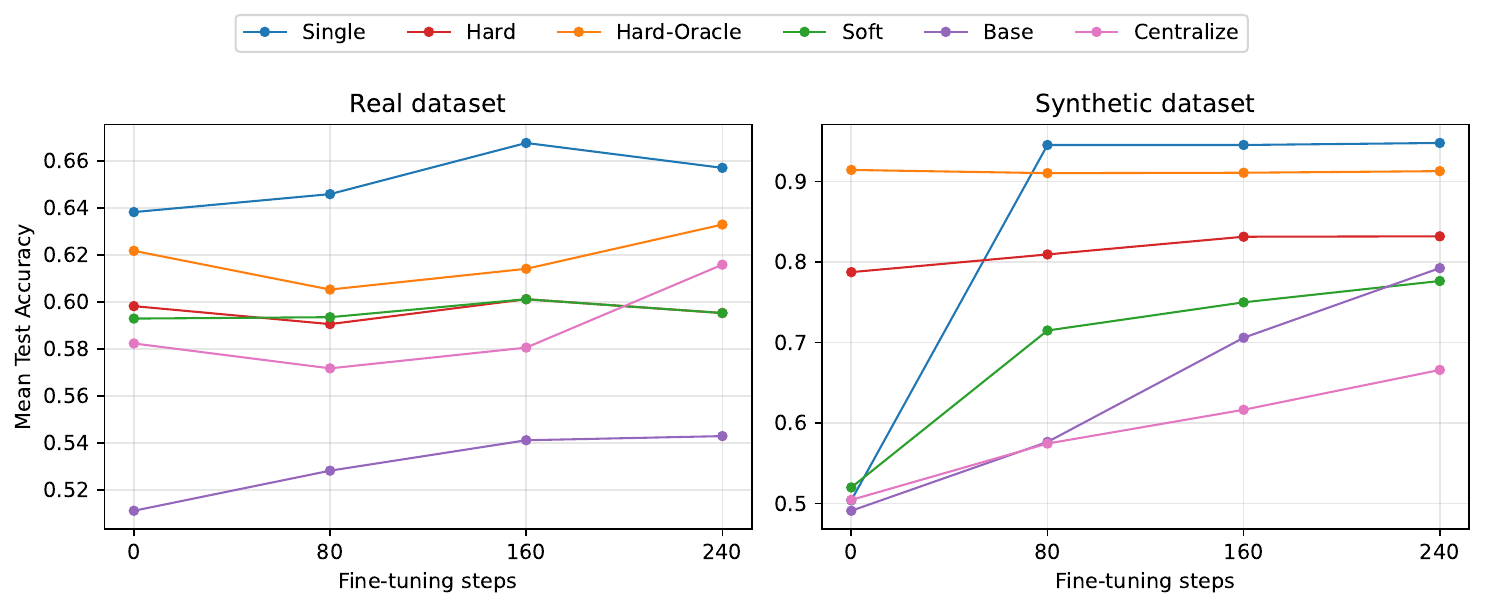}
    \caption{Mean preference-prediction accuracy on the real (left) and synthetic (right) datasets as a function of fine-tuning steps, comparing global initialization strategies \textsc{Single}, \textsc{Hard}, \textsc{Hard-Oracle}, \textsc{Soft}, \textsc{Base}, and \textsc{Centralized}.}
    \label{fig:global-init-topologies}
\end{figure}

\begin{figure}[h!]
    \centering
    \includegraphics[width=0.85\linewidth]{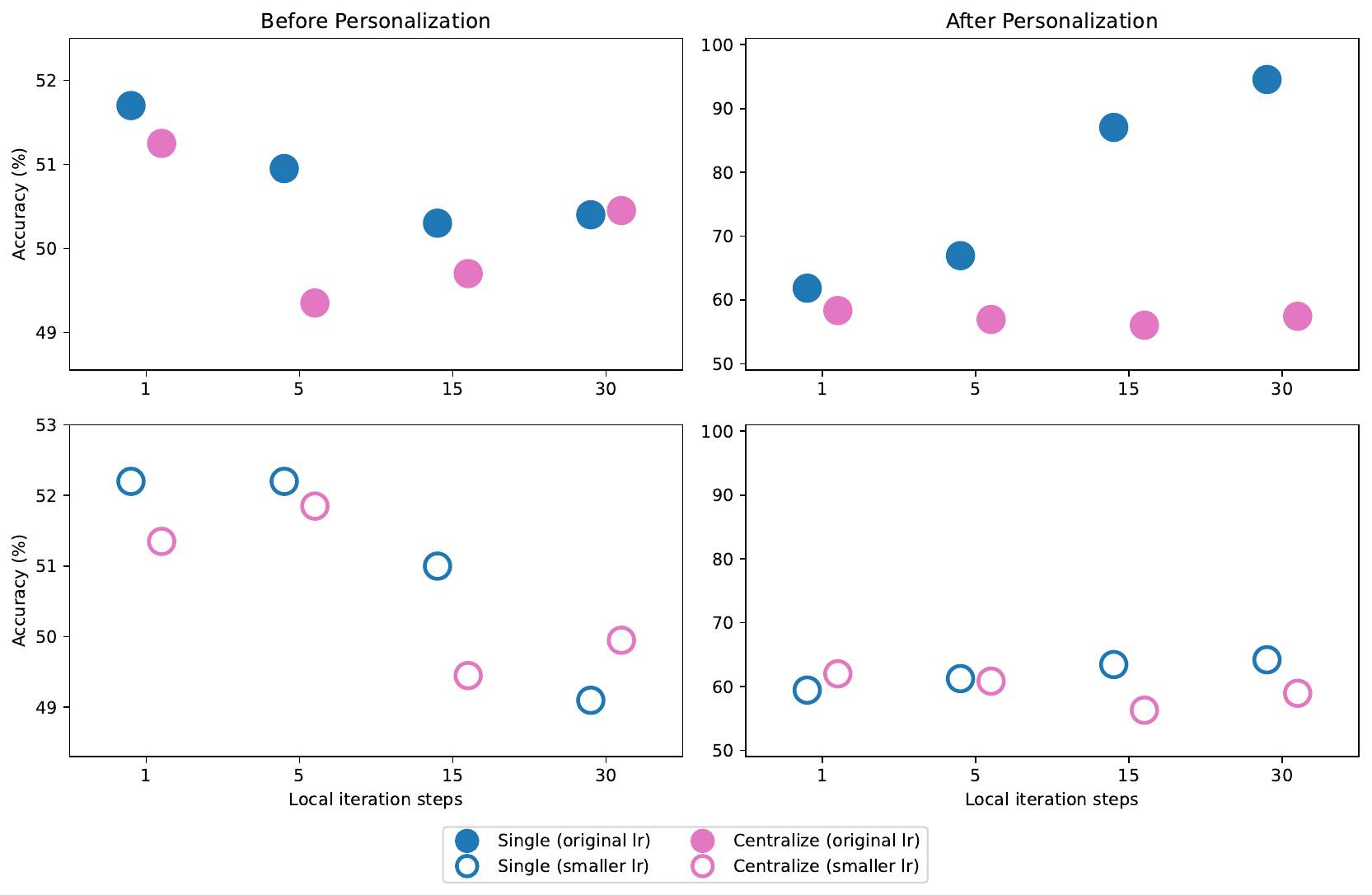}
    \caption{Sensitivity of \textsc{Single} vs.\ \textsc{Centralized} on the synthetic dataset under matched compute. We vary local iteration steps $\tau \in \{1,5,15,30\}$ and the learning rate (default vs.\ $10\times$ smaller), and report mean accuracy before personalization ($\tau=0$) and after $80$ fine-tuning updates.}

    \label{fig:single-centralize-sensitivity}

\end{figure}

\subsection{Preference Conflicts: Centralized Batches vs.\ FL Aggregation}
\label{sec:fl-vs-centralized}

Across both datasets, \textsc{Single} provides a better initialization for personalization than \textsc{Centralized}. As shown in Figure~\ref{fig:global-init-topologies}, \textsc{Single} consistently outperforms \textsc{Centralized} across fine-tuning budgets $\tau$, and on the synthetic dataset \textsc{Centralized} even underperforms the random-start baseline (\textsc{Base}). A plausible explanation is that centralized training mixes opposing preferences within batches, yielding inconsistent supervision and poorer downstream generalization ~\citep{liu2020bad}.

Crucially, allowing each client to sufficiently adapt the model to its own preference data before aggregation mitigates this issue. Figure~\ref{fig:single-centralize-sensitivity} corroborates this through a sensitivity analysis that varies the learning rate (default vs.\ $10\times$ smaller) and the local iteration steps used to construct the initial model ($\tau\in\{1,5,15,30\}$) under matched compute. While both methods start near random accuracy before personalization, after $80$ fine-tuning updates \textsc{Single} increasingly outperforms \textsc{Centralized} as $\tau$ grows (with similar behavior at $\tau=1$), indicating that sufficient client-side adaptation prior to averaging is key to producing an effective initialization for downstream personalization.

\newpage
\section{Algorithmic Details of Global Initialization Topologies for Federated Personalized Reward Models}
\label{sec:alg_details}
In this section, we detail the three training topologies used in our empirical study.
All methods share the same overall goal: first learn a \emph{global} initialization
via federated training (GFL), and then obtain \emph{personalized} reward models (PFL)
by local fine-tuning on each client's preference data.
The topologies differ only in how they construct and update global models during GFL:
(i) a single shared global model (\textsc{Single}),
(ii) multiple hard-clustered global models (\textsc{Hard} / \textsc{Hard-Balance}),
(iii) a soft-clustered multi-global scheme based on fusion weights (\textsc{Soft}),
and (iv) an oracle hard-clustered multi-global scheme with fixed clusters
(\textsc{Hard-Oracle}). We use \textsc{Hard} as our main multi-global design,
and treat \textsc{Hard-Balance} as a balanced variant mainly reported in the Section~\ref{sec:suppl}. The algorithms are agnostic to the specific reward-model parameterization
(e.g., binary selector or scalar-valued reward); they only assume a differentiable
preference loss defined on $(x, y^{+}, y^{-})$ tuples.

In all three topologies, we implement federated training using
LoRA-style parameter-efficient fine-tuning~\citep{hu2022lora,houlsby2019parameter}:
each client and the server share the same frozen base LLM for the reward
model, and the trainable parameters $\theta$ correspond only to the attached
LoRA adapters. During communication, only these low-rank adapters are transmitted and
aggregated, while the shared base model remains frozen on all clients, so
both communication and memory overhead are dominated by the small LoRA
parameters.
\subsection{Single Global Model}
\label{app:alg_single}

\begin{algorithm}[h!]
\small
\caption{Single: Federated Training of a Single Global Reward Model and Final Personalization}
\label{alg:single}
\KwInput{%
  number of rounds $R$; local steps $\tau$; client set $[M]$;
  initial global parameters $\theta^{(0)}$}
\KwResult{%
  global reward model $\theta^{(R)}$ and personalized models
  $\{\theta_m^{\text{PFL}}\}_{m=1}^M$
}

\BlankLine
\textbf{Phase 1: Federated training of a single global model}\;

\For{$r = 1,2,\dots,R$}{
  Sample a subset of participating clients $S_r \subseteq [M]$\;
  
  \tcp{(1) Server $\rightarrow$ clients: broadcast current global model}
  \ForEach{$m \in S_r$ \textbf{in parallel}}{
    Client $m$ starts from $\theta^{(r-1)}$\;
    Run $\tau$ local update steps on $D_{\text{train}}^m$ and send
      $\theta^{(r)}_{m}$ to the server\;
  }

  \tcp{(2) FedAvg aggregation at the server}
  $\theta^{(r)} \leftarrow \frac{1}{|S_r|} \sum_{m\in S_r}\,\theta_{m}^{(r)}$\;
}

\BlankLine
\textbf{Phase 2: Final personalization from the global initialization}\;

\tcp{(1) Server $\rightarrow$ clients: broadcast final global model}
Server broadcasts $\theta^{(R)}$ to all clients\;

\tcp{(2) Client-side personalization }
\ForEach{$m \in [M]$ \textbf{in parallel}}{
  Initialize $\theta_m^{(\text{PFL},0)} \leftarrow \theta^{(R)}$\;\\
  Fine-tune on $D^m_{\text{train}}$ to obtain
    $\theta_m^{\mathrm{PFL}}$\;
}

\Return{$\theta^{(R)},\ \{\theta_m^{\text{PFL}}\}_{m=1}^M$}\;
\end{algorithm}

\paragraph{Description.}
Algorithm~\ref{alg:single} implements the simplest design: a single shared global reward model trained with FedAvg~\citep{mcmahan2017communication} on preference tuples $(x, y^{+}, y^{-})$, followed by client-side fine-tuning.
At round $r$, the server holds the current global parameters $\theta^{(r-1)}$, samples a subset of clients $S_r \subseteq [M]$, and sends $\theta^{(r-1)}$ to each $m \in S_r$. Each selected client performs $\tau$ local optimization steps
on its local preference data and returns an updated model $\theta_m^{(r)}$; the server then averages $\{\theta_m^{(r)}\}_{m\in S_r}$ to obtain the next
global model $\theta^{(r)}$. This corresponds to the single-global training design used in FedBis~\citep{wu2024towards}, which we adopt as a baseline for studying the role of global initialization in personalized reward modeling.

In Phase~2, the final global model $\theta^{(R)}$ is reused purely as a \emph{global initialization} for personalization~\citep{oh2021fedbabu}: each client $m$ starts from $\theta^{(R)}$, performs additional local updates
on $D^m_{\text{train}}$ without further communication, and obtains a personalized reward model $\theta_m^{\text{PFL}}$. Thus, the ``Single'' design consists of one shared global reward model trained federatively, followed by independent local adaptation on each client.

\subsection{Hard-Clustered Multi-Global Models (Hard / Hard-Balance)}
\label{app:alg_hard}

\begin{algorithm}[h!]
\small
\caption{Hard: Federated Training with Hard-Clustered Global Models and Final Personalization (Hard/ Hard-Balance)}
\label{alg:hard}
\KwInput{%
  number of global models $K$; total rounds $R$;
  reassignment period $T$; local update steps $\tau$;
  client set $[M]$; client weights $\{p_m\}_{m\in[M]}$;
  initial parameters $\{\theta_k^{(0)}\}_{k=1}^K$;
  mode $\textsc{mode} \in \{\textsc{Hard-Plain},\textsc{Hard-Balance}\}$
}
\KwResult{%
  hard-clustered global models $\{\theta_k^{(R)}\}_{k=1}^K$ (GFL) and
  personalized models $\{\theta_m^{\text{PFL}}\}_{m=1}^M$
}

\BlankLine
\textbf{Phase 1: Warm-up (independent FedAvg for each model)}\;

\For{$k = 1,2,\dots,K$}{
  \For{$t = 1,2,\dots,T$}{
    Sample a client subset $S_{t,k} \subseteq [M]$\;
    \ForEach{$m \in S_{t,k}$ \textbf{in parallel}}{
      Client $m$ receives $\theta_k^{(t-1)}$\;
      Run $\tau$ local update steps on $D_{\text{train}}^m$ and send
      $\theta^{(t)}_{k,m}$ to the server\;
    }
    \tcp{FedAvg update for model $k$}
    $\theta_k^{(t)} \leftarrow \frac{1}{|S_{t,k}|}
      \sum_{m\in S_{t,k}} \theta^{(t)}_{k,m}$\;
  }
}

\BlankLine
\textbf{Phase 2: Dynamic client-driven hard assignment}\;

\For{$r = T+1,T+2,\dots,R$}{

  \If{$r-1 \equiv 0\;(\bmod T)$}{
    \tcp{(1) Client-driven reassignment by validation loss}
    Server broadcasts $\{\theta_k^{(r-1)}\}_{k=1}^K$ to all clients\;
    \ForEach{$m \in [M]$}{
      $\ell_{m,k} \gets
        \text{ValLoss}(\theta_k^{(r-1)}, D^m_{\mathrm{val}})$
      for all $k \in [K]$\;
      $k(m) \gets \arg\min_{k} \ell_{m,k}$\;
    }
    $\mathcal{A}_k \leftarrow \{\,m \in [M] : k(m) = k\,\}$
    for all $k \in [K]$\;

    \tcp{(2) Optional balancing step (\textsc{Hard-Balance} only; empirical results in the Section~\ref{sec:suppl})}
    \If{$\textsc{mode} = \textsc{Hard-Balance}$}{
      \While{$\max_k |\mathcal{A}_k| - \min_k |\mathcal{A}_k| > 1$}{
        $k_{\max} \gets \arg\max_{k} |\mathcal{A}_k|$\;
        $k_{\min} \gets \arg\min_{k} |\mathcal{A}_k|$\;\\
        Select $m \in \mathcal{A}_{k_{\max}}$ whose second-best model index
        $v$ satisfies $v = k_{\min}$\;\\
        Move $m$ from $\mathcal{A}_{k_{\max}}$ to $\mathcal{A}_{k_{\min}}$\;
      }
    }
  }

  \tcp{(3) Single global sampling and cluster-wise training}
  Sample a subset of participating clients $S_r \subseteq [M]$\;\\
  \For{$k = 1,2,\dots,K$}{
    Define $S_{r,k} \leftarrow S_r \cap \mathcal{A}_k$
    \tcp*{$\{S_{r,k}\}_{k=1}^K$ is a disjoint partition of $S_r$}
    \If{$S_{r,k} \neq \emptyset$}{
      Server sends $\theta_k^{(r-1)}$ to all $m \in S_{r,k}$\;\\
      \ForEach{$m \in S_{r,k}$ \textbf{in parallel}}{
        Client $m$ runs $\tau$ local update steps on $D_{\text{train}}^m$
        and sends $\theta^{(r)}_{k,m}$ to the server\;
      }
      \tcp{Weighted aggregation with residual of previous model}
      $\theta_k^{(r)} \leftarrow 
        \Bigl(1 - \sum_{m\in S_{r,k}} p_m \Bigr)\theta_k^{(r-1)}
        + \sum_{m\in S_{r,k}} p_m \,\theta^{(r)}_{k,m}$\;
    }
  }
}

\BlankLine
\textbf{Phase 3: Final personalization (common to Hard/ Hard-Balance)}\;

\tcp{(1) Server $\rightarrow$ clients: final hard assignment by validation loss}
Server broadcasts $\{\theta_k^{(R)}\}_{k=1}^K$ to all clients\;
\ForEach{$m \in [M]$}{
  $\ell_{m,k}^{\text{final}} \gets
    \text{ValLoss}(\theta_k^{(R)}, D^m_{\mathrm{val}})$
  for all $k \in [K]$\;
  $k_m^{\text{final}} \gets \arg\min_{k} \ell_{m,k}^{\text{final}}$\;
}

\tcp{(2) Client-side personalization from the assigned global model}
\ForEach{$m \in [M]$ \textbf{in parallel}}{
  Initialize $\theta_m^{(\text{PFL},0)} \leftarrow
    \theta_{k_m^{\text{final}}}^{(R)}$\;
  Fine-tune on $D^m_{\text{train}}$ to obtain
    $\theta_m^{\mathrm{PFL}}$\;
}

\Return{$\{\theta_k^{(R)}\}_{k=1}^K,\ \{\theta_m^{\text{PFL}}\}_{m=1}^M$}\;
\end{algorithm}

\paragraph{Description.}
Algorithm~\ref{alg:hard} is a direct adaptation of the multi-global design
proposed in FedBiscuit~\citep{wu2024towards}, reinterpreted explicitly in an
EM-style manner~\citep{dempster1977maximum} for our personalized reward
modeling setting. The server maintains $K$ cluster-specific global models,
but to avoid the prohibitive communication cost of broadcasting all $K$
models at every FL round, it sends \emph{only one} model to each client:
in round $r$, client $m$ receives the model of its assigned cluster and
performs local updates starting from this model. Cluster assignments are
updated only once every $T$ rounds via a hard-clustering step.

The EM interpretation is as follows:
\begin{itemize}
    \item \textbf{E-step (assignment):} every $T$ rounds, each client $m$
    evaluates all $K$ models on its validation set $D^m_{\mathrm{val}}$
    and is hard-assigned to the model $k_m$ with the lowest validation loss,
    forming clusters $\{\mathcal{A}_k\}_{k=1}^K$.
    \item \textbf{M-step (clustered FL):} in the subsequent $T$ rounds,
    each cluster $\mathcal{A}_k$ runs standard FL \emph{only} on its own
    model $\theta_k$. At round $r$, the server samples a global set $S_r$
    of participating clients, induces per-cluster subsets
    $S_{r,k} = S_r \cap \mathcal{A}_k$, sends $\theta_k^{(r-1)}$ to clients
    in $S_{r,k}$, and updates $\theta_k^{(r)}$ via a modified aggregation
    rule described below.
\end{itemize}

Algorithm~\ref{alg:hard} proceeds in three phases. Phases~1 and~2 jointly
train the $K$ global models in this multi-global GFL design, and Phase~3
converts them into personalized PFL solutions. In Phase~1, each global model is warmed up independently with FedAvg for $T$ rounds to obtain diverse initializations before the first E-step. Phase~2 then runs the EM loop: every $T$ rounds, clients are reassigned to models based on validation loss (E-step), followed by $T$ rounds of clustered FL where each cluster $\mathcal{A}_k$ updates only its own model $\theta_k$ (M-step).

In Phase~2, the aggregation step for each cluster model follows the rule
originally proposed in~\citep{wu2024towards}:
\begin{equation}
    p_m
    =
    \frac{|D^{m}_{\text{train}}|}
         {\bigl|\cup_{j\in[M]} D^{j}_{\text{train}}\bigr|},
    \qquad
    \theta_k^{(r)}
    =
    \Bigl(1 - \sum_{m\in S_{r,k}} p_m \Bigr)\,\theta_k^{(r-1)}
    +
    \sum_{m\in S_{r,k}} p_m \,\theta^{(r)}_{k,m},
    \label{eq:fedbiscuit_agg}
\end{equation}
where $S_{r,k}$ is the set of clients assigned to model $k$ at round $r$.
In this multi-model setting, the total client weight
$\sum_{m\in S_{r,k}} p_m$ for a given model $k$ can vary greatly across
rounds, so naive FedAvg over $S_{r,k}$ can overly distort $\theta_k^{(r-1)}$
when only a few or small clients participate. The rule in
\eqref{eq:fedbiscuit_agg} keeps a $(1 - \sum_{m\in S_{r,k}} p_m)$ fraction
of $\theta_k^{(r-1)}$ when this sum is small, preventing unstable learning.

Within Phase~2, we consider two modes. Our main method, \textsc{Hard},
uses the EM-style multi-global design described above without any
additional size control on the clusters: client assignments are purely
driven by validation loss. For completeness, we also implement a balanced variant,
\textsc{Hard-Balance}, which corresponds more closely to the original
FedBiscuit~\citep{wu2024towards}: after each reassignment step,
it applies a heuristic that moves some clients from over-populated
clusters to under-populated ones until the cluster sizes differ by at most
one. The full algorithmic description above applies to both modes; the
empirical results for \textsc{Hard-Balance} are reported in the Section~\ref{sec:suppl}.

In Phase~3, both variants are converted to PFL in the same way: each client
selects its best global model among $\{\theta_k^{(R)}\}_{k=1}^K$ based on
validation loss and locally fine-tunes it on $D^m_{\text{train}}$ to
obtain $\theta_m^{\text{PFL}}$. This yields personalized models that are
directly comparable to the single-global baseline.

\newpage
\subsection{Soft-Fusion Multi-Global Models}
\label{app:alg_softfusion}

\begin{algorithm}[h!]
\small
\caption{Soft-Fusion: Federated Training with Soft-Clustered Global Models and Final Personalization}
\label{alg:softfusion}
\SetKwFunction{UpdateW}{UpdateSoftWeights}
\KwInput{%
  number of global models $K$; total rounds $R$;
  reassignment period $T$; local update steps $\tau$;
  client set $[M]$; initial global parameters $\{\theta_k^{(0)}\}_{k=1}^K$
}
\KwResult{%
  global models $\{\theta_k^{(R)}\}_{k=1}^K$ (GFL) and
  personalized models $\{\theta_m^{\text{PFL}}\}_{m=1}^M$
}

\BlankLine
\textbf{Step 0: Initial soft assignments via validation accuracy}\;

$\{w_{k,m}^{(0)}\}_{k,m} \leftarrow$ \UpdateW{$\{\theta_k^{(0)}\}_{k=1}^K$}\;

\BlankLine
\textbf{Step 1: Federated multi-global training}\;

\For{$r = 1,2,\dots,R$}{
  Sample participating clients $S_r \subseteq [M]$\;

  \tcp{(1) Server $\rightarrow$ clients: fusion-based initialization}
  \ForEach{$m \in S_r$}{
    $\theta^{\mathrm{fuse},(r-1)}_m \leftarrow
      \sum_{k=1}^K w_{k,m}^{(r-1)} \theta_k^{(r-1)}$\;
    Server sends $\theta^{\mathrm{fuse},(r-1)}_m$ to client $m$\;
  }

  \tcp{(2) Local training at clients}
  \ForEach{$m \in S_r$ \textbf{in parallel}}{
    Initialize $\theta_m^{(r,0)} \leftarrow \theta^{\mathrm{fuse},(r-1)}_m$\;
    Run $\tau$ local update steps on $D^m_{\text{train}}$ and
    send $\theta_m^{+,(r)}$ to the server\;
  }

  \tcp{(3) Expert-wise aggregation at server}
  \For{$k = 1,2,\dots,K$}{
    $Z_k^{(r)} \leftarrow \sum_{m \in S_r} w_{k,m}^{(r-1)}$\;
    $\theta_k^{(r)} \leftarrow
      \frac{1}{Z_k^{(r)}} \sum_{m \in S_r}
      w_{k,m}^{(r-1)} \theta_m^{+,(r)}$\;
  }

  \tcp{(4) Update soft assignment weights every $T$ rounds and at $r{=}R$}
  \If{$(r \bmod T = 0)\ \mathrm{or}\ r = R$}{
    $\{w_{k,m}^{(r)}\}_{k,m} \leftarrow$ \UpdateW{$\{\theta_k^{(r)}\}_{k=1}^K$}\;
  }
}

\BlankLine
\textbf{Step 2: Final personalization phase}\;

\tcp{(1) Server: compute fusion-based initialization per client}
\ForEach{$m \in [M]$}{
  $\theta_m^{\text{init}} \leftarrow
    \sum_{k=1}^K w_{k,m}^{(R)} \theta_k^{(R)}$\;
  Server sends $\theta_m^{\text{init}}$ to client $m$\;
}

\tcp{(2) Client-side personalization from fused initialization}
\ForEach{$m \in [M]$ \textbf{in parallel}}{
  Initialize $\theta_m^{(\text{PFL},0)} \leftarrow \theta_m^{\text{init}}$\; 
  Fine-tune on $D^m_{\text{train}}$ to obtain
    $\theta_m^{\mathrm{PFL}}$\;
}

\BlankLine
\SetKwProg{Fn}{Function}{:}{}
\Fn{\UpdateW{$\{\theta_k\}_{k=1}^K$}}{
  \tcp{Server broadcasts models, clients compute weights, server collects}
  Server broadcasts $\{\theta_k\}_{k=1}^K$ to all clients\;\\
  \ForEach{$m \in [M]$ \textbf{in parallel}}{
    \For{$k = 1,2,\dots,K$}{
      $\mathrm{ValAcc}_{k,m} \gets
        \text{ValAcc}(\theta_k, D^m_{\mathrm{val}})$\;
    }
    $S_m \leftarrow \sum_{k=1}^K \mathrm{ValAcc}_{k,m}$\;
    \For{$k = 1,2,\dots,K$}{
      $w_{k,m} \leftarrow \mathrm{ValAcc}_{k,m} / S_m$\;
    }
    Client $m$ sends $\{w_{k,m}\}_{k=1}^K$ to the server\;
  }
  Server collects all $\{w_{k,m}\}_{k\in[K],\,m\in[M]}$\;\\
  \Return{$\{w_{k,m}\}_{k \in [K],\, m \in [M]}$}\;
}
\Return{$\{\theta_k^{(R)}\}_{k=1}^K,\ \{\theta_m^{\text{PFL}}\}_{m=1}^M$}\;
\end{algorithm}

\paragraph{Description.}
Algorithm~\ref{alg:softfusion} introduces our proposed \emph{Soft-Fusion}
design: a soft-clustered multi-global extension of the EM-style design in
FedBiscuit~\citep{wu2024towards}, tailored for federated personalized reward
modeling under communication constraints. As in the hard-clustered case, we
avoid the prohibitive cost of broadcasting all $K$ models at every FL round:
each client still receives \emph{only one} model per round, while all $K$
global models are updated using information from all clients through soft
weights. Unlike the hard-clustered design, we do not run a separate
FedAvg warm-up phase; training starts directly from the initial global
parameters and the soft weights computed in Step~0.

Instead of hard-assigning each client to a single global model, Soft-Fusion
maintains, for each client $m$, a vector of weights
$\{w_{k,m}\}_{k=1}^K$ over the $K$ models, derived from validation
accuracies on $D^m_{\mathrm{val}}$. This is similar in spirit to
mixture-of-experts style soft assignment~\citep{jacobs1991adaptive,
shazeer2017outrageously}, but the communication pattern per round remains
the same as in the single-global baseline. In the GFL phase (Step~1 in
Algorithm~\ref{alg:softfusion}), the server constructs for each participating
client a \emph{fused} initialization
$\theta_m^{\mathrm{fuse},(r-1)} = \sum_{k=1}^K w_{k,m}^{(r-1)}
\theta_k^{(r-1)}$ and sends only this fused model. The client runs $\tau$
local update steps from $\theta_m^{\mathrm{fuse},(r-1)}$, returns
$\theta_m^{+,(r)}$, and the server updates each expert via a weighted
aggregation
\[
    \theta_k^{(r)}
    \;=\;
    \arg\min_{\theta}
    \sum_{m \in S_r}
    w_{k,m}^{(r-1)}
    \bigl\lVert \theta - \theta_m^{+,(r)} \bigr\rVert_2^2,
\]
whose closed-form solution is the weighted average in
Algorithm~\ref{alg:softfusion}. The weights $\{w_{k,m}\}$ are refreshed
every $T$ rounds (and at $r = R$) by temporarily broadcasting all $K$ models
for validation and recomputing the weights, analogous to the grouping step
in hard-clustered EM but without discrete assignment.

In the PFL phase (Step~2), the final multi-global system is converted into
personalized models by forming, for each client $m$, a fused initialization
$\theta_m^{\text{init}} = \sum_{k=1}^K w_{k,m}^{(R)} \theta_k^{(R)}$ and
performing local fine-tuning on $D^m_{\text{train}}$ without further
aggregation, yielding $\theta_m^{\text{PFL}}$. Despite using soft clustering
and multiple experts, each communication round still sends only one model
per client, so the per-round communication cost matches that of the
single-global baseline. This design therefore isolates the effect of a
soft-clustered global initialization from both the single-global FedBis-style
training and the hard-clustered FedBiscuit-style multi-global training in
our federated personalized reward modeling experiments.

Finally, under the standard LoRA initialization in our implementation, the down-projection matrices $A$ are randomly initialized while the up-projection matrices $B$ are set to zero, so the effective low-rank update $BA$ is initially zero and every global model behaves exactly like the shared base model. Consequently, all $K$ models start from the same function, their validation performance on each client is almost identical, and the soft weights $w_{k,m}$ become nearly uniform across $k$,
collapsing Soft-Fusion to the single-global baseline. To avoid this, we initialize the $K$ adapters as small, mutually diverse low-rank perturbations of the base weights so that the initial global models differ slightly but meaningfully, keeping the soft assignments non-uniform from the beginnin (see Subsection~\ref{app:softfusion_lora_init} for details).

\subsection{Hard-Oracle: Oracle Hard-Clustered Multi-Global GFL}
\label{app:hard_oracle}

\begin{algorithm}[h!]
\small
\caption{Hard-Oracle: Federated Training with Oracle Hard Clusters and Final Personalization}
\label{alg:hard_oracle}
\KwInput{%
  number of clusters $K$; total baseline FL rounds $R$;
  local update steps $\tau$; client set $[M]$;
  initial parameters $\{\theta_k^{(0)}\}_{k=0}^K$
}
\KwResult{%
  cluster-wise global models $\{\theta_k^{(R_k)}\}_{k=1}^K$ (GFL)
  and personalized models $\{\theta_m^{\mathrm{PFL}}\}_{m=1}^M$
}

\BlankLine
\textbf{Phase~1: Oracle clustering from local-only models}\;

\tcp{(1) Train local-only models}
\ForEach{$m \in [M]$ \textbf{in parallel}}{
  Initialize $\theta_m^{(0)} \leftarrow \theta_{0}^{(0)}$\;
  Fine-tune on $D^m_{\text{train}}$
  to obtain $\theta_m^{\mathrm{local}}$\;
}

\tcp{(2) Construct distance matrix $D$}
\ForEach{$(i,v)\!\in\![M]^2$}{
  Obtain hard predictions from $\theta_i^{\mathrm{local}}$ and
  $\theta_v^{\mathrm{local}}$ on all test sets
  $\{D^j_{\text{test}}\}_{j=1}^M$
  
  Compute the MCC~\citep{matthews1975comparison} similarity $S(i,v)$ and
  set $D(i,v) \leftarrow 1 - S(i,v)$\;
}

\tcp{(3) Oracle hard clustering via $K$-medoids}
Run $K$-medoids~\citep{kaufman2009finding} on $D$ to obtain clusters
$\{\mathcal{A}_k\}_{k=1}^K$\;

\tcp{(4) Round-budget normalization for fairness}
Set
$
  R_k = \Bigl\lceil R \cdot \frac{|\mathcal{A}_k|}{M} \Bigr\rceil
$
to match the overall training effort per client of the \textit{Single} GFL\;

\tcp{(5) Initialize cluster-global models for subsequent GFL}
For all $k \in [K]$, set the cluster-wise global model at round $0$ as
$\theta_k^{(0)}$\;

\BlankLine
\textbf{Phase~2: Cluster-wise GFL with fixed oracle clusters}\;\\
\For{$k = 1,2,\dots,K$; }{
  \For{$r = 1,2,\dots,R_k$}{\tcp{Cluster $\mathcal{A}_k$ remains fixed; no reassignment is performed}\;
    Sample $S_{k,r} \subseteq \mathcal{A}_k$\;
    {
      Server sends $\theta_k^{(r-1)}$ to all $m \in S_{k,r}$\;\\
      \ForEach{$m \in S_{k,r}$ \textbf{in parallel}}{
        Run $\tau$ local steps on $D^m_{\text{train}}$
        starting from $\theta_k^{(r-1)}$ to obtain $\theta_{k,m}^{(r)}$\;
      }
      \tcp{FedAvg aggregation}
      $\theta_k^{(r)} \leftarrow
        \frac{1}{|S_{k,r}|}
        \sum_{m\in S_{k,r}} \theta_{k,m}^{(r)}$\;
    }
  }
}

\BlankLine
\textbf{Phase~3: Final personalization}\;

\tcp{Each cluster-global model is sent only to its assigned clients}
\For{$k = 1,2,\dots,K$}{
  Server broadcasts $\theta_k^{(R_k)}$ to all $m \in \mathcal{A}_k$\;\\
  \ForEach{$m \in \mathcal{A}_k$ \textbf{in parallel}}{
    Initialize $\theta_m^{(\text{PFL},0)} \leftarrow \theta_k^{(R_k)}$\;
    Fine-tune on $D^m_{\text{train}}$ to obtain
    $\theta_m^{\mathrm{PFL}}$\;
  }
}

\Return{
  $\{\theta_k^{(R_k)}\}_{k=1}^K,\
  \{\theta_m^{\mathrm{PFL}}\}_{m=1}^M$}\;
\end{algorithm}

\paragraph{Description.}
Algorithm~\ref{alg:hard_oracle} evaluates the potential of
hard-clustered multi-global GFL under idealized conditions where cluster
assignments are determined \emph{before} federated training begins.
Unlike validation-driven reassignment in dynamic multi-global approaches
(e.g., \textit{Hard} or \textit{Soft}), Hard-Oracle keeps cluster
assignments fixed throughout GFL. This allows us to study the effect of
multi-global topologies without confounding influences from dynamic
clustering.

\textbf{Phase~1} identifies oracle clusters based solely on behavioral
similarity of client-specific local-only models.
Each client fine-tunes $\theta_{0}^{(0)}$ on its own data to yield
$\theta_m^{\mathrm{local}}$, a model well-adapted to its own test
distribution and reflective of its underlying behavioral preference.
We then compare every pair of clients $(i,v)$ by examining how these models
make \emph{hard} binary predictions across all test distributions.
For each test-owner $j \in [M]$, we first obtain hard labels from the
predicted probabilities $p_i(j,t)$ via
\[
    y_i(j,t) = \mathbf{1}\bigl[p_i(j,t) \ge 0.5\bigr],
\]
and define the standard contingency counts
\[
\begin{aligned}
    a_j &= \sum_{t=1}^{40} \mathbf{1}\bigl[y_i(j,t)=1 \wedge y_v(j,t)=1\bigr],\\
    b_j &= \sum_{t=1}^{40} \mathbf{1}\bigl[y_i(j,t)=1 \wedge y_v(j,t)=0\bigr],\\
    c_j &= \sum_{t=1}^{40} \mathbf{1}\bigl[y_i(j,t)=0 \wedge y_v(j,t)=1\bigr],\\
    d_j &= \sum_{t=1}^{40} \mathbf{1}\bigl[y_i(j,t)=0 \wedge y_v(j,t)=0\bigr],
\end{aligned}
\]
where $t$ indexes the test examples for client $j$.
Note that, by construction of our splits (Section~\ref{sec:prelim}),
each client has exactly $|D^{j}_{\text{test}}| = 40$ test examples,
so the sums above range from $t=1$ to $40$.
The Matthews correlation coefficient (MCC)~\citep{matthews1975comparison} is
\[
\phi_j(i,v)
=
\frac{a_j(i,v)\,d_j(i,v) - b_j(i,v)\,c_j(i,v)}
     {\sqrt{(a_j+b_j)(a_j+c_j)(b_j+d_j)(c_j+d_j)}} ,
\]
with $\phi_j(i,v)=0$ if the denominator is zero.
We average this over all test-owners and normalize it to $[0,1]$:
\[
S(i,v)=\frac{1}{2}\Bigl(\frac{1}{M}\sum_{j=1}^{M}\phi_j(i,v)+1\Bigr),
\qquad
D(i,v)=1-S(i,v).
\]
This distance matrix $D$ characterizes how similarly two clients behave
under different evaluation distributions.

We obtain oracle hard clusters
$\{\mathcal{A}_k\}_{k=1}^K$ by solving the standard $K$-medoids objective
\[
\min_{\{\mathcal{A}_k,M_k\}}
\sum_{k=1}^{K} \sum_{i\in\mathcal{A}_k} D(i,M_k),
\qquad
M_k \in \mathcal{A}_k,
\]
using the PAM algorithm~\citep{kaufman2009finding}.
PAM alternates between (i) assigning each client to its nearest medoid,
and (ii) swapping medoids with non-medoids whenever the total cost
strictly decreases. This yields fixed oracle clusters before any FL
training occurs.
To ensure fair comparison against the \textit{Single} GFL baseline,
each cluster is allocated a proportional number of FL rounds,
$R_k=\lceil R\cdot\,|\mathcal{A}_k|/M\rceil$.

\textbf{Phase~2} performs independent cluster-wise GFL.
Model $\theta_k$ is trained only by clients in $\mathcal{A}_k$ for $R_k$
rounds using FedAvg~\cite{mcmahan2017communication}, without any inter-cluster communication or client
reassignment. Thus, the $K$ global models evolve in isolation and target
different behavioral modes revealed in Phase~1.

\textbf{Phase~3} converts the cluster-global models into personalized
solutions. Each client receives exactly one model
(the one belonging to its fixed oracle cluster)
and fine-tunes it locally to obtain
its final personalized reward model $\theta_m^{\mathrm{PFL}}$,
requiring no additional validation or model selection.

By fixing the clustering structure and matching the overall client-side
training effort to the \textit{Single} baseline, Hard-Oracle provides a
clear reference point for understanding the benefits and limitations of
multi-global architectures in personalized federated reward modeling.

\subsection{LoRA Initialization for Soft-Fusion}
\label{app:softfusion_lora_init}

Without special care, Soft-Fusion would collapse to the single-global
baseline. To keep the communication budget comparable across topologies, we
omit a per-model FedAvg warm-up phase: warming up all $K$ models for $T$
rounds would require $TK$ federated rounds, which is much larger than in the
single-global or hard-clustered settings. GFL therefore starts directly from
the initial LoRA adapters. With the standard LoRA initialization, the
down-projection matrices $A$ are randomly initialized and the up-projection
matrices $B$ are set to zero, so $BA = 0$ and every global model initially
behaves exactly like the shared base model. As a result, all $K$ models
start from the same function, their validation performance is essentially
identical, and the initial soft weights $w_{k,m}$ become uniform across $k$,
so both the fused initialization and expert-wise aggregation reduce to
single-global behavior and the multi-global design provides no additional
effect.

To avoid this collapse while keeping the LoRA updates small, we initialize
the $K$ adapters as small, mutually diverse low-rank perturbations of the
base weights. Concretely, for each LoRA module with output dimension
$d_{\mathrm{out}}$, input dimension $d_{\mathrm{in}}$, and rank $r$, we
first generate orthonormal matrices
$U_{\mathrm{big}} \in \mathbb{R}^{d_{\mathrm{out}} \times K r}$ and
$V_{\mathrm{big}} \in \mathbb{R}^{d_{\mathrm{in}} \times K r}$ by applying
QR decomposition to Gaussian random matrices. For adapter
$k \in [K]$, we take contiguous column blocks
$U_k \in \mathbb{R}^{d_{\mathrm{out}} \times r}$ and
$V_k \in \mathbb{R}^{d_{\mathrm{in}} \times r}$ from
$U_{\mathrm{big}}$ and $V_{\mathrm{big}}$, and set
\[
  A_k \;\gets\; V_k^\top, 
  \qquad
  B_k \;\gets\; \varepsilon_{\text{layer}} (1 + \text{jitter}\cdot k)\, U_k,
\]
where $\varepsilon_{\text{layer}}$ is scaled from the Frobenius norm of the
corresponding base weight (or a fan-in heuristic if unavailable), and
$\text{jitter}$ introduces a small adapter-specific scaling offset.
In our experiments, we use $\varepsilon_{\text{layer}} = 0.005$
for the global scaling and set $\text{jitter}= 0.01$.
The same procedure is applied to LoRA embedding modules when present. This
orthogonal, norm-controlled initialization makes each global model a slightly
different perturbation of the base model, so that the initial validation
accuracies and soft weights $w_{k,m}$ are non-uniform, while keeping all
LoRA updates in a small neighborhood of the base weights.

\section{Supplemental Evidence}
\label{sec:suppl}

\begin{figure*}[t]
    \centering
    \begin{subfigure}[t]{0.43\textwidth}
        \centering
        \includegraphics[height=0.38\textheight]{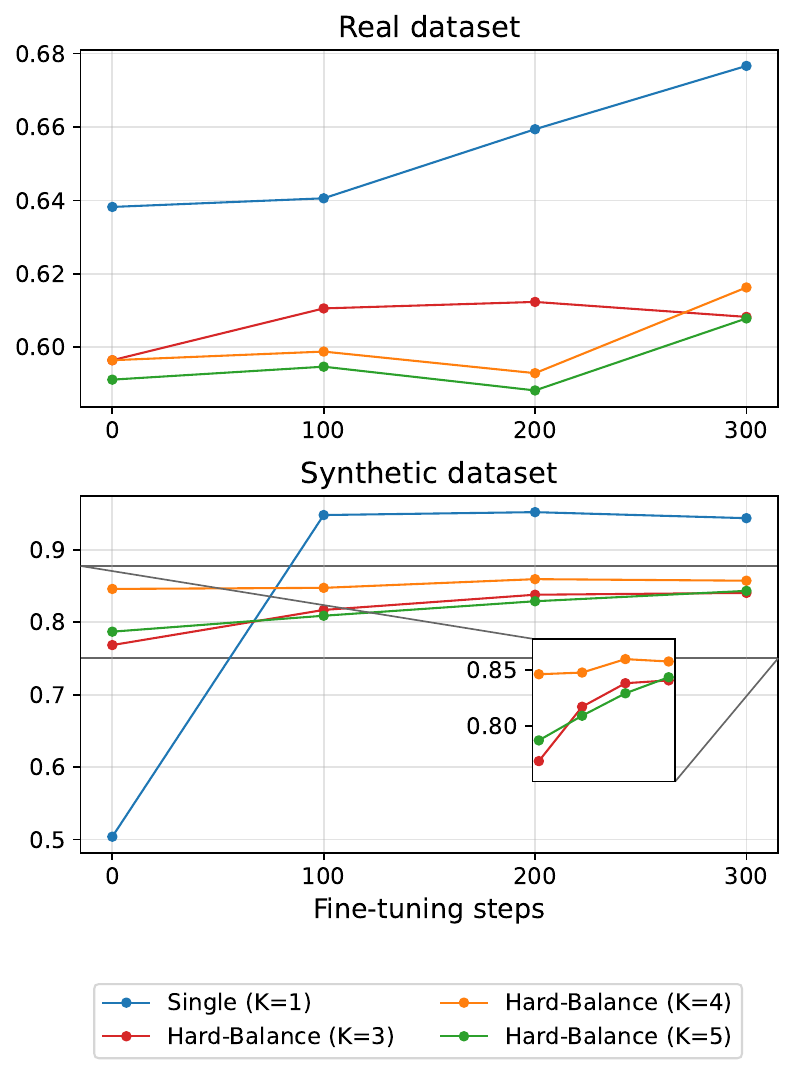}
        \caption{PFL test accuracy}
        \label{fig:pfl-single-vs-hard-balance}
    \end{subfigure}
    \hfill
    \begin{subfigure}[t]{0.55\textwidth}
        \centering
        \includegraphics[height=0.38\textheight]{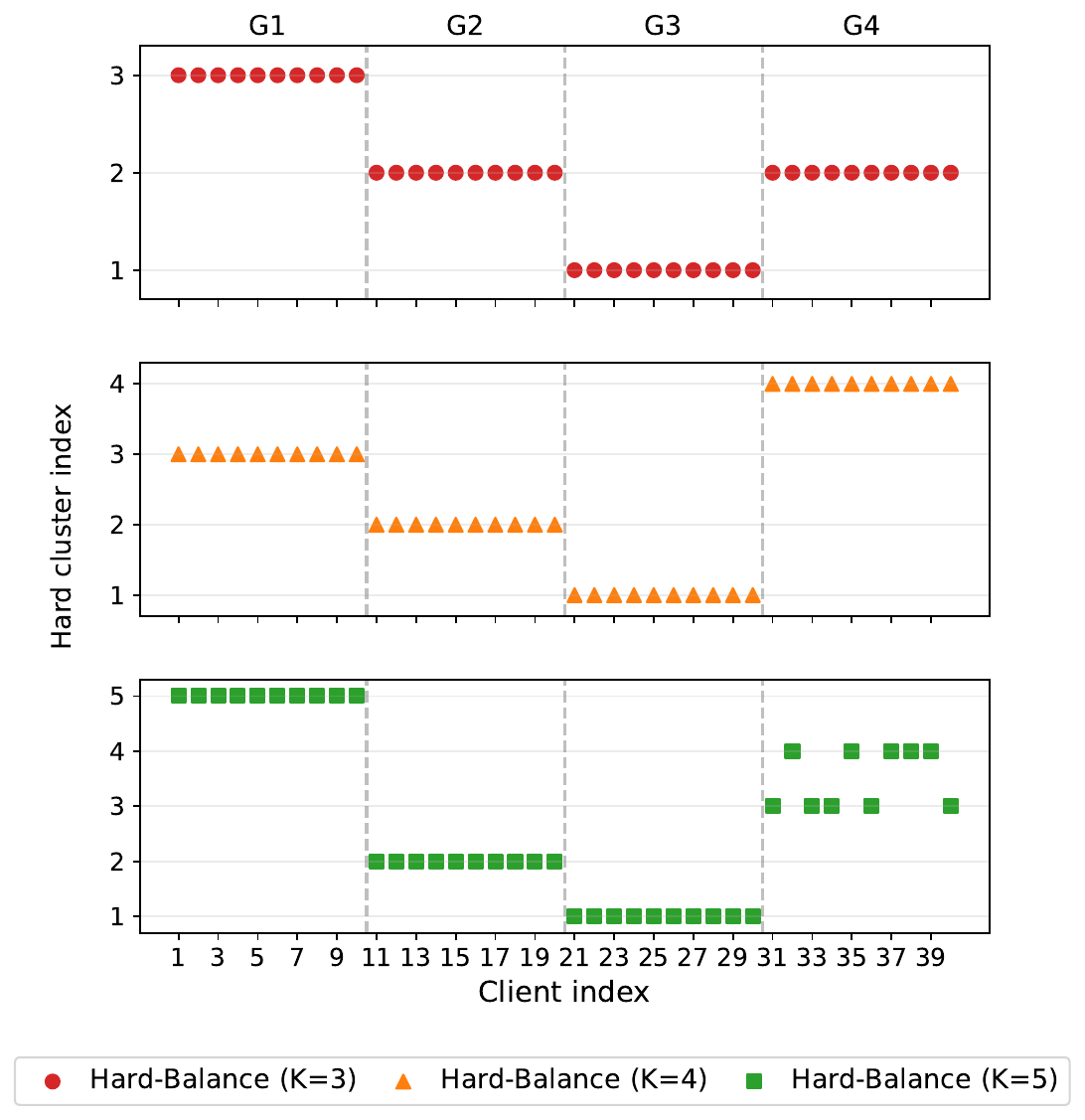}
        \caption{Hard-cluster assignment}
        \label{fig:pfl-hard-balance-assignment}
    \end{subfigure}

    \caption{Comparison of (a) mean preference-prediction accuracy of $\theta_m^{\text{PFL}}$ on $\mathcal{D}_{\text{test}}^{m}$ as a function of fine-tuning steps for the real and synthetic datasets, comparing a \textsc{single} with \textsc{hard-balance}~($K\in\{3,4,5\}$), and (b) hard-cluster indices assigned to
each client on the synthetic dataset for the same
\textsc{hard} settings.}
    \label{fig:pfl-single-hard-balance-overall}
\end{figure*}

\subsection{Within FL: Single vs.\ Hard-Balance}

Figure~\ref{fig:pfl-single-vs-hard-balance} compares \textsc{Single} with \textsc{Hard-Balance}---a cluster-size balancing variant of \textsc{Hard} introduced in FedBiscuit~\citep{wu2024towards}---for $K \in \{3,4,5\}$ on both the real-world and synthetic datasets. Adding balancing does not change our main conclusion: \textsc{Single} achieves higher personalized accuracy than every \textsc{Hard-Balance} variant. On the synthetic dataset, we simulate four equally sized preference groups G1--G4 (10 clients each), and \textsc{Hard-Balance} performs best at $K=4$, matching the true number of groups. Figure~\ref{fig:pfl-hard-balance-assignment} shows that $K=4$ assigns each group to its own cluster, whereas $K=3$ merges different groups (e.g., G2 and G4) into a shared cluster despite being consistent within groups, and $K=5$ splits a group (notably G4) across multiple clusters; correspondingly, personalization degrades as $K$ deviates from the true number of groups.

\subsection{\textsc{Single} Remains Best Under Group-Size Imbalance}
\label{sec:unbalanced-groups}

We further consider a more challenging synthetic setting with size-imbalanced preference groups, where $|G1|=|G4|=5$ and $|G2|=|G3|=15$, to test whether our conclusions persist under group-size skew.
As shown in Figure~\ref{fig:unbal-synth-acc}, \textsc{Single} still achieves the highest mean personalized accuracy across fine-tuning budgets $\tau$.
We additionally report the resulting hard assignments in Figure~\ref{fig:unbal-synth-assign} for \textsc{Hard} and \textsc{Hard-Balance} with $K=4$.
Although $K$ matches the true number of groups, the learned assignments are not group-consistent under this size-imbalanced setting (e.g., G4 under \textsc{Hard-Balance} and G3 under \textsc{Hard} split across clusters).

\begin{figure*}[t]
    \centering

    \begin{subfigure}[t]{0.53\textwidth}
        \centering
        \includegraphics[height=0.30\textheight,width=\linewidth,keepaspectratio]{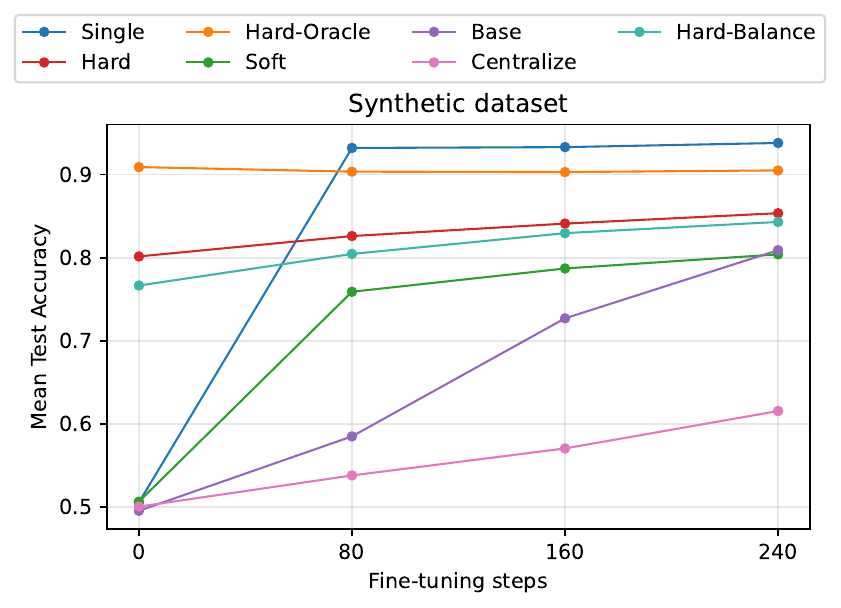}
        \caption{PFL test accuracy}
        \label{fig:unbal-synth-acc}
    \end{subfigure}
    \hfill
    \begin{subfigure}[t]{0.45\textwidth}
        \centering
        \includegraphics[height=0.30\textheight,width=\linewidth,keepaspectratio]{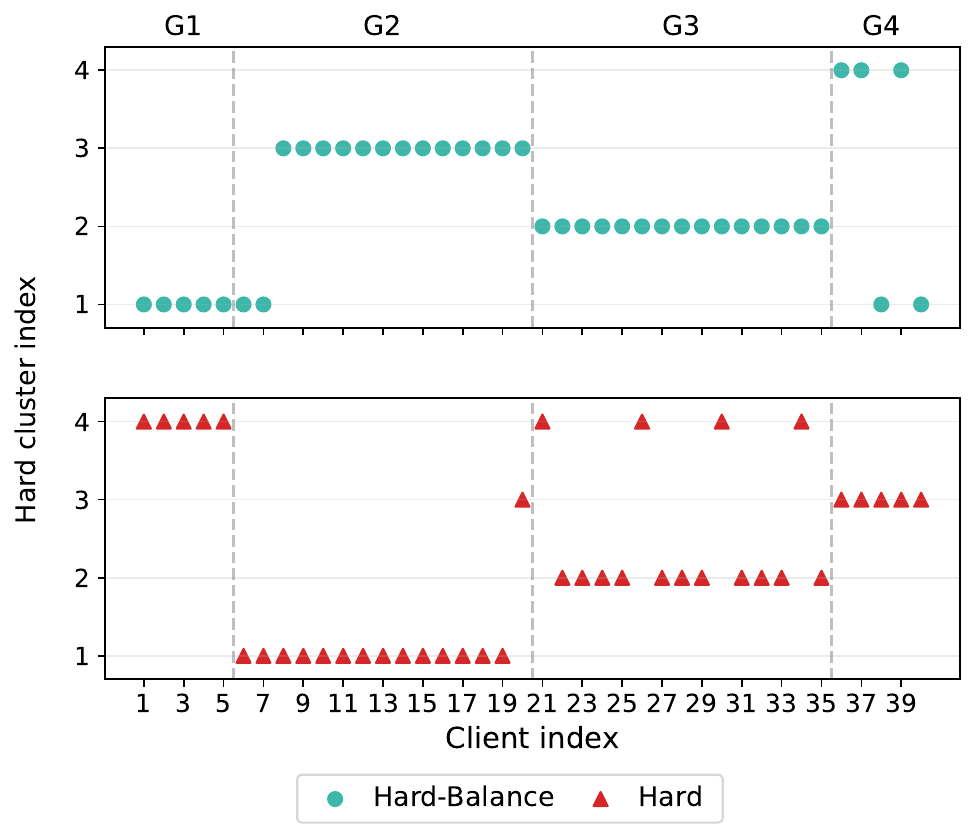}
        \caption{Hard-cluster assignment ($K=4$)}
        \label{fig:unbal-synth-assign}
    \end{subfigure}

    \caption{Unbalanced synthetic results with group-size skew. (a) Mean preference-prediction accuracy of $\theta_m^{\text{PFL}}$ on $\mathcal{D}_{\text{test}}^{m}$ as a function of fine-tuning steps, comparing global initialization strategies \textsc{Single}, \textsc{Hard-Oracle}, \textsc{Base}, \textsc{Hard-Balance}, \textsc{Hard}, \textsc{Soft}, and \textsc{Centralized}. (b) Client hard-cluster indices under \textsc{Hard} and \textsc{Hard-Balance} with $K=4$ on the same dataset.}
    \label{fig:unbal-synth-overall}
\end{figure*}

\section{Conclusion and Future Work}
\label{sec:conclusion}

This work investigates, in FL, whether multi-global models are actually needed as initializations for personalization under preference heterogeneity. In our experiments, a single global model trained with FL and then locally fine-tuned yields better personalized performance than multi-global variants and also outperforms a centralized model trained on pooled data. In FL, each client updates a single shared model on its own preference-derived data, so local fine-tuning can adapt this shared initialization effectively and additional global models are unnecessary. Future work may explore single-global FL objectives and regularizers that better handle data heterogeneity, with the goal of yielding stronger personalized models.

\chapter{Dissertation Conclusion}

This dissertation studied how to build efficient personalized models in federated settings with heterogeneous data and preferences. The central conclusion is that effective personalization depends on a global training phase that preserves stable feature vectors and transferable knowledge. We first showed that federated averaging degrades primarily due to feature-norm mismatch between local and global models, and proposed FedFN to reduce this mismatch and stabilize learning under non-IID data. We then analyzed how local alignment objectives can cause forgetting of unobserved classes, and proposed FedDr+, which combines dot-regression with global feature distillation to improve local alignment while preserving global knowledge. Finally, in federated personalized reward modeling under preference heterogeneity, we investigated whether multi-global models are actually needed as initializations for personalization and empirically found that a single global model trained with federated learning and then locally fine-tuned yields the best personalized performance, outperforming multi-global variants and also a centralized model trained on pooled data. These results indicate that clients can effectively adapt a shared initialization using their own preference-derived data, so introducing additional global models is unnecessary for strong personalization.

\bibliographystyle{unsrtnat}   
\bibliography{references}

\acknowledgment[4]
먼저, 석·박사 과정 동안 저를 지도해 주신 윤세영 지도교수님께 진심으로 감사드립니다. 연구실에 처음 들어왔을 때 하신 “초심을 잃지 말자”라는 말씀은 오랜 시간 연구를 이어오는 동안 늘 되짚게 된 기준이었습니다. 연구가 잘 풀리지 않거나 성과가 바로 드러나지 않을 때에도 성급히 결론을 내리기보다 문제의 본질을 끝까지 고민하도록 이끌어 주셨고, 연구의 흐름을 존중하며 기다려 주셨습니다. 그 과정에서 스스로 판단하고 책임지는 연구자의 태도를 갖추게 되었고, 연구를 대하는 자세 전반에 큰 영향을 받았습니다. 긴 시간 변함없이 지도해 주신 교수님께 깊이 감사드립니다.

바쁘신 일정에도 불구하고 학위 심사에 참여해 주시고 냉철한 피드백을 주신 양은호 교수님, 이문용 교수님, 박찬영 교수님, 김희영 교수님께 깊이 감사드립니다. 교수님들의 조언 덕분에 논문을 보완하고 더 완성도 있게 마무리할 수 있었습니다.\newline

안수명 교수님께 감사드립니다. 교수님께서는 연구 전반에서 방향을 잡아 주시며 꾸준히 이끌어 주셨고, 개인적으로는 연구와 삶 모두에서 큰 영향을 받은 ‘대부’와 같은 존재였습니다. 그 가르침을 통해 연구자로서뿐 아니라 한 사람으로서도 성장할 수 있었습니다. 정민찬 학생에게도 감사드립니다. 연구실에서 많은 시간을 함께 보내며 연구와 일상에 대한 이야기를 나누었고, 크고 작은 고민들을 자연스럽게 공유할 수 있었던 동료였습니다.

상묵이형, 재훈이형, 기훈이형과 태현이 등 OSI 연구실 졸업생분들께 감사드립니다. 선배님들의 기여 덕분에 더 나은 환경에서 연구를 이어갈 수 있었고, 제가 맡았던 연구와 프로젝트도 끝까지 잘 마무리할 수 있었습니다. 대학원 생활을 버티는 동안 힘들 때마다 터놓고 장난도 치고, 진지한 이야기까지 나누며 의지할 수 있었던 호정, 상민, 상화, 종우, 성년, 영록에게도 고맙다고 말하고 싶습니다. 제가 연구하다 잘 안 풀릴 때 자기 일처럼 함께 고민해 주고 도움을 주시며 의지할 수 있었던 지환이형, 보령, 성우, 정현과 남규에게도 감사드립니다. 오랫동안 연구실에 묵묵히 기여해 주고 있는 세혁이형, 용식이형에게도 고맙습니다. 또한 보이지 않는 곳에서 연구실 운영과 서버 관리를 책임지며 연구에 집중할 수 있는 환경을 만들어 주신 랩장님들과 서버 관리자분들께도 깊이 감사드립니다. 졸업하고도 계속 챙겨 주고 걱정해 주시는 상욱이형, 낙일이형, 종협이형도 감사드립니다. 제가 여유가 없어 자주 못 뵈었는데, 꼭 찾아뵙겠습니다. 이외에도 일일이 언급하지 못한 OSI 연구실 모든 선후배와 동료들에게 감사드립니다. 여러분이 있었기에 오늘에 이를 수 있었습니다. 진심으로 감사드립니다.\newline

학부 시절부터 함께 공부하며 인연을 쌓아 온 연승, 정헌, 준용, 현민, 형탁, 정현이형, 호용이형에게도 감사드립니다. 함께 스터디하며 서로에게 자극이 되었고, 그 과정에서 자연스럽게 성장할 수 있었습니다. 아울러 동아리 Aquila에서 만난 친구들과 새터반 동료들 또한 학부 시절을 함께 채워 준 소중한 인연으로 오래 기억하겠습니다.

그리고 제 집이자 학교이며 삶의 터전이 되어 준 KAIST에도 깊이 감사드립니다. 고등학교 시절부터 정말 들어오고 싶었던 곳이었고, 이곳에서 학부부터 석·박사 과정까지 뛰어난 동료들과 함께 배우고 경쟁하며 성장할 수 있었습니다. 그 과정에서 쌓아 온 경험들이 모여 오늘의 제가 될 수 있었습니다.\newline

마지막으로 긴 여정을 지켜봐 주며 늘 곁에서 힘이 되어준 가족들에게 감사드립니다. 언제나 무조건적으로 지지해 주시고 제 편이 되어 주셨던 가족분들께 진심으로 존경과 감사의 마음을 전합니다. 가족의 믿음과 응원이 있었기에 이 과정을 끝까지 걸어올 수 있었습니다.

\begin{flushright}
2025년 12월 25일\\
김성윤 올림
\end{flushright}

  \label{paperlastpagelabel}     
\end{document}